\documentclass[11pt]{article}

\usepackage{ragged2e}
\usepackage{lmodern}
\usepackage[utf8]{inputenc} 
\usepackage[T1]{fontenc}    
\usepackage{amsfonts}       
\usepackage{nicefrac}       
\usepackage{microtype}      
\usepackage{fullpage}
\usepackage{mystyle1}

\usepackage{appendix}
\usepackage{setspace}
\title{On Reward-Free RL with Kernel and Neural Function Approximations: Single-Agent MDP and Markov Game}

\begin{document}

\author{Shuang Qiu\thanks{University of Michigan. 
Email: \texttt{qiush@umich.edu}.} 
	\qquad
		Jieping Ye\thanks{University of Michigan. 
    Email: \texttt{jpye@umich.edu}.}
	\qquad    
       	Zhaoran Wang\thanks{Northwestern University.
    Email: \texttt{zhaoranwang@gmail.com}.}
    \qquad
       	Zhuoran Yang\thanks{
       Princeton University. 
    	Email: \texttt{zy6@princeton.edu}.}
}

\maketitle


\begin{abstract}
To achieve sample efficiency in reinforcement learning (RL), it necessitates efficiently exploring the underlying environment.
Under the offline setting, addressing the exploration challenge  lies in collecting an offline dataset with sufficient coverage.
Motivated by such a challenge, we study the reward-free RL problem, where an agent aims to thoroughly explore the environment without any pre-specified reward function. Then, given any  extrinsic reward, the agent computes the policy via a planning algorithm with offline data collected in the exploration phase. 
Moreover, we tackle this problem under the context of  function approximation, leveraging powerful function approximators.
 Specifically, we propose to explore via an optimistic variant of the value-iteration algorithm incorporating kernel and neural function approximations, where we  adopt the associated exploration bonus as the exploration reward. Moreover, we design exploration and planning algorithms for both single-agent MDPs and zero-sum Markov games and prove that our methods can achieve $\widetilde{\mathcal{O}}(1 /\varepsilon^2)$ sample complexity for generating a $\varepsilon$-suboptimal policy or  $\varepsilon$-approximate Nash equilibrium when given an arbitrary extrinsic reward. To the best of our knowledge, we establish the first provably efficient reward-free RL algorithm with kernel and neural function approximators. 
\end{abstract}

\section{Introduction}

While  reinforcement learning (RL) with function approximations has achieved great empirical success \citep{mnih2015human,silver2016mastering,silver2017mastering,vinyals2019grandmaster}, its application is  mostly enabled by massive interactions with the unknown environment, 
especially when the state space is large and function approximators such as neural networks are employed.   
To achieve sample efficiency, any RL algorithm needs to accurately learn the transition model either explicitly or implicitly, which brings the need of efficient exploration.

Under the setting of offline RL, the agent aims to learn the optimal policy only from an offline dataset collected a priori, without any interactions with the environment. Thus, the collected offline dataset  
should have sufficient coverage of the trajectory generated by the optimal policy. 
However, in real-world RL applications, 
the reward function is often designed by the learner based on the domain knowledge. The learner might have a set of reward functions  to choose from or use an adaptive algorithm for reward design \citep{laud2004theory,grzes2017reward}.
In such a scenario, 
it is often desirable to collect an offline dataset that covers all the possible optimal trajectories associated with a set of reward functions. 
With such a benign offline dataset, for any arbitrary reward function, the RL agent has sufficient information to estimate the corresponding policy.


To study such a problem in a principled manner, we focus on the framework of reward-free RL,
which consists of an exploration phase and a planning phase. Specifically,  
in the exploration phase, the agent interacts with the environment without accessing any pre-specified rewards and collects empirical trajectories for the subsequent planning phase. 
During the planning phase, using the offline data collected in the exploration phase, the agent computes the optimal policy when given an extrinsic reward function, without further interactions with the environment.

Recently, many works focus on designing provably sample-efficient reward-free RL algorithms. For the single-agent tabular case,  \citet{jin2020reward,kaufmann2020adaptive,menard2020fast,zhang2020nearly} achieve  $\tilde{\cO}(\poly(H, |\cS|, |\cA|)/\varepsilon^2)$ sample complexity for obtaining $\varepsilon$-suboptimal policy, where $|\cS|, |\cA|$ are the sizes of state and action space, respectively. 
In view of the large action and state spaces, the works \citet{zanette2020provably,wang2020reward}  theoretically analyze reward-free RL by applying the linear function approximation for the single-agent Markov decision process (MDP), which achieve $\tilde{\cO}( \poly(H, \mathfrak{d})/\varepsilon^2)$ sample complexity with $\mathfrak{d}$ denoting the dimension of the feature space. However, RL algorithms combined with nonlinear function approximators such as kernel and neural function approximators have shown great empirical successes in a variety of application
problems (e.g., \citet{duan2016benchmarking,silver2016mastering,silver2017mastering,wang2018deep,vinyals2019grandmaster}), thanks to their expressive power. 
  On the other hand, although reward-free RL algorithms for the multi-player Markov games in the tabular case have been studied in \citet{bai2020provable,liu2020sharp}, there is still a lack of works theoretically studying multi-agent scenarios with the function approximation. Thus, the following question remains open:

\begin{center}
\emph{Can we design provably efficient reward-free RL algorithms with kernel and neural function approximations for both single-agent MDPs and Markov games? }    
\end{center}

The main challenges of answering the above question lie in how to appropriately integrate nonlinear approximators into the framework of reward-free RL and how to incentivize the exploration by designing exploration rewards and bonuses that fit such approximation. In this paper, we provide an affirmative answer to the above question by tackling these challenges. Our contributions are summarized as follows:

\vspace{2pt}
\noindent\textbf{Contributions.} In this paper, we first propose provable sample and computationally efficient reward-free RL algorithms with kernel and neural function approximations for the single-agent MDP setting. 
Our exploration algorithm is an optimistic variant of the least-squares value iteration algorithm, incorporating kernel and neural function approximators, which adopts the associated (scaled) bonus as the exploration reward. Further with the planning phase, our method achieves an $\tilde{\mathcal{O}}(1/\varepsilon^2)$ sample complexity to generate an $\varepsilon$-suboptimal policy for an \emph{arbitrary} extrinsic reward function. Moreover, we extend the proposed method for the single-agent setting to the zero-sum Markov game setting such that the algorithm can achieve an $\tilde{\cO}(1/\varepsilon^2)$ sample complexity to generate a policy pair which is an $\varepsilon$-approximate Nash equilibrium. Particularly, in the planning phase for Markov games, our algorithm only involves finding the Nash equilibrium of matrix games formed by Q-function that can be solved \emph{efficiently}, which is of independent interest. The sample complexities of our methods match the $\tilde{\cO}(1/\varepsilon^2)$ results in existing works for tabular or linear function approximation settings. To the best of our knowledge, we establish the first provably efficient reward-free RL algorithms with kernel and neural function approximators for both single-agent and multi-agent settings.

\vspace{5pt}
\noindent\textbf{Related Work.} There have been a lot of works focusing on designing provably efficient reward-free RL algorithms for both single-agent and multi-agent RL problems. For the single-agent scenario, \citet{jin2020reward} formalizes the reward-free RL for the tabular setting and provide theoretical analysis for the proposed algorithm with an $\tilde{\cO}(\poly(H, |\cS|, |\cA|)/\varepsilon^2)$ sample complexity for achieving $\varepsilon$-suboptimal policy. The sample complexity for the tabular setting is further improved in several recent works \citep{kaufmann2020adaptive,menard2020fast,zhang2020nearly}.  Recently, \citet{zanette2020provably,wang2020reward} study the reward-free RL from the perspective of the linear function approximation. For the multi-agent setting, \citet{bai2020provable} studies the reward-free exploration for the zero-sum Markov game for the tabular case. \citet{liu2020sharp} further proposes provable reward-free RL algorithms for multi-player general-sum games. 

Our work is also closely related to a line of works that study RL algorithms with function approximations. There are a great number of works \citep{yang2019sample,yang2020reinforcement,cai2019provably,zanette2020frequentist,jin2020provably,wang2019optimism,ayoub2020model,zhou2020provably,kakade2020information} studying different RL problems with (generalized) linear function approximation. Furthermore, \citet{wang2020provably} proposes an optimistic least-square value iteration (LSVI) algorithm with general function approximation. Our work is more related to the problems of kernelized contextual bandits \citep{srinivas2009gaussian,valko2013finite,chowdhury2017kernelized}, neural contextual bandits \citep{zhou2020neural}, and reward-based online RL with kernel and neural function approximations \citep{yang2020provably}.
Specifically, \citep{srinivas2009gaussian,valko2013finite,chowdhury2017kernelized} propose algorithms for kernelized contextual bandits with constructing corresponding upper confidence bound (UCB) bonuses to encourage exploration. For neural contextual bandits, \citet{zhou2020neural} proposes a neural network-based algorithm named NeuralUCB, which uses a neural random feature mapping to construct a UCB for exploration. Adapting the UCB bonus construction from these works, the recent work \citet{yang2020provably} studies optimistic LSVI algorithms for the online RL problem with kernel and neural function approximations, which covers the contextual bandit problem as a special case. However, these previous works only study the settings where the exploration is executed with reward feedbacks, which cannot be directly applied to the reward-free RL problem. Inspired by the aforementioned works, our work extends the idea of kernel and neural function approximations to the reward-free single-agent MDPs and zero-sum Markov games.

\section{Preliminaries}

In this section, we introduce the basic notations and problem backgrounds for this paper.

\subsection{Markov Decision Process}

Consider an episodic single-agent MDP defined by the tuple $(\cS, \cA,  H, \PP, r)$, where $\cS$ denotes the state space, $\cA$ is the action space of the agent, $H$ is the length of each episode, $\PP = \{\PP_h\}_{h=1}^H$ is the transition model with $\PP_h(s'|s,a)$ denoting the transition probability at the $h$-th step from the state $s\in\cS$ to the state $s'\in\cS$ when the agent takes action $a\in\cA$, and $r = \{r_h\}_{h=1}^H$ with $r_h : \cS \times \cA \mapsto [0, 1]$ denotes the reward function. Specifically, we assume that the true transition model $\PP$ is \emph{unknown} to the agent which necessitates the reward-free exploration. The policy of an agent is a collection of probability distributions $\pi = \{ \pi_h\}_{h=1}^H$ where $\pi_h: \cS \mapsto \Delta_\cA$\footnote{If a policy $\pi_h$ is deterministic, for simplicity, we slightly abuse the notion by letting $\pi_h:\cS\mapsto\cA$. A similar notation is also defined for the Markov game setting.} with $\Delta_\cA$ denoting a probability simplex defined on the space $\cA$.

For a specific policy $\{\pi_h\}_{h=1}^H$ and reward function $\{r_h\}_{h=1}^H$, under the transition model $\{\PP_h\}_{h=1}^H$, we define the associated value function $V_h^{\pi}(s, r): \cS \mapsto \RR$ at the $h$-th step as $V_h^{\pi}(s, r) := \allowbreak \EE[ \sum_{h'=h}^H r_{h'} (s_{h'}, a_{h'})  \given s_h = s, \pi, \PP ]$.
The corresponding action-value function (Q-function)  $Q_h^\pi: \cS \times \cA \mapsto \RR$ is further defined as $Q_h^\pi(s, a, r):=\EE[ \sum_{h'=h}^H r_{h'}(s_{h'}, a_{h'})  \given s_h = s, a_h = a,  \pi, \PP]$. Therefore, we have the Bellman equation as $V_h^{\pi}(s, r) = \langle Q_h^{\pi}(s, \cdot, r), \pi_h(s)\rangle_\cA$\footnote{When a policy $\pi_h$ is deterministic, we have the value function defined as $V_h^{\pi}(s, r) = Q_h^{\pi}(s,  \pi_h(s), r)$.} and $Q_h^{\pi}(s, a, r) = r_h (s, a) +  \langle  \PP_h(\cdot | s, a), V_{h+1}^{\pi} (\cdot, r) \rangle_\cS$, where we let $\langle\cdot, \cdot \rangle_\cS$, $\langle \cdot, \cdot \rangle_\cA$ denote the inner product over the spaces $\cS$, $\cA$. The above Bellman equation holds for all $h \in [H]$ with setting $V_{H+1}^{\pi} (s) = 0, \forall s \in \cS$. In the rest of this paper, for simplicity of the notation, we rewrite $\langle  \PP_h(\cdot | s, a), V_{h+1}(\cdot, r) \rangle_\cS = \PP_hV_{h+1}(s,a,r)$ for any transition probability $\PP_h$ and value function $V(\cdot, r)$. Moreover, we denote $\pi_r^*$ as the optimal policy w.r.t. $r$, such that $\pi_r^*$ maximizes $V_1^\pi(s_1, r)$\footnote{WLOG, we assume the agent starts from a fixed state $s_1$ at $h=1$. We also make the same assumption for the Markov game setting.}. Then, we define $Q_h^*(s,a,r) := Q_h^{\pi_r^*}(s,a,r)$ and $V_h^*(s,r) := V_h^{\pi_r^*}(s,r)$. We say $\tilde{\pi}$ is an $\varepsilon$-suboptimal policy if it satisfies
\begin{align*}
V_1^*(s_1, r) - V_1^{\tilde{\pi}}(s_1, r) \leq \varepsilon.
\end{align*}

\subsection{Zero-Sum Markov Game}

In this paper, we also consider an episodic zero-sum Markov game characterized by the tuple $(\cS, \cA, \cB,  H, \PP, r)$, where $\cS$ denotes the state space, $\cA$ and $\cB$ are the action spaces for the two players, $H$ is the length of each episode, $\PP = \{\PP_h\}_{h=1}^H$ is the transition model with $\PP_h(s'|s,a, b)$ denoting the transition probability at the $h$-th step from the state $s$ to the state $s'$ when Player 1 takes action $a\in\cA$ and Player 2 takes action $b \in \cB$, and $r=\{r_h\}_{h=1}^H$ with $r_h : \cS \times \cA \times \cB \mapsto [0, 1]$ denotes the reward function. Similarly, we assume the transition model $\PP=\{\PP_h\}_{h=1}^H$ is \emph{unknown} to both players. 
The policy of Player 1 is a collection of probability distributions $\pi = \{ \pi_h\}_{h=1}^H$ with $\pi: \cS \mapsto \Delta_\cA$.  Analogously, the policy of Player 2 is a collection of probability distributions $\nu = \{ \nu_h\}_{h=1}^H$ with $\nu: \cS \mapsto \Delta_\cB$. Here $\Delta_\cA$ and $\Delta_\cB$ are probability simplexes defined on the spaces $\cA$ and $\cB$.  

For a specific policy pair $\pi$ and $\nu$ and reward function $\{r_h\}_{h\in [H]}$, under the transition model $\{\PP_h\}_{h\in [H]}$, we define the value function $V_h^{\pi, \nu}(s, r): \cS \mapsto \RR$ at the $h$-th step as $V_h^{\pi, \nu}(s, r) := \EE[ \sum_{h'=h}^H r_{h'} (s_{h'}, a_{h'}, b_{h'})  \given s_h = s, \pi, \nu, \PP ]$.
We further define the corresponding action-value function (Q-function) $Q_h^{\pi, \nu}: \cS \times \cA \times \cB \mapsto \RR$ as $Q_h^{\pi, \nu}(s,a,b, r):=\EE[ \sum_{h'=h}^H r_{h'}(s_{h'}, a_{h'}, b_{h'} )  \given s_h = s, a_h = a, b_h = b,  \pi, \PP]$.
Thus, we have the Bellman equation as $V_h^{\pi, \nu}(s, r) =\EE_{a\sim \pi_h, b\sim \nu_h} [ Q_h^{\pi, \nu}(s, a , b, r)]$ and $Q_h^{\pi, \nu}(s, a, b, r) = r_h (s, a, b) +  \PP_h V_{h+1}^{\pi, \nu} (s,a,b,r)$,
where, for ease of notation, we also let $\PP_h V_{h+1}^{\pi, \nu} (s,a,b,r) \allowbreak = \big\langle  \PP_h(\cdot | s, a, b), V_{h+1}^{\pi, \nu} (\cdot, r) \big\rangle_\cS$.

We define the Nash equilibrium (NE) $(\pi^\dag, \nu^\dag)$ as a solution to $\max_{\pi} \min_{\nu} V_1^{\pi, \nu}(s_1)$, where we have $V_1^{\pi^\dag, \nu^\dag}(s_1, r) = \max_{\pi} \min_{\nu} V_1^{\pi, \nu}(s_1,r) = \min_{\nu}  \max_{\pi} V_1^{\pi, \nu}(s_1,r)$.
For simplicity, we let $V_h^{\dagger}(s,r) = V_h^{\pi^\dag, \nu^\dag}(s, r)$ and also $Q_h^{\dagger}(s,a,b,r) = Q_h^{\pi^\dag, \nu^\dag}(s,a,b, r)$ denote the value function and Q-function under the NE $(\pi^\dag, \nu^\dag)$ at $h$-th step. We further define the best response for Player 1 with policy $\pi$ as $\bre(\pi):=\argmin_\nu V_1^{\pi, \nu}(s_1, r)$ and the best response for Player 2 with policy $\nu$ as $\bre(\nu):=\argmax_\pi V_1^{\pi, \nu}(s_1, r)$. Thus, we say $(\tilde{\pi}, \tilde{\nu})$ is an $\varepsilon$-approximate NE if it satisfies
\begin{align*}
V_1^{\bre(\tilde{\nu}), \tilde{\nu}}(s_1,r) - V_1^{\tilde{\pi}, \bre(\tilde{\pi})}(s_1,r) \leq \varepsilon,
\end{align*}
where $V_1^{\bre(\tilde{\nu}), \tilde{\nu}}(s_1,r) \geq V_1^{\dag}(s_1,r)  \geq V_1^{\tilde{\pi}, \bre(\tilde{\pi})}(s_1,r)$ always holds. On the other hand, we let $V^*_1(s, r) = \max_{\pi, \nu} V^{\pi, \nu}_1(s, r)$, namely the maximal value function when $h=1$. Then, we have the associated value function and Q-function for the $h$-th step $V^*_h(s, r)$ and $Q^*_h(s, a, b, r)$.


\subsection{Reproducing Kernel Hilbert Space}

We study the kernel function approximation based on the reproducing kernel Hilbert space (RKHS). With slight abuse of notion, we let $\cZ = \cS \times \cA$ for the single-agent MDP setting and  $\cZ = \cS \times \cA \times \cB$ for the zero-sum game setting, such that $z=(s,a)\in \cZ$ or $z=(s,a,b) \in \cZ$ for different cases. We assume that the space $\cZ$ is the input space of the approximation function, where $\cZ$ is a compact space on $\RR^d$. This can also be achieved if there is a preprocessing method to embed $(s,a)$ or $(s,a,b)$ into the space $\RR^d$. We let $\cH$ be a RKHS defined on the space $\cZ$ with the kernel function $\ker: \cZ \times \cZ \mapsto \RR$. We further define the inner product on the RKHS $\cH$ as $\langle\cdot, \cdot\rangle_\cH : \cH\times \cH \mapsto \RR $ and the norm $\|\cdot\|_\cH: \cH \mapsto \RR$. We have a feature map $\phi: \cZ \mapsto \cH$ on the RKHS $\cH$ and define the function $f(z):=\langle f, \phi(z)\rangle_\cH$ for $f\in \cH$. Then the kernel is defined as 
\begin{align*}
\ker(z,z'):=\langle \phi(z), \phi(z')\rangle_\cH, \quad \forall z, z'\in \cZ.
\end{align*}
We assume that $\sup_{z\in \cZ} \ker(z,z) \leq 1$ such that $\|\phi(z)\|_\cH \leq 1$ for any $z\in \cZ$.

\subsection{Overparameterized Neural Network}\label{sec:over_NN}

This paper further considers a function approximator utilizing the overparameterized neural network. Overparameterized neural networks have drawn a lot of attention recently in both theory and practice \citep{neyshabur2018towards,allen2018learning,arora2019fine,gao2019convergence,bai2019beyond}.
Specifically, in our work, we have a two-layer neural network $f(\cdot; b, W): \cZ \mapsto \RR$ with $2m$ neurons and weights $(\bv, W)$, which can be represented as 
\begin{align}\label{eq:neural_func}
f(z; \bv, W) = \frac{1}{\sqrt{2m}} \sum_{i=1}^{2m} v_i \cdot \act(W_i^\top z),
\end{align}
where $\act$ is the activation function, and $\bv = [v_1, \cdots, v_{2m}]^\top$ and $W = [W_1, W_2,\cdots, W_{2m}]$. Here, we assume that $z=(s,a)$ or $z=(s,a,b)$ with $z\in \cZ$ satisfies $\|z\|_2 = 1$, i.e., $z$ is normalized on a unit hypersphere in $\RR^d$. Let $W^{(0)}$ be the initial value of $W$ and $\bv^{(0)}$ be the initialization of $\bv$. The initialization step for the above model is performed as follows: we let $v_i\sim \text{Unif}(\{-1, 1\})$ and $W_i^{(0)} \sim N(0, I_d/d)$ for all $i \in [m]$, where $I_d$ is an identity matrix in $\RR^{d\times d}$, and $v_i^{(0)} = - v_{i-m}^{(0)}$, $W_i^{(0)} = W_{i-m}^{(0)}$ for all $i \in \{m+1, 2m\}$. Here we let $N(0, I_d/d)$ denote Gaussian distribution. In this paper, we let $\bv$ be fixed as $\bv^{(0)}$ and we only learn $W$ for ease of theoretical analysis. Thus, we represent $f(z;, \bv, W)$ by $f(z; W)$ to simplify the notation. This neural network model is widely studied in recent papers on the analysis of neural networks, e.g., \citet{gao2019convergence,bai2019beyond}. When the model is overparameterized, i.e., $m$ is sufficiently large, we can characterized the dynamics of the training such neural network by neural tangent kernel (NTK) \citep{jacot2018neural}.  Here we define 
\begin{align}\label{eq:neural_grad}
\hspace{-0.18cm}\varphi(z;W) :=[\nabla_{W_1} f(z;W)^\top, \cdots, \nabla_{W_{2m}} f(z;W)^\top]^\top,
\end{align}
where we let $\nabla_{W_i} f(z;W)$ be a column vector such that $\varphi(z;W) \in \RR^{2md}$. Thus, conditioned on the randomness in the initialization of $W$ by $\W0$, we further define the kernel 
\begin{align*}
\ker_m(z,z') = \langle \varphi(z;\W0), \varphi(z';\W0) \rangle, \forall z,z'\in \cZ.
\end{align*}
In addition, we consider a linearization of the model $f(z, W)$ at the initial value $\W0$, such that we have $f_{\lin}(z; W) := f(z; W^{(0)}) + \langle \varphi(z;\W0), W-W^{(0)}\rangle$. Furthermore, the following equation holds: $f_{\lin}(z; W)  = \langle \varphi(z; \W0), W-W^{(0)}\rangle$ since $f(z; W^{(0)}) = 0$ by the initialization scheme. We can see that the linearized function $f_{\lin}(z; W)$ is a function on RKHS with the kernel $\ker_m(z,z')$. When the model is overparameterized with $m\rightarrow \infty$, the kernel $\ker_m(z,z')$ converges to an NTK kernel, which is defined as $\ker_{\ntk} = \EE_{\bomega\sim N(0, I_d/d)} [\act'(\bomega^\top z) \cdot \act'(\bomega^\top z') \cdot z^\top z']$, where $\act'$ is the derivative of the activation function $\act$.

\section{Single-Agent MDP Setting}

In this section, we introduce our method under the single-agent MDP setting with kernel and neural function approximations. Then, we present our theoretical results.

 \begin{algorithm}[t]\caption{Exploration Phase for Single-Agent MDP} 
	\begin{algorithmic}[1]
		\State {\bfseries Initialize:} $\delta > 0$ and $\varepsilon > 0$. 
		\For{episode $k=1,\ldots,K$}   
		   	\State Let $V_{H+1}^{k}(\cdot) = \boldsymbol 0$ and $Q_{H+1}^{k}(\cdot, \cdot) = \boldsymbol 0$
		        \For{step $h=H, H-1,\ldots, 1$} 
					\State \label{line:bonus_explore}Construct bonus term $u_h^k (\cdot, \cdot)$
					\State Compute exploration reward $r_h^k(\cdot, \cdot) = u_h^k(\cdot, \cdot)/H$ 
					\State \label{line:func_explore}Compute approximation function $f_h^k(\cdot, \cdot)$
					\State	\label{line:Q_explore}$Q_h^k(\cdot,\cdot) =  \Pi_{[0, H]} [ (f_h^k + r_h^k+ u_h^k) (\cdot, \cdot)]$
					\State $V_h^k(\cdot) =  \max_{a\in \cA} Q_h^k(\cdot, a)$
					\State $\pi_h^k(\cdot) =  \argmax_{a\in \cA} Q_h^k(\cdot, a)$
	            \EndFor	    				

				\State Take actions following $a_h^k \sim \pi_h^{k}(s_h^k),\  \forall h \in [H]$. 	
    \EndFor        
    \State {\bfseries Return:} $\{(s_h^k, a_h^k) \}_{(h,k)\in [H]\times [K]}$.      
	\end{algorithmic}\label{alg:exploration_phase_single}
\end{algorithm}

\subsection{Kernel Function Approximation}
Our proposed method is composed of the reward-free exploration phase and planning phase with the given extrinsic reward function. The exploration phase and planning phase are summarized in Algorithm \ref{alg:exploration_phase_single} and Algorithm \ref{alg:plan_phase_single}. \citet{}

Specifically, the exploration algorithm is an optimistic variant of the value-iteration algorithm with the function approximation. In Algorithm \ref{alg:exploration_phase_single}, we use $Q_h^k$ and $V_h^k$ to denote the optimistic Q-function and value function for the exploration rewards. During the exploration phase, the agent does not access the true reward function and explore the environment for $K$ episodes based on the policy $\{\pi_h^k\}_{(h,k)\in [H]\times[K]}$ determined by the value function $V_h^k$, and collects the trajectories $\{s_h^k, a_h^k\}_{(h,k)\in [H]\times[K]}$ for the subsequent planning phase. Thus, instead of approximating the Q-function directly, we seek to approximate $\PP_h V^k_{h+1}$ by a clipped function $f_h^k(s, a)$ for any $(s,a)\in \cS\times \cA$, where $f_h^k(\cdot, \cdot)$ is estimated by solving a regularized kernel regression problem as below. Based on this kernel approximation, we construct an associated UCB bonus term $u_h^k$ to facilitate exploration, whose form is specified by the kernel function approximator. Moreover, although the true reward is not available to the agent, to guide the exploration, we construct the exploration reward by scaling the bonus $u_h^k$, guiding the agent to explore state-action pairs with high uncertainties characterized by $u_h^k$. Then, the Q-function $Q_h^k$ is a combination of $r_h^k(s,a)$, $f_h^k(s,a)$, and $u_h^k(s,a)$ as shown in Line \ref{line:Q_explore} of Algorithm \ref{alg:exploration_phase_single}. In this paper, we define a clipping operator as $\Pi_{[0,H]}[x]:= \min\{x, H\}^+ = \min\{\max\{x, 0\}, H\}$. Note that the exploration phase in Algorithm \ref{alg:exploration_phase_single} is not restricted to the kernel case and can be combined with other approximators, e.g., neural networks, as will be shown later.

At the $k$-th episode, given the visited trajectories $\{(s_h^\tau, a_h^\tau)\}_{\tau=1}^{k-1}$, we construct the approximator for each $h\in [H]$ by solving the following regularized kernel regression problem
\begin{align*}
\hat{f}_h^k = 	\min_{f\in \cH}\sum_{\tau=1}^{k-1}[V_{h+1}^k(s_{h+1}^\tau)-f(z_h^\tau)]^2 + \lambda \|f\|_\cH^2,
\end{align*}
where $f(z_h^\tau) = \langle f, \phi(z_h^\tau)\rangle_\cH$ with $z_h^\tau = (s_h^\tau, a_h^\tau)$, and $\lambda$ is a hyperparameter to be determined later. As we will discuss in Lemma \ref{lem:compute_estimate} in the appendix, the closed form solution to the above problem is $\hat{f}_h^k(z) = \langle \hat{f}_h^k, \phi(z)\rangle_\cH = \psi_h^k(z)^\top (\lambda \cdot I +\cK_h^k )^{-1} \yb_h^k$, where we define $\psi_h^k (z) :=  [\ker(z, z_h^1), \cdots, \ker(z, z_h^{k-1})]^\top$, $\yb_h^k := [V_{h+1}^k(s_{h+1}^1), \cdots, V_{h+1}^k(s_{h+1}^{k-1})  ]^\top$, and also $\cK_h^k := [\psi_h^k (z_h^1),\cdots, \psi_h^k (z_h^{k-1})]$ (recalling that $z = (s,a)$).

We let $f_h^k(z) = \Pi_{[0, H]}[\hat{f}_h^k(z)]$ by clipping operation to guarantee $f_h^k(z) \in [0, H]$ such that in Algorithm \ref{alg:exploration_phase_single}, we let
\begin{align}\label{eq:kernel_approx}
f_h^k(z) =  \Pi_{[0,H]}[\psi_h^k(z)^\top (\lambda \cdot I +\cK_h^k )^{-1} \yb_h^k],
\end{align}
In addition, we construct our bonus term inspired by the construction of the UCB bonus for kernelized contextual bandit or RL with kernel function approximation as in \citet{srinivas2009gaussian,valko2013finite,chowdhury2017kernelized,yang2020provably}. Different from the aforementioned papers, due to the reward-free setting, the bonus term $u_h^k(\cdot,\cdot)$ here only quantifies the uncertainty of estimating $\PP_hV_{h+1}^k(\cdot,\cdot)$ with the kernel function approximator. Then, the bonus term is defined as 
\begin{align}\label{eq:kernel_bonus}
u_h^k(z) := \min\{\beta \cdot w_h^k(z), H\} 
\end{align}
where $\beta$ is a hyperparameter to be determined and we set
\begin{align*}
w_h^k(z) = \lambda^{-\frac{1}{2}} [ \ker(z,z)- \psi_h^k(z)^\top (\lambda I + \cK_h^k )^{-1}  \psi_h^k(z)]^{\frac{1}{2}}.
\end{align*}

\begin{algorithm}[t]\caption{Planning Phase for Single-Agent MDP} 
	\begin{algorithmic}[1]
		\State {\bfseries Initialize:} Reward function $\{r_h\}_{h\in [H]}$ and exploration data $\{(s_h^k, a_h^k)\}_{(h,k)\in [H]\times [K]}$
		        		
			\For{step $h=H, H-1, \ldots,1$}   
				\State  \label{line:bonus_plan}Compute bonus term $u_h(\cdot,\cdot)$
 	
			   	\State \label{line:func_plan}Compute approximation function $f_h(\cdot, \cdot)$
					\State $Q_h(\cdot,\cdot) =  \Pi_{[0, H]}[(f_h + r_h + u_h)(\cdot, \cdot)]$	
					\State $V_h(\cdot) =  \max_{a\in \cA} Q_h(\cdot, a)$
					\State $\pi_h(\cdot) =  \argmax_{a\in \cA} Q_h(\cdot, a)$
	            \EndFor	
	\State {\bfseries Return:} $\{\pi_h\}_{h\in[H]}$
	\end{algorithmic}\label{alg:plan_phase_single}
\end{algorithm}

The planning phase can be viewed as a single-episode version of the optimistic value iteration algorithm. Using all the collected trajectories $\{s_h^k, a_h^k\}_{(h,k)\in [H]\times[K]}$, we can similarly construct the approximation of $\PP_h V_{h+1}$ by solving
\begin{align}
\hat{f}_h = 	\argmin_{f\in \cH}\sum_{\tau=1}^K[V_{h+1}(s_{h+1}^\tau)-f(z_h^\tau)]^2 + \lambda \|f\|_\cH^2. \label{eq:kernel_plan_regression}
\end{align}
Thus, the kernel approximation function can be estimated as
\begin{align*}
&f_h(z) = \Pi_{[0,H]} [\hat{f}_h(z)] = \Pi_{[0,H]} [\psi_h(z)^\top (\lambda \cdot I +\cK_h )^{-1} \yb_h],
\end{align*}
and the bonus term is 
\begin{align*}
u_h(z) := \min\{ \beta\cdot w_h(z), H\}
\end{align*}
with setting
\begin{align*}
w_h(z) = \lambda^{-\frac{1}{2}} [ \ker(z,z)- \psi_h(z)^\top (\lambda I + \cK_h )^{-1}  \psi_h(z)]^{\frac{1}{2}},
\end{align*}
where we define $\psi_h (z) :=  [\ker(z, z_h^1), \cdots, \ker(z, z_h^K)]^\top$, $\yb_h := [V_{h+1}(s_{h+1}^1), \cdots, V_{h+1}(s_{h+1}^K)  ]^\top$, and also $\cK_h := [\psi_h (z_h^1),\cdots, \psi_h (z_h^K)]$. 
Given an arbitrary reward function $r_h$, with the kernel approximator $f_h$ and the bonus $u_h$, one can compute the optimistic Q-function $Q_h$ and the associated value function $V_h$. The learned policy $\pi_h$ is obtained by value iteration based on the optimistic Q-function. Algorithm \ref{alg:plan_phase_single} is also a general planning scheme that can be generalized to other function approximator, for example, the neural function approximator.

The combination of Algorithm \ref{alg:exploration_phase_single} and Algorithm \ref{alg:plan_phase_single} can be viewed as a generic framework for different function approximators under the single-agent MDP setting. It  generalizes the reward-free RL method with the linear function approximation studied in the prior work \citet{wang2020reward}.  The linear function approximation can be also viewed as a special case of the kernel function approximation with $\ker(z,z')=\langle \phi(z), \phi(z')\rangle$ where $\phi(z)$ is in Euclidean space. Based on the above framework, we further propose the reward-free algorithm for the single-agent MDP with the neural function approximation in the next subsection.

\begin{remark} Note that in the kernel function approximation setting, we directly define the kernel $\ker(z,z')$ for the algorithms  instead of the feature map $\phi(z)$ which may potentially lie in an infinite dimensional space.
\end{remark}

\subsection{Neural Function Approximation}\label{sec:neural_single}

For the neural function approximation setting, the agent also runs Algorithm \ref{alg:exploration_phase_single} for exploration and Algorithm \ref{alg:plan_phase_single} for planning. Different from the kernel function approximation, in the exploration phase, at the $k$-th episode, given the visitation history $\{s_h^\tau, a_h^\tau\}_{\tau=1}^{k-1}$, we construct the approximation for each $h\in [H]$ by solving the following regularized regression problem
\begin{align}
\begin{aligned}\label{eq:neural_approx_explore}
W_h^k=\argmin_{W\in \RR^{2md}} &\sum_{\tau=1}^{k-1}[V_{h+1}^k(s_{h+1}^\tau) - f(z_h^\tau; W)]^2 + \lambda \|W-\W0\|_2^2, 
\end{aligned}
\end{align}
where we assume that there exists an optimization oracle that can return the global optimizer of the above problem. The initialization of $\W0$ and $\bv^{(0)}$ for the function $f(z; W)$ follows the scheme as we discussed in Section \ref{sec:over_NN}.  As shown in many recent works \citep{du2019gradient,du2018gradient,arora2019fine},  when $m$ is sufficiently large, with random initialization, some common optimizers, e.g., gradient descent, can find the global minimizer of the empirical loss efficiently with a linear convergence rate. Once we obtain $W_h^k$, the approximation function is constructed as $f_h^k(z) = \Pi_{[0,H]} [f(z; \Whk)]$.
Adapting the construction of the UCB bonus for neural contextual bandit and RL with neural function approximation \citep{zhou2020neural,yang2020provably} to our reward-free setting, the corresponding exploration bonus $u_h^k$ in the exploration phase is of the form $u_h^k(z) := \min\{\beta \cdot w_h^k(z),H\}$ where
\begin{align}
w_h^k(z) = [\varphi(z; \Whk)^\top (\Lambda_h^k)^{-1} \varphi(z; \Whk)]^{\frac{1}{2}}. \label{eq:neural_bonus_explore}
\end{align}
Here we define the invertible matrix $\Lambda_h^k := \lambda I_{2md} + \sum_{\tau=1}^{k-1} \varphi(z_h^\tau; \Whk)\varphi(z_h^\tau; \Whk)^\top$with $\varphi(z_h^\tau; W)$ as \eqref{eq:neural_grad}.

In the planning phase, given the collection of trajectories in $K$ episodes of exploration phase, we construct the neural approximation of $\PP_hV_{h+1}(z)$ as solving a least square problem, i.e., $W_h$ is the global optimizer of
\begin{align*}
\hspace{-0.2cm}\min_{W\in \RR^{2md}} \sum_{\tau=1}^{K}[V_{h+1}(s_{h+1}^\tau) - f(z_h^\tau; W)]^2 + \lambda \|W-\W0\|_2^2, 
\end{align*}
such that $f_h(z) = \Pi_{[0,H]} [f(z; W_h)]$. Analogously, the bonus term for the planning phase is of the form $u_h(z) := \min\{\beta \cdot w_h(z),H\}$ where
\begin{align*}
w_h(z) = [\varphi(z; W_h)^\top (\Lambda_h)^{-1} \varphi(z; W_h)]^{\frac{1}{2}}, 
\end{align*}
where we define the invertible matrix $\Lambda_h := \lambda I_{2md} + \sum_{\tau=1}^K \varphi(z_h^\tau; W_h)\varphi(z_h^\tau; W_h)^\top$.

\subsection{Theoretical Results for Single-Agent MDP} \label{sec:result_single}

\noindent\textbf{Kernel Function Approximation.} In this subsection, we first present the result for the kernel function approximation setting. We make the following assumptions.
\begin{assumption}\label{assump:kernel} For any value function $V:\cS \mapsto \RR $, we assume that $\PP_h V (z)$ is in a form of $\langle \phi(z), \wb_h \rangle_\cH$ for some $\wb_h \in \cH$. In addition, we assume there exists a fixed constant $R_Q$ such that $\|\wb_h\|_\cH \leq R_Q H$.
\end{assumption}

One example for this assumption is that the transition model is in a form of $\PP_h(s'|z) = \langle \phi(z), \wb'_h(s') \rangle_\cH$ such that $\PP_h V(z) = \int_{\cS} V_{h+1}(s') \langle \phi(z), \wb'_h(s') \rangle_\cH \mathrm{d}s'$ where we can write $\wb_h = \allowbreak\int_{\cS} V_{h+1}(s') \wb'_h(s') \mathrm{d}s'$. This example can be viewed as a generalization of the linear transition model \citep{jin2020provably} to the RKHS.

In our work, we use maximal information gain \citep{srinivas2009gaussian} to measure the function space complexity, i.e., 
\begin{align*}
\Gamma(\mathfrak{C}, \sigma; \ker) = \sup_{\cD\subseteq \cZ} 1/2\cdot  \log\det (I +  \cK_\cD/\sigma),
\end{align*}
where the supremum is taken over all possible sample sets $\cD \subseteq\cZ$ with $|\cD| \leq \mathfrak{C}$, and $\cK_\cD$ is the Gram matrix induced by $\cD$ based on some kernel $\ker$ of RKHS. The value of $\Gamma(\mathfrak{C}, \sigma; \ker)$ reflects how fast the the eigenvalues of $\cH$ decay to zero and can be viewed as a proxy of the dimension of $\cH$ when $\cH$ is infinite-dimensional. To characterize the complexity, we define a Q-function class $\overline{\cQ}$ of the form
\begin{align} \label{eq:Q_func_class}
\overline{\cQ}(c, R, B) = \{Q: Q \text{ satisfies the form of } {Q}^{\sharp} \}.
\end{align}
where we define ${Q}^{\sharp}$ in the following form $
{Q}^{\sharp}(z) =\min\{ c(z) + \Pi_{[0,H]} [\langle \wb, \phi(z)\rangle_\cH] +  g(z) , H\}^+$ with some $\wb$ satisfying $\|\wb\|_\cH \leq R$, $\|\phi(z)\|_\cH \leq 1$, and also $g(z) =B\cdot \min\{  \|\phi(z)\|_{\Lambda_{\cD}^{-1}} , H/\beta\}^+$. Here $\Lambda_{\cD}$ is an adjoint operator with the form $\Lambda_{\cD} = \lambda I_\cH + \sum_{ z'\in \cD} \phi(z') \phi(z')^\top$ with $I_\cH$ denoting identity mapping on $\cH$ and $\cD\subseteq \cZ$ with $|\cD| \leq K$. Here we define the $\varsigma$-covering number of the class $\overline{\cQ}$ w.r.t. the $\ell_\infty$-norm as $\overline{\cN}_\infty(\varsigma; R, B)$ with an upper bound $\cN_\infty(\varsigma; R, B)$. As formally discussed in Section \ref{sec:convering} of the appendix, we compute the covering number upper bound $\cN_\infty(\varsigma; R, B)$. As we can see in Algorithms \ref{alg:exploration_phase_single} and \ref{alg:plan_phase_single}, we have $Q_h^k \in \overline{\cQ}(\bm{0}, R, (1+1/H)\beta)$ and $Q_h \in \overline{\cQ}(r_h, R', \beta)$ for some $R$ and $R'$. Based on the above assumptions and definitions, we have the following result.

\begin{theorem}\label{thm:main_kernel_single} Suppose that $\beta$ satisfies the condition $16H^2\big[ R^2_Q  +\log\cN_{\infty}(\varsigma^*; R_K,  2\beta) + 2 \Gamma(K, \lambda; \ker) \allowbreak + 6\log(2KH) + 5  \big] \leq  \beta^2$. Under the kernel function approximation setting with a kernel $\ker$, letting $\lambda = 1+1/K$, $R_K=2H\sqrt{\Gamma(K, \lambda; \ker)}$, and $\varsigma^* = H/K$, with probability at least $1-(2K^2H^2)^{-1}$, the policy generated via Algorithm \ref{alg:plan_phase_single} satisfies  $V^*_1(s_1, r) - V^\pi_1(s_1, r)  \leq \cO(\beta \sqrt{H^4 [ \Gamma(K, \lambda; \ker) + \log(KH)]}/\sqrt{K})$, after exploration for $K$ episodes with Algorithm \ref{alg:exploration_phase_single}.
\end{theorem}
The covering number $\cN_{\infty}(\varsigma^*; R_K,  2\beta)$ and the information gain $\Gamma(K, \lambda; \ker)$ reflect the function class complexity. To understand the result in Theorem \ref{thm:main_kernel_single}, we consider kernels $\ker$ with two different types of eigenvalue decay conditions: (i) $\gamma$-finite spectrum (where $\gamma \in \ZZ_+$) and (ii) $\gamma$-exponential spectral decay (where $\gamma >0$). 

For the case of $\gamma$-finite spectrum, we have $\beta = \cO(\gamma H\sqrt{\log(\gamma KH)})$, $\log\cN_{\infty}(\varsigma^*; R_K,  2\beta) \allowbreak = \cO(\gamma^2 \log(\gamma KH))$, and $\Gamma(K, \lambda; \ker) = \cO(\gamma \log K)$, which further implies that to achieve $V^*_1(s_1, r) - V^\pi_1(s_1, r) \leq \varepsilon$, it requires $\tilde{\cO}(H^6 \gamma^3 /\varepsilon^2)$ rounds of exploration, where $\tilde{\cO}$ hides the logarithmic dependence on $\gamma$ and $1/\varepsilon$. 

Therefore, when the problem reduces to the setting of linear function approximation, the above result becomes $\tilde{\cO}(H^6 \mathfrak{d}^3 /\varepsilon^2)$ by letting $\gamma = \mathfrak{d}$, where $\mathfrak{d}$ is the feature dimension. This is consistent with the result in \citet{wang2020reward}, which studies the linear approximation setting for reward-free RL. Furthermore, the sample complexity becomes $\tilde{\cO}(H^6 |\cS|^3|\cA|^3 /\varepsilon^2)$ by setting $\gamma = |\cS||\cA|$, when the problem reduces to the tabular setting.

For the case of $\gamma$-exponential spectral decay with $\gamma > 0$, we have $\log\cN_{\infty}(\varsigma^*; R_K,  2\beta) = \cO((\log K)^{1+2/\gamma}+(\log\log H)^{1+2/\gamma})$, $\beta =  \cO( H\sqrt{\log( KH)} (\log K)^{1/\gamma})$, $\Gamma(K, \lambda; \ker) = \cO( (\log K)^{1+1/\gamma})$. Therefore, to obtain an $\varepsilon$-suboptimal policy, it requires $\cO(H^6 C_\gamma  \cdot  \allowbreak \log^{4+6/\gamma}(\varepsilon^{-1}) /\varepsilon^2) = \tilde{\cO}(H^6 C_\gamma /\varepsilon^2)$ rounds  of exploration, where $C_\gamma$ is some constant depending on $1/\gamma$. Please see Section \ref{sec:convering} for detailed definitions and discussions.

\vspace{5pt}

\noindent\textbf{Neural Function Approximation.} Next, we present the result for the neural function approximation setting. 
\begin{assumption}\label{assump:neural} For any value function $V$, we assume that $\PP_h V(z)$ can be represented as $\PP_h V(z) = \int_{\RR^d}  \act'(\boldsymbol{\omega}^\top z) \cdot z^\top \balpha_h(\boldsymbol\omega) \mathrm{d} p_0(\boldsymbol\omega)$ for some $\balpha_h(\boldsymbol\omega)$ with $\balpha : \RR^d \mapsto \RR^d$ and $\sup_{\bomega}\|\balpha(\bomega)\|\leq R_Q H /\sqrt{d}$. Here $p_0$ is the density of Gaussian distribution $N(0, I_d/d)$.
\end{assumption}
As discussed in \citet{gao2019convergence,yang2020provably}, the function class characterized by $f(z)=\int_{\RR^d}  \act'(\boldsymbol{\omega}^\top z) \cdot z^\top \balpha_h(\boldsymbol\omega) \mathrm{d} p_0(\boldsymbol\omega)$ is an expressive subset of RKHS $\cH$. One example is that the transition model can be written as $\PP_h(s'|z) = \int_{\RR^d}  \act'(\boldsymbol{\omega}^\top z) \cdot z^\top \balpha'_h(\boldsymbol\omega; s') \mathrm{d} p_0(\boldsymbol\omega)$ such that we have $\balpha_h(\boldsymbol\omega)  = \int_\cS \balpha'_h(\boldsymbol\omega; s') V_{h+1}(s') \mathrm{d}s'$. This example also generalizes the linear transition model \citep{jin2020provably} to the overparameterized neural network setting. Similar to \eqref{eq:Q_func_class}, we also define a Q-function class based on a normalized version of $\varphi(z, \W0)$,  which further can be analyzed using the same notations $\overline{\cQ}$ and $\cN_\infty$ (See Lemma \ref{lem:bonus_concentrate_neural} for details).

\begin{theorem}\label{thm:main_neural_single} Suppose that $\beta$ satisfies the condition that $8H^2 [R_Q^2 (1+\sqrt{\lambda/d})^2 + 4 \Gamma(K, \lambda; \ker_m) + 10+4\log\cN_{\infty}(\varsigma^*; R_K,  2\beta) + 12\log(2KH)] \leq \beta^2$ with $m=\Omega(K^{19} H^{14} \log^3 m)$. Under the overparameterized neural function approximation setting, letting $\lambda = C(1+1/K)$ for some constant $C\geq 1$, $R_K=H\sqrt{K}$, and $\varsigma^* = H/K$, with probability at least $1-(2K^2H^2)^{-1}-4m^{-2}$, the policy generated via Algorithm \ref{alg:plan_phase_single} satisfies $V^*_1(s_1, r) - V^\pi_1(s_1, r)  \leq \cO(\beta \sqrt{H^4  [\Gamma(K, \lambda; \ker_m)+\log(KH)]}/\sqrt{K} + H^2\beta\iota)$ with $\iota = 5K^{7/12}H^{1/6} m^{-1/12} \log^{1/4} m $, after exploration for $K$ episodes with Algorithm \ref{alg:exploration_phase_single}.
\end{theorem}
 
In Theorem \ref{thm:main_neural_single}, there is an error term $H^2\beta\iota$ that depends on $m^{-1/12}$. In the regime of overparameterization, when $m$ is sufficiently large, this term can be extremely small and $\iota \rightarrow 0, \ker_m\rightarrow \ker_\ntk$ if $m\rightarrow \infty$. Here $\Gamma(K, \lambda; \ker_m)$ and $\cN_{\infty}(\varsigma^*; R_K,  2\beta)$ characterize the intrinsic complexity of the function class. In particular, when $m$ is large, the overparamterized neural function setting can be viewed as a special case of RKHS with a misspecification error. If the eigenvalues of the kernel $\ker_m$ satisfy finite spectrum or exponential spectral decay, we know that $\beta$, $\Gamma(K, \lambda; \ker_m)$, and $\log \cN_{\infty}(\varsigma^*; R_K,  2\beta)$ are of the same orders to the ones in the discussion after Theorem \ref{thm:main_kernel_single}. Moreover, if $m$ is sufficiently large such that $H^2\beta\iota\leq \varepsilon$,  we obtain an $\tilde{\cO}(1/\varepsilon^2)$ sample complexity to achieve an $\cO(\varepsilon)$-suboptimal policy.

Overall, the above results show that with the kernel function approximation and overparameterized neural function approximation, Algorithms \ref{alg:exploration_phase_game} and \ref{alg:plan_phase_game} guarantee $\tilde{\cO}(1/\varepsilon^2)$ sample complexity for achieving $\varepsilon$-suboptimal policy, which matches existing $\tilde{\cO}(1/\varepsilon^2)$ results for the single-agent MDP for the tabular case or with linear function approximation in terms of $\varepsilon$.

\section{Zero-Sum Markov Game Setting}

In this section, we introduce the algorithms under the Markov game setting with kernel and neural function approximations. We further present their theoretical results on the sample complexity.

\subsection{Kernel Function Approximation}

The exploration phase and planning phase for the zero-sum game are summarized in Algorithm \ref{alg:exploration_phase_game} and Algorithm \ref{alg:plan_phase_game}. 

Specifically, in the exploration phase, the exploration policies for both players are obtained by taking maximum on Q-function over both action spaces. Thus, Algorithm \ref{alg:exploration_phase_game} in essence is an extension of Algorithm \ref{alg:exploration_phase_single} and performs the same exploration step, if we view the pair $(a,b)$ as an action $\ba := (a,b)$ on the action space $\cA\times \cB$ and regard the exploration policy pair $(\pi_h^k(s), \nu_h^k(s))$ as a product policy $(\pi_h^k \otimes \nu_h^k)(s)$. Thus, the approximator $f_h^k(z)$ and the bonus term $u_h^k(z)$ share the same forms as \eqref{eq:kernel_approx} and \eqref{eq:kernel_bonus} if we slightly abuse the notation by letting $z = (s,a,b)$.

 \begin{algorithm}[t]\caption{Exploration Phase for Zero-Sum Markov Game} 
	\begin{algorithmic}[1]
		\State {\bfseries Initialize:} $\delta > 0$ and $\varepsilon > 0$. 
		\For{episode $k=1,\ldots,K$}   
		   	\State Let $V_{H+1}^{k}(\cdot) = \boldsymbol 0$ and $Q_{H+1}^{k}(\cdot, \cdot, \cdot) = \boldsymbol 0$
		        \For{step $h=H, H-1,\ldots, 1$} 
					\State \label{line:bonus_explore_game}Construct bonus term $u_h^k (\cdot, \cdot, \cdot)$
					\State Exploration reward $r_h^k(\cdot, \cdot, \cdot) = u_h^k(\cdot, \cdot, \cdot)/H$ 
					\State \label{line:func_explore_game}Compute approximation function $f_h^k(\cdot, \cdot, \cdot)$
					\State	$Q_h^k(\cdot,\cdot,\cdot) =  \Pi_{[0, H]} [(f_h^k + r_h^k + u_h^k) (\cdot, \cdot,\cdot)]$
					\State $V_h^k(\cdot) =  \max_{a\in \cA, b\in \cB} Q_h^k(\cdot, a, b)$
					\State $(\pi_h^k(\cdot), \nu_h^k(\cdot)) =  \argmax_{a\in \cA, b\in \cB} Q_h^k(\cdot, a, b)$
	            \EndFor	    				

				\State Take actions following $a_h^k \sim \pi_h^{k}(s_h^k)$ and also $b_h^k \sim \nu_h^{k}(s_h^k),\forall h \in [H]$	
    \EndFor        
    \State {\bfseries Return:} $\{(s_h^k, a_h^k, u_h^k) \}_{(h,k)\in [H]\times [K]}$
	\end{algorithmic}\label{alg:exploration_phase_game}
\end{algorithm}

In the planning phase, the algorithm generates the policies for two players in a separate manner. While maintaining two optimistic Q-functions,  their policies are generated by finding NE of two games with payoff matrices $\overline{Q}_h$ and $\underline{Q}_h$ respectively, i.e., 
$(\pi_h(s), \overline{D}_0(s))$ is the solution to 
\begin{align*}
\max_{\pi'}\min_{\nu'} \EE_{a \sim \pi', b\sim\nu'} [\overline{Q}_h(s, a ,b)],
\end{align*}
and $(\underline{D}_0(s), \nu_h(s))$ is the solution to 
\begin{align*}
\max_{\pi'}\min_{\nu'} \EE_{a \sim \pi', b\sim\nu'}  [\underline{Q}_h(s, a, b)],
\end{align*}
which can be solved efficiently in computation by many existing min-max optimization algorithms (e.g., \citet{koller1994fast}). 

Moreover, we construct the approximation functions for Player 1 and Player 2 similarly via \eqref{eq:kernel_plan_regression} by letting $z = (s,a,b)$ and placing the value function with $\overline{V}$ and $\underline{V}$ separately such that we have 
\begin{align*}
\overline{f}_h(z) =\Pi_{[0, H]} [\psi_h(z)^\top (\lambda \cdot I +\cK_h )^{-1} \overline{\yb}_h],\\
\underline{f}_h(z) = \Pi_{[0, H]} [\psi_h(z)^\top (\lambda \cdot I +\cK_h )^{-1} \underline{\yb}_h],
\end{align*}
where $\overline{\yb}_h := [\overline{V}_{h+1}(s_{h+1}^1), \cdots, \overline{V}_{h+1}(s_{h+1}^K)  ]^\top$ and $\underline{\yb}_h := [\underline{V}_{h+1}(s_{h+1}^1), \cdots, \underline{V}_{h+1}(s_{h+1}^K)  ]^\top$. Then, for the bonus term, Players 1 and 2 share the one of the same form, i.e., $\overline{u}_h(z) = \underline{u}_h(z) := u_h(z) = \min\{ \beta\cdot w_h(z), H\}$ with
\begin{align*}
w_h(z) = \lambda^{-\frac{1}{2}} [ \ker(z,z)- \psi_h(z)^\top (\lambda I + \cK_h )^{-1}  \psi_h(z)]^{\frac{1}{2}}.
\end{align*}

The combination of Algorithm \ref{alg:exploration_phase_game} and Algorithm \ref{alg:plan_phase_game} is a generic framework for distinct function approximators under the zero-sum Markov game setting. Based on the above framework, we further propose the reward-free algorithm for the zero-sum Markov game with the neural function approximation in the next subsection.

\subsection{Neural Function Approximation}

For the neural function approximation, the exploration and planning phases follow Algorithm \ref{alg:exploration_phase_game} and \ref{alg:plan_phase_game}. In the exploration phase, following the same discussion for the exploration algorithm with kernel function approximation, Algorithm \ref{alg:exploration_phase_game} with the neural approximator is intrinsically the same as Algorithm \ref{alg:exploration_phase_single}. Thus, one can follow the same approaches to construct the neural function approximator $f_h^k(z) =\Pi_{[0,H]}  [f(z; W_h^k)]$ and the bonus $u_h^k(z)$ as in \eqref{eq:neural_approx_explore} and \eqref{eq:neural_bonus_explore} with only letting $z = (s,a,b)$.

For the planning phase (Algorithm \ref{alg:plan_phase_game}), letting $z=(s,a,b)$, we construct approximation functions separately for Player 1 and Player 2 via solving two regression problems
\begin{align*}
\overline{W}_h = \argmin_{W\in \RR^{2md}} \sum_{\tau=1}^{K}[\overline{V}_{h+1}(s_{h+1}^\tau) - f(z_h^\tau; W)]^2 + \lambda \|W-\W0\|_2^2, \\
\underline{W}_h =\argmin_{W\in \RR^{2md}} \sum_{\tau=1}^{K}[\underline{V}_{h+1}(s_{h+1}^\tau) - f(z_h^\tau; W)]^2 + \lambda \|W-\W0\|_2^2, 
\end{align*}
such that we let $\overline{f}_h(z) = \Pi_{[0,H]} [f(z; \overline{W}_h)]$ and $\underline{f}_h(z) = \Pi_{[0,H]} [f(z; \underline{W}_h)]$.
The bonus terms $\overline{u}_h$ and $\underline{u}_h$ for Players 1 and 2 are $\overline{u}_h(z) := \min\{\beta \cdot \overline{w}_h(z),H\}$ and $\underline{u}_h(z) := \min\{\beta \cdot \underline{w}_h(z),H\}$ with
\begin{align*}
\overline{w}_h(z) = [\varphi(z; \overline{W}_h)^\top (\overline{\Lambda}_h)^{-1} \varphi(z; \overline{W}_h)]^{\frac{1}{2}}, \\
\underline{w}_h(z) = [\varphi(z; \underline{W}_h)^\top (\underline{\Lambda}_h)^{-1} \varphi(z; \underline{W}_h)]^{\frac{1}{2}}, 
\end{align*}
where we define the invertible matrices $\overline{\Lambda}_h := \lambda I_{2md} + \sum_{\tau=1}^K \varphi(z_h^\tau; \overline{W}_h)\varphi(z_h^\tau; \overline{W}_h)^\top$ and $\underline{\Lambda}_h := \lambda I_{2md} + \sum_{\tau=1}^K \varphi(z_h^\tau; \underline{W}_h)\varphi(z_h^\tau; \underline{W}_h)^\top$.

\begin{algorithm}[t]\caption{Planning Phase for Zero-Sum Markov Game} 
    \setstretch{1.1}
	\begin{algorithmic}[1]
		\State {\bfseries Initialize:} Reward function $\{r_h\}_{h\in [H]}$ and exploration data $\{(s_h^k, a_h^k, u_h^k)\}_{(h,k)\in [H]\times [K]}$
		        		
			\For{step $h=H, H-1, \ldots,1$}   
				\State  \label{line:bonus_plan_game}Compute bonus term $\overline{u}_h(\cdot,\cdot,\cdot)$ and $\underline{u}_h(\cdot,\cdot,\cdot)$
 	
			   	\State \label{line:func_plan_game}Compute approximations $\overline{f}_h(\cdot, \cdot, \cdot)$ and$\underline{f}_h(\cdot, \cdot, \cdot)$
					\State $\overline{Q}_h(\cdot, \cdot, \cdot) =  \Pi_{[0, H]}[(\overline{f}_h+ r_h + \overline{u}_h)(\cdot, \cdot, \cdot)]$
					\State $\underline{Q}_h(\cdot, \cdot, \cdot) =  \Pi_{[0, H]}[(\underline{f}_h + r_h - \underline{u}_h)(\cdot, \cdot, \cdot)]$
					\State Let $(\pi_h(s), \overline{D}_0(s))$ be NE for $\overline{Q}_h(s, \cdot, \cdot)$, $\forall s\in \cS$
					\State Let $(\underline{D}_0(s), \nu_h(s))$ be NE for $\underline{Q}_h(s, \cdot, \cdot)$, $\forall s\in \cS$	
					
					\State $\overline{V}_h(s) =  \EE_{a\sim\pi_h(s), b\sim \overline{D}_0(s)}[\overline{Q}_h(s, a, b)]$, $\forall s\in \cS$
					\State $\underline{V}_h(s) =  \EE_{a\sim\underline{D}_0(s), b\sim\nu_h(s)} [\underline{Q}_h(s, a, b)]$, $\forall s\in \cS$
	            \EndFor	
	\State {\bfseries Return:} $\{\pi_h\}_{h\in[H]}, \{\nu_h\}_{h\in[H]}$
	\end{algorithmic}\label{alg:plan_phase_game}
\end{algorithm}

\subsection{Theoretical Results for Zero-Sum Markov Game}

In this subsection, we present the results for the zero-sum Markov game setting. Particularly, we make the same assumptions as in Section \ref{sec:result_single} with only letting $z = (s,a,b)$. Moreover, we also use the same Q-function class $\overline{\cQ}$ as \eqref{eq:Q_func_class}, such that we can see in Algorithms \ref{alg:exploration_phase_game} and  \ref{alg:plan_phase_game},  $Q_h^k \in \overline{\cQ}(\bm{0}, R, (1+1/H)\beta)$ for some $R$, and $\overline{Q}_h\in \overline{\cQ}(r_h, R', \beta)$ for some $R'$. To characterize the space which $\underline{Q}_h$ lies in, we define a specific Q-function class $\underline{\cQ}$ of the form
\begin{align} \label{eq:Q_func_class_neural}
\underline{\cQ}(c, R, B) = \{Q: Q \text{ satisfies the form of } Q^\flat\},
\end{align}
\noindent where ${Q}^{\flat}(z) =\min\{ c(z) + \Pi_{[0,H]} [\langle \wb, \phi(z)\rangle_\cH] -  g(z) , H\}^+$ for
some $\wb$ satisfying $\|\wb\|_\cH \leq R$ and also $g(z) =B\cdot \max\{  \|\phi(z)\|_{\Lambda_{\cD}^{-1}} , H/\beta\}^+$. Thus, we have $\underline{Q}_h \in \underline{\cQ}(r_h, R', \beta)$. As we show in Section \ref{sec:convering}, $\overline{\cQ}(c, R, B)$ and $\underline{\cQ}(c, R, B)$ have the same covering number upper bound w.r.t $\|\cdot\|_\infty$. Then, we can use the same notation $\cN_\infty$ to denote such upper bound.
Thus, we have the following result for kernel approximation.
\begin{theorem}\label{thm:main_kernel_game} Suppose that $\beta$ satisfies the condition $16H^2\big[ R^2_Q  +\log\cN_{\infty}(\varsigma^*; R_K,  2\beta)  + 2 \Gamma(K, \lambda; \ker) + 6\log(4KH) + 5  \big] \leq  \beta^2$. Under the kernel function approximation setting with a kernel $\ker$, letting $\lambda = 1+1/K$, $R_K=2H\sqrt{\Gamma(K, \lambda; \ker)}$, and $\varsigma^* = H/K$,  with probability at least $1-(2K^2H^2)^{-1}$,  the policy pair generated via Algorithm \ref{alg:plan_phase_game} satisfies $V_1^{\bre(\nu), \nu}(s_1,r) - V_1^{\pi, \bre(\pi)}(s_1,r) \leq \cO(\beta \sqrt{ H^4 [\Gamma(K, \lambda; \ker) + \log(KH)]}/\sqrt{K})$, after exploration for $K$ episodes with Algorithm \ref{alg:exploration_phase_game}.
\end{theorem}
We further obtain the result for the neural function approximation scenario.
\begin{theorem}\label{thm:main_neural_game} Suppose that $\beta$ satisfies the condition that $8H^2[10 + 12\log(4K /\delta) + R_Q^2 (1+\sqrt{\lambda/d})^2 +4\log\cN_{\infty}(\varsigma^*; R_K,  2\beta)  + 4 \Gamma(K, \lambda; \ker_m)]  \leq \beta^2$ with $m=\Omega(K^{19} H^{14} \log^3 m)$. Under the overparameterized neural function approximation setting, letting $\lambda = C(1+1/K)$ for some constant $C\geq 1$, $R_K=H\sqrt{K}$, and $\varsigma^* = H/K$, with probability at least $1-(2K^2H^2)^{-1}-4m^{-2}$,  the policy pair generated via Algorithm \ref{alg:plan_phase_game} satisfies $V_1^{\bre(\nu), \nu}(s_1,r) - V_1^{\pi, \bre(\pi)}(s_1,r) \leq \cO(\beta \sqrt{ H^4  [\Gamma(K, \lambda; \ker_m) + \log (KH)]}/\sqrt{K} +  H^2 \beta \iota)$ with $\iota = 5K^{7/12}H^{1/6} m^{-1/12} \log^{1/4} m $, after exploration for $K$ episodes with Algorithm \ref{alg:exploration_phase_game}.
\end{theorem}

Following the same discussion as in Section \ref{sec:result_single}, the above results show that with the kernel function approximation and overparameterized neural function approximation, Algorithms \ref{alg:exploration_phase_game} and \ref{alg:plan_phase_game} guarantee an $\tilde{\cO}(1/\varepsilon^2)$ sample complexity to achieve an $\varepsilon$-approximate NE. In particular, when our problem reduces to the Markov game with linear function approximation, the algorithm requires $\tilde{\cO}(H^6 \mathfrak{d}^3 /\varepsilon^2)$ sample complexity to achieve an $\varepsilon$-approximate NE, where $\mathfrak{d}$ is the feature dimension. This also complements the result of the reward-free RL for the Markov game with the linear function approximation. For the tabular case, \citet{bai2020provable} gives an $\tilde{\cO}(H^5|\cS|^2 |\cA| |\cB|/\varepsilon^2)$ sample complexity and \citet{liu2020sharp} gives an $\tilde{\cO}(H^4|\cS| |\cA| |\cB|/\varepsilon^2)$ sample complexity. Our analysis gives an  $\tilde{\cO}(H^6|\cS|^3 |\cA|^3 |\cB|^3/\varepsilon^2)$ sample complexity by simply letting $\mathfrak{d} = |\cS| |\cA| |\cB|$, which matches the existing results in terms of $\varepsilon$. Though the dependence on $H,|\cS|, |\cA|, |\cB|$ is not as tight as existing results, our work presents a more general analysis for the function approximation setting which is not fully studied in previous works.

\section{Theoretical Analysis}

\subsection{Proof Sketches of Theorem \ref{thm:main_kernel_single} and Theorem \ref{thm:main_neural_single} }

We first show the proof sketch for Theorem \ref{thm:main_kernel_single}. Our goal is to bound the term $V_1^*(s_1, r) - V_1^\pi(s_1, r)$. By the optimistic updating rule in the planning phase, according to Lemma \ref{lem:bonus_plan}, we have $V_1^*(s_1, r) \leq V_1(s_1) $ such that $V_1^*(s_1, r) - V_1^\pi(s_1, r) \leq V_1(s_1) - V_1^\pi(s_1, r)$. Then we only need to consider bounding $V_1(s_1) - V_1^\pi(s_1, r)$. Further by this lemma, for any $h\in [H]$, we have
\begin{align}
\begin{aligned} \label{eq:proof_kernel_diff}
V_h(s) - V_h^\pi(s, r) &\leq  r_h(s,\pi_h(s)) + \PP_h V_{h+1}(s,\pi_h(s)) + 2u_h(s,\pi_h(s))  - Q_h^\pi(s, \pi_h(s), r)\\
&= \PP_h V_{h+1} (s,\pi_h(s)) - \PP_h V_{h+1}^\pi(s,\pi_h(s), r) + 2u_h(s,\pi_h(s)).
\end{aligned}
\end{align}
where we use the fact that $ Q_h^\pi(s, \pi_h(s), r) = r_h(s,\pi_h(s)) + \PP_h V_{h+1}^\pi(s,\pi_h(s), r)$.
Recursively applying the above inequality and also using $V_{H+1}^\pi(s, r) = V_{H+1} (s)  = 0$ give
\begin{align*}
&V_1(s_1) - V_1^\pi(s_1,r) \leq  \EE_{\PP}[\textstyle{\sum}_{h=1}^H   2u_h(s_h,\pi_h(s_h))| s_1 ]=2H \cdot V_1^\pi(s_1, u/H). 
\end{align*}
Moreover, by Lemma \ref{lem:explore_plan_connect}, we build a connection between the exploration and planing phase, which is $ V_1^\pi(s_1, u/H) \leq K^{-1} \sum_{k=1}^K V_1^*(s_1, r^k)$. Therefore, combining the above results together, we eventually obtain
\begin{align*}
&V_1^*(s_1, r) - V_1^\pi(s_1, r) \leq 2H/K\cdot  \textstyle{\sum}_{k=1}^K V_1^*(s_1, r^k) \leq \cO\big( \beta\sqrt{H^4 [\Gamma(K, \lambda; \ker)+\log(KH)]} /\sqrt{K} \big),
\end{align*}
where the last inequality is by Lemma \ref{lem:bonus_explore} and the fact that $\beta \geq H$. This completes the proof of Theorem \ref{thm:main_kernel_single}. The proof of this theorem is inspired by \citet{wang2020reward} for the linear function approximation setting and is a non-trivial generalization to the nonlinear kernel function approximation scenario here. Please see detailed proof in Section \ref{sec:proof_main_kernel_single}.

Next, we show the proof sketches of Theorem \ref{thm:main_neural_single}. By Lemma \ref{lem:bonus_plan_neural}, we have $V_1^*(s_1, r)  \leq V_1(s_1) + H\beta\iota$ by optimism, such that $V_1^*(s_1, r) - V_1^\pi(s_1, r) \leq V_1(s_1) - V_1^\pi(s_1, r) + H\beta\iota$. Note that different from the proof of Theorem \ref{thm:main_kernel_single}, there is an extra bias term $H\beta\iota$ introduced by the neural function approximation. Further by Lemma \ref{lem:bonus_plan_neural}, and using the same argument as \eqref{eq:proof_kernel_diff}, we have
\begin{align*}
&V_h(s) - V_h^\pi(s, r) \leq 2u_h(s,\pi_h(s)) +\beta\iota  + \PP_h V_{h+1} (s,\pi_h(s)) - \PP_h V_{h+1}^\pi(s,\pi_h(s), r),
\end{align*}
which introducing another bias $\beta\iota $. Recursively applying the above inequality with $V_{H+1}^\pi(s,r) = V_{H+1} (s)  = 0$ gives
\begin{align*}
V_1(s_1) - V_1^\pi(s_1,r) =2H \cdot V_1^\pi(s_1, u/H) + H\beta\iota. 
\end{align*}
Thus, with Lemma \ref{lem:explore_plan_connect_neural} connecting the exploration and planning phases such that $V_1^\pi(s_1, u/H) \allowbreak\leq K^{-1} \sum_{k=1}^K V_1^*(s_1, r^k) + 2\beta\iota$, combining all the above results eventually yields
\begin{align*}
V_1^*(s_1, r) - V_1^\pi(s_1, r) & \leq 2H/K\cdot \textstyle{\sum}_{k=1}^K V_1^*(s_1, r^k) + 4H\beta\iota\\
&\leq \cO\Big(\beta\sqrt{H^4 [\Gamma(K, \lambda; \ker_m) + \log(KH)]} \big/\sqrt{K} + H^2 \beta \iota \Big),
\end{align*}
where the second inequality and the last inequality is by Lemma \ref{lem:bonus_explore_neural} and the fact that $\beta \geq H$. This completes the proof. Please see detailed proof in Section \ref{sec:proof_main_neural_single}.

\subsection{Proof Sketches of Theorem \ref{thm:main_kernel_game} and Theorem \ref{thm:main_neural_game} }

In the proofs of Theorem \ref{thm:main_kernel_game} and Theorem \ref{thm:main_neural_game} and the corresponding lemmas, to simply the notations, we let $\EE_{a\sim \pi_h, b\sim \nu_h, s'\sim \PP_h}$ denote the expectation with $a\sim \pi_h(s), b\sim \nu_h(s), s'\sim \PP_h(\cdot|s,a,b)$ given the current state $s$ and arbitrary policies $\pi_h$, $\nu_h$ at the $h$-th step.

For the proof sketch of  Theorem \ref{thm:main_kernel_game}, we decompose $V_1^{\bre(\nu), \nu}(s_1, r) - V_1^{\pi, \bre(\pi)}(s_1, r)$ into two terms $V_1^\dagger(s_1, r) - V_1^{\pi, \bre(\pi)}(s_1, r)$ and $V_1^{\bre(\nu), \nu}(s_1, r)  - V_1^\dagger(s_1, r)$ and bound them separately. To bound the first term, by Lemma \ref{lem:bonus_plan_game}, we have $V_1^\dagger(s_1, r) - V_1^{\pi, \bre(\pi)}(s_1, r) \leq \overline{V}_1(s_1) - V_1^{\pi, \bre(\pi)}(s_1, r)$. Note that by the updating rule for $\overline{V}_h$ in Algorithm \ref{alg:plan_phase_game}, we have
\begin{align*}
\overline{V}_h(s) &= \min_{\nu'} \EE_{a\sim \pi_h, b\sim\nu'}[\overline{Q}_h(s,a,b)] \leq  \EE_{a\sim \pi_h, b\sim \bre(\pi)_h}[\overline{Q}_h(s,a,b)],
\end{align*}
such that further by Lemma \ref{lem:bonus_plan_game}, there is
\begin{align*}
&\overline{V}_h(s_h) - V_h^{\pi, \bre(\pi)}(s_h, r) \\
&\qquad\leq  \EE_{a_h\sim \pi_h, b_h\sim \bre(\pi)_h} [(\PP_h\overline{V}_{h+1}  + r_h + 2u_h)(s_h,a_h,b_h)]- V_h^{\pi, \bre(\pi)}(s_h, r) \\
&\qquad= \EE_{a_h\sim \pi_h, b_h\sim \bre(\pi)_h, s_{h+1}\sim\PP_h} [ \overline{V}_{h+1}(s_{h+1}) -  V_{h+1}^{\pi, \bre(\pi)}(s_{h+1}, r) + 2u_h(s_h,a_h,b_h) ],
\end{align*}
where the equality uses $V_h^{\pi, \bre(\pi)}(s_h, r) = \EE_{a_h\sim \pi_h, b_h\sim \bre(\pi)_h} [r_h(s_h,a_h, b_h) + \PP_h V_{h+1}^{\pi, \bre(\pi)}(s_h,a_h,b_h, r) ]$. Recursively applying the above inequality yields
\begin{align*}
\overline{V}_1(s_1) - V_1^{\pi, \bre(\pi)}(s_1,r) & \leq  \EE_{\pi, \bre(\pi), \PP}[\textstyle{\sum}_{h=1}^H 2u_h(s_h, a_h, b_h)| s_1 ] \\
&\qquad =2H \cdot V_1^{\pi, \bre(\pi)}(s_1, u/H). 
\end{align*}
Combining the above results eventually gives
\begin{align*}
V_1^\dagger(s_1, r) - V_1^{\pi, \bre(\pi)}(s_1, r) &\leq 2H \cdot V_1^{\pi, \bre(\pi)}(s_1, u/H) \leq \frac{2H}{K} \sum_{k=1}^K V_1^*(s_1, r^k) \\
&\leq \cO(\beta\sqrt{H^4 [\Gamma(K, \lambda;\ker)+\log (KH)]} /\sqrt{K} ),
\end{align*}
where the second inequality is due to Lemma \ref{lem:explore_plan_connect_game} and the last inequality is by Lemma \ref{lem:bonus_explore_game}.  The upper bound of the difference $V_1^\dagger(s_1, r) - V_1^{\pi, \bre(\pi)}(s_1, r)$ is also $\cO(\beta\sqrt{H^4 [\Gamma(K, \lambda;\ker)+\log (KH)]} /\sqrt{K} )$ with the similar proof idea. This completes the proof of Theorem \ref{thm:main_kernel_game}. Please see Section \ref{sec:proof_main_kernel_game} for details.

The proof of Theorem \ref{thm:main_neural_game} follows the same argument as above. The only difference is that the neural function approximation introduces bias terms depending on $\iota$ as we discussed in the proof sketch of Theorem \ref{thm:main_neural_single}. Thus, the final bound is $\cO(\beta \sqrt{ H^4  [\Gamma(K, \lambda; \ker_m) + \log (KH)]}/\sqrt{K} +  H^2 \beta \iota)$. Please see Section \ref{sec:proof_main_neural_game} for the detailed proof.

\section{Conclusion}
 
In this paper, we study the reward-free RL algorithms with kernel and neural function approximators for both single-agent MDPs and zero-sum Markov games. We prove that our methods can achieve $\widetilde{\mathcal{O}}(1 /\varepsilon^2)$ sample complexity for generating an $\varepsilon$-suboptimal policy or  $\varepsilon$-approximate NE.

\bibliography{bibliography}
\bibliographystyle{ims}

\newpage
\begin{appendices}
\onecolumn
\vspace{1em}
\renewcommand{\thesection}{\Alph{section}}

\vspace{1em}

\section{Discussion of Function Space Complexity} \label{sec:convering}

To characterize the function space complexity, we first introduce the notions for the eigenvalues of the RKHS. Define $\cL^2(\cZ)$ as the space of square-integrable functions on $\cZ$ w.r.t. Lebesgue measure and define $\langle \cdot, \cdot \rangle_{\cL^2}$ as the inner product on the space $\cL^2(\cZ)$. According to  Mercer’s Theorem \citep{steinwart2008support}, the kernel function $\ker(z,z')$ has a spectral expansion as $\ker(z,z')=\sum_{i=1}^\infty \sigma_i \varrho_i(z) \varrho_i(z')$ where $\{\varrho_i\}_{i\geq 1}$ are 	a set of orthonormal basis on $\cL^2(\cZ)$ and $\{\sigma_i\}_{i\geq 1}$ are positive eigenvalues. In this paper, we consider two types of eigenvalues' properties and make the following assumptions.
\begin{assumption} Assume $\{\sigma_i\}_{i\geq 1}$ satisfies one of the following eigenvalue decay conditions for some constant $\gamma > 0$:
\begin{itemize}
\item[(a)] $\gamma$-finite spectrum: we have $\sigma_i = 0$ for all $i > \gamma$;
\item[(b)] $\gamma$-exponential spectral decay: there exist constants $C_1>0$ and $C_2>0$ such that $\sigma_i \leq C_1 \exp (-C_2 \cdot i^\gamma)$ for all $i \geq 1$.
\end{itemize}
\end{assumption}

\vspace{5pt}

\noindent\textbf{Covering Numbers.} Next, we characterize the upper bound of the covering numbers of the Q-function sets  $\overline{\cQ}(c, R, B)$ and $\underline{\cQ}(c, R, B)$.  For any $Q_1, Q_2 \in \overline{\cQ}(c, R, B)$, we have
\begin{align*}
Q_1(z) = \min\left\{ c(z) + \Pi_{[0,H]} [\langle \wb_1, \phi(z)\rangle] + B\cdot \max\{  \|\phi(z)\|_{\Lambda_{\cD_1}^{-1}} , H/\beta\}^+ , H\right\}^+,\\
Q_2(z) = \min\left\{ c(z) + \Pi_{[0,H]} [\langle \wb_2, \phi(z)\rangle] + B\cdot \max\{  \|\phi(z)\|_{\Lambda_{\cD_2}^{-1}} , H/\beta\}^+ , H\right\}^+,
\end{align*}
for some $\wb_1, \wb_2$ satisfying $\|\wb_1\|_\cH \leq R$ and $\|\wb_2\|_\cH \leq R$. Then, due to the fact that the truncation operator is non-expansive, we have
\begin{align*}
\|Q_1(\cdot) - Q_2(\cdot)\|_\infty \leq \sup_z |\langle \wb_1 - \wb_2, \phi(z) \rangle_\cH| + B\sup_z \left|\|\phi(z)\|_{\Lambda_{\cD_1}^{-1}} - \|\phi(z)\|_{\Lambda_{\cD_2}^{-1}}\right|.
\end{align*}
The above inequality shows that it suffices to bound the covering numbers of of the RKHS norm ball of radius $R$ and the set of functions of the form $\|\phi(z)\|_{\Lambda_{\cD}^{-1}}$. Thus, we define the function class $\cF_\lambda:=\{\|\phi(\cdot)\|_\Upsilon: \|\Upsilon\|_{\op}\leq 1/\lambda\}$ since $\|\Lambda_\cD^{-1}\|_\op \leq 1/\lambda$ according to the definition of $\Lambda_\cD$. Let $\overline{\cN}_\infty (\epsilon; R, B)$ be the $\epsilon$-covering number of $\overline{\cQ}$ w.r.t. $\|\cdot\|_\infty$, $\cN_\infty (\epsilon, \cH, R)$ be the $\epsilon$-covering number of RKHS norm ball of radius $R$ w.r.t. $\|\cdot\|_\infty$, and $\cN_\infty (\epsilon, \cF, 1/\lambda)$ be the $\epsilon$-covering number of $\cF_\lambda$ w.r.t. $\|\cdot\|_\infty$. Thus, we have
\begin{align*}
\overline{\cN}_\infty (\epsilon; R, B) \leq\cN_\infty (\epsilon/2, \cH, R) \cdot\cN_\infty (\epsilon/(2B), \cF, 1/\lambda).
\end{align*}
We define the upper bound
\begin{align*}
\cN_\infty (\epsilon; R, B) := \cN_\infty (\epsilon/2, \cH, R) \cdot \cN_\infty (\epsilon/(2B), \cF, 1/\lambda).
\end{align*}
Then, we know
\begin{align*}
\log\cN_\infty (\epsilon; R, B) = \log\cN_\infty (\epsilon/2, \cH, R) +  \log\cN_\infty (\epsilon/(2B), \cF, 1/\lambda).
\end{align*}
Moreover, for any $Q_1, Q_2 \in \underline{\cQ}(c, R, B)$, we have
\begin{align*}
Q_1(z) = \min\left\{ c(z) + \Pi_{[0,H]} [\langle \wb_1, \phi(z)\rangle] - B\cdot \max\{  \|\phi(z)\|_{\Lambda_{\cD_1}^{-1}} , H/\beta\}^+ , H\right\}^+,\\
Q_2(z) = \min\left\{ c(z) + \Pi_{[0,H]} [\langle \wb_2, \phi(z)\rangle] - B\cdot \max\{  \|\phi(z)\|_{\Lambda_{\cD_2}^{-1}} , H/\beta\}^+ , H\right\}^+,
\end{align*}
which also implies 
\begin{align*}
\|Q_1(\cdot) - Q_2(\cdot)\|_\infty \leq \sup_z |\langle \wb_1 - \wb_2, \phi(z) \rangle_\cH| + B\sup_z \left|\|\phi(z)\|_{\Lambda_{\cD_1}^{-1}} - \|\phi(z)\|_{\Lambda_{\cD_2}^{-1}}\right|.
\end{align*}
Thus, we can bound the covering number $\underline{\cN}_\infty(\epsilon; R, B)$ of $\underline{\cQ}(c, R, B)$ in the same way, i.e., $\underline{\cN}_\infty(\epsilon; R, B)\allowbreak \leq \cN_\infty (\epsilon; R, B)$.

According to \citet{yang2020provably}, we have the following covering number upper bounds 
\begin{itemize}
\item[(a)] $\gamma$-finite spectrum:
\begin{align*}
&\log\cN_\infty (\epsilon/2, \cH, R)  \leq C_3 \gamma [\log(2R/\epsilon) + C_4],\\
&\log\cN_\infty (\epsilon/(2B), \cF, 1/\lambda) \leq C_5 \gamma^2 [\log(2B/\epsilon) + C_6];
\end{align*}

\item[(b)] $\gamma$-exponential spectral decay:
\begin{align*}
&\log\cN_\infty (\epsilon/2, \cH, R)  \leq C_3  [\log(2R/\epsilon) + C_4]^{1+1/\gamma},\\
&\log\cN_\infty (\epsilon/(2B), \cF, 1/\lambda) \leq C_5  [\log(2B/\epsilon) + C_6]^{1+2/\gamma}.
\end{align*}
\end{itemize}

\vspace{5pt}

\noindent\textbf{Maximal Information Gain.} Here we give the definition of maximal information gain and discuss its upper bounds based on different kernels.

\begin{definition}[Maximal Information Gain \citep{srinivas2009gaussian}]\label{def:eff_dim} For any fixed integer $\mathfrak{C}$ and any $\sigma > 0$, we define the maximal information gain associated with the RKHS $\cH$ as
\begin{align*}
\Gamma(\mathfrak{C}, \lambda; \ker) = \sup_{\cD\subseteq \cZ} \frac{1}{2} \log\det (I + \cK_\cD/\lambda),
\end{align*}
where the supremum is taken over all discrete subsets of $\cZ$ with cardinality no more than $\mathfrak{C}$, and $\cK_\cD$ is the Gram matrix induced by $\cD \subseteq \cZ$ based on the kernel $\ker$. 
\end{definition}
According to Theorem 5 in \citet{srinivas2009gaussian}, we have the maximal information gain characterized as follows
\begin{itemize}
\item[(a)] $\gamma$-finite spectrum:
\begin{align*}
\Gamma(K, \lambda; \ker) \leq C_7 \gamma \log K;
\end{align*}

\item[(b)] $\gamma$-exponential spectral decay:
\begin{align*}
\Gamma(K, \lambda; \ker) \leq C_7 (\log K)^{1+1/\gamma }.
\end{align*}
\end{itemize}

\noindent\textbf{Sample Complexity.} Given the above results, for the kernel approximation setting, according to the discussion in the proof of Corollary 4.4 in \citet{yang2020provably}, under the parameter settings in Theorem \ref{thm:main_kernel_single} or Theorem \ref{thm:main_kernel_game}, we have that for $\gamma$-finite spectrum setting, 
\begin{align*}
&\beta = \cO(\gamma H\sqrt{\log(\gamma KH)}), \quad \log\cN_{\infty}(\varsigma^*; R_K,  2\beta) = \cO(\gamma^2 \log(\gamma KH)),  \\
&\Gamma(K, \lambda; \ker) = \cO(\gamma\log K),
\end{align*}
which implies after $K$ episodes of exploration, the upper bound in Theorem \ref{thm:main_kernel_single} or Theorem \ref{thm:main_kernel_game} is 
\begin{align*}
 \cO \left(\sqrt{H^6\gamma^3\log^2(\gamma KH)/K} \right).
\end{align*}
This result further implies that to obtain an $\varepsilon$-suboptimal policy or $\varepsilon$-approximate NE, it requires $\tilde{\cO}(H^6 \gamma^3 /\varepsilon^2)$ rounds of exploration. 
In addition, for the $\gamma$-exponential spectral decay setting, we have
\begin{align*}
&\beta =  \cO( H\sqrt{\log( KH)} (\log K)^{1/\gamma}), \quad \log\cN_{\infty}(\varsigma^*; R_K,  2\beta) = \cO((\log K)^{1+2/\gamma}+(\log\log H)^{1+2/\gamma}), \\
&\Gamma(K, \lambda; \ker) = \cO( (\log K)^{1+1/\gamma}),
\end{align*}
which implies that after $K$ episodes of exploration, the upper bound in Theorem \ref{thm:main_kernel_single} or Theorem \ref{thm:main_kernel_game} is
\begin{align*}
 \cO \left(\sqrt{H^6\log^{2+3/\gamma}(KH)/K} \right).
\end{align*}
Then, to obtain an $\varepsilon$-suboptimal policy or $\varepsilon$-approximate NE, it requires $\cO(H^6 C_\gamma \log^{4+6/\gamma}(\varepsilon^{-1}) /\varepsilon^2) = \tilde{\cO}(H^6 C_\gamma /\varepsilon^2)$ episodes of exploration, where $C_\gamma$ is some constant depending on $1/\gamma$. 

The above results also hold for the neural function approximation under both single-agent MDP and Markov game setting if the kernel $\ker_m$ satisfies the $\gamma$-finite spectrum or $\gamma$-exponential spectral decay and the network width $m$ is sufficiently large such that the error term $H^2\beta \iota\leq \varepsilon$. Then, we can similarly obtain the upper bounds in Theorems \ref{thm:main_neural_single} and \ref{thm:main_neural_game}.

\noindent\textbf{Linear and Tabular Cases.} For the linear function approximation case, we have a feature map $\phi(s) \in \RR^\mathfrak{d}$, where $\mathfrak{d}$ is the feature dimension. Therefore, the associated kernel can be represented as $\ker(s,s') = \phi(s)^\top \phi(s') = \sum_{i=1}^\mathfrak{d} \phi_i(s) \phi_i(s')$. Thus, we know that under the linear setting, the kernel $\ker$ has $\mathfrak{d}$-finite spectrum. Thus, letting $\gamma = \mathfrak{d}$ in the $\gamma$-finite spectrum case, we have
\begin{align*}
&\beta = \cO(\mathfrak{d} H\sqrt{\log(\mathfrak{d} KH)}), \quad \log\cN_{\infty}(\varsigma^*; R_K,  2\beta) = \cO(\mathfrak{d}^2 \log(\mathfrak{d} KH)),  \\
&\Gamma(K, \lambda; \ker) = \cO(\mathfrak{d} \log K),
\end{align*}
which further implies that to achieve $V^*_1(s_1, r) - V^\pi_1(s_1, r) \leq \varepsilon$, it requires $\tilde{\cO}(H^6 \mathfrak{d}^3 /\varepsilon^2)$ rounds of exploration. This is consistent with the result in \citet{wang2020reward} for the single-agent MDP. This result also hold for the Markov game setting.

For the tabular case, since $\phi(z) = \be_z$ is the
canonical basis in $\RR^{|\cZ|}$, we have $\gamma = |\cZ|$ for the above $\gamma$-finite spectrum case. Therefore, for the single-agent MDP setting, we have $|\cZ| = |\cS||\cA|$, which implies 
\begin{align*}
&\beta = \cO( H|\cS||\cA|\sqrt{\log(|\cS||\cA| KH)}), \quad \Gamma(K, \lambda; \ker) = \cO(|\cS||\cA| \log K),  \\
&\log\cN_{\infty}(\varsigma^*; R_K,  2\beta) = \cO(|\cS|^2|\cA|^2 \log(|\cS||\cA| KH)).
\end{align*}
Then, the sample complexity becomes $\tilde{\cO}(H^6 |\cS|^3|\cA|^3 /\varepsilon^2)$ to obtain an $\varepsilon$-suboptimal policy. For the two-player Markov game setting, we have $|\cZ| = |\cS||\cA||\cB|$, which implies 
\begin{align*}
&\beta = \cO( H|\cS||\cA||\cB|\sqrt{\log(|\cS||\cA||\cB| KH)}), \quad \Gamma(K, \lambda; \ker) = \cO(|\cS||\cA| |\cB|\log K),  \\
&\log\cN_{\infty}(\varsigma^*; R_K,  2\beta) = \cO(|\cS|^2|\cA|^2|\cB|^2 \log(|\cS||\cA||\cB| KH)) .
\end{align*}
Then, the sample complexity becomes $\tilde{\cO}(H^6 |\cS|^3|\cA|^3 |\cB|^3/\varepsilon^2)$ to obtain an $\varepsilon$-approximate NE.


\section{Proofs for Single-Agent MDP with Kernel Function Approximation}

\subsection{Lemmas}

\begin{lemma}[Solution of Kernel Ridge Regression] \label{lem:compute_estimate} The approximation vector $\hat{f}_h^k \in \cH$ is obtained by solving the following kernel ridge regression problem
\begin{align*}
\minimize_{f\in \cH} \sum_{\tau=1}^{k-1} [V_{h+1}^k (s_{h+1}^\tau) - f(z_h^\tau) \rangle_\cH]^2 + \lambda \|f\|^2_\cH,
\end{align*}
such that we have
\begin{align*}
\hat{f}_h^k(z) = \langle\phi(z), \hat{f}_h^k \rangle_\cH = \psi_h^k(z)^\top (\lambda \cdot I +\cK_h^k )^{-1} \yb_h^k,
\end{align*}
where we define 
\begin{align}
\begin{aligned}\label{eq:kernel_sol_def_proof}
 &\psi_h^k (z) := \Phi_h^k  \phi(z)= [\ker(z, z_h^1), \cdots, \ker(z, z_h^{k-1})]^\top, \\
 &\Phi_h^k = [\phi(z_h^1), \phi(z_h^2), \cdots, \phi(z_h^{k-1})]^\top, \\
&\yb_h^k = [V_{h+1}^k(s_{h+1}^1), V_{h+1}^k(s_{h+1}^2) , \cdots, V_{h+1}^k(s_{h+1}^{k-1})  ]^\top,\\
 & \cK_h^k := \Phi_h^k (\Phi_h^k)^\top = \begin{bmatrix}
\ker(z_h^1, z_h^1) & \ldots & \ker(z_h^1, z_h^{k-1}) \\ 
\vdots &\ddots   &\vdots \\ 
\ker( z_h^{k-1} , z_h^1) &\ldots  & \ker( z_h^{k-1}, z_h^{k-1})
\end{bmatrix}, 
\end{aligned}
\end{align}
with denoting $z=(s,a)$ and $z_h^\tau = (s_h^\tau, a_h^\tau)$, and $\ker(x,y)= \langle \phi(z), \phi(z') \rangle_\cH, \forall z,z' \in \cZ = \cS\times \cA$.
\end{lemma}

\begin{proof}
We seek to solve the following kernel ridge regression problem in the RKHS
\begin{align*}
\hat{f}_h^k = \argmin_{f\in \cH} \sum_{\tau=1}^{k-1} [V_{h+1}^k (s_{h+1}^\tau) - f(s_h^\tau,a_h^\tau) \rangle_\cH]^2 + \lambda \|f\|^2_\cH, 
\end{align*}
which is equivalent to
\begin{align*}
\hat{f}_h^k = \argmin_{f\in \cH} \sum_{\tau=1}^{k-1} [V_{h+1}^k(s_{h+1}^\tau) - \langle f, \phi(s_h^\tau,a_h^\tau) \rangle_\cH]^2 + \lambda \langle f, f\rangle_\cH.
\end{align*}
By the first-order optimality condition, the above kernel ridge regression problem admits the following closed-form solution
\begin{align} \label{eq:RKHS_sol}
\hat{f}_h^k = (\Lambda^k_h)^{-1}(\Phi_h^k)^\top \yb_h^k, 
\end{align}
where we define
\begin{align*}
\Lambda_h^k = \sum_{\tau=1}^{k-1} \phi(s_h^{\tau}, a_h^{\tau}) \phi(s_h^{\tau}, a_h^{\tau})^\top + \lambda \cdot I_\cH = \lambda \cdot I_\cH + (\Phi_h^k)^\top \Phi_h^k, 
\end{align*}
with $I_\cH$ being the identity mapping in RKHS. Thus, by  \eqref{eq:RKHS_sol}, we have 
\begin{align*} 
\langle \hat{f}_h^k, \phi(z)\rangle_\cH =  \langle (\Lambda^k_h)^{-1}(\Phi_h^k)^\top \yb_h^k, \phi(s,a) \rangle_\cH, \quad \forall(z) \in \cS\times \cA,
\end{align*}
which can be further rewritten in terms of kernel  $\ker$ as follows
\begin{align}
\begin{aligned}\label{eq:RKHS_approx_ker}
\langle \hat{f}_h^k, \phi(z)\rangle_\cH &=  \langle (\Lambda^k_h)^{-1}(\Phi_h^k)^\top \yb_h^k, \phi(z) \rangle_\cH \\
&= \phi(z)^\top [\lambda \cdot I_\cH + (\Phi_h^k)^\top \Phi_h^k]^{-1} (\Phi_h^k)^\top \yb_h^k\\
&= \phi(z)^\top (\Phi_h^k)^\top[\lambda \cdot I +\Phi_h^k (\Phi_h^k)^\top ]^{-1} \yb_h^k \\
&= \psi_h^k(z)^\top (\lambda \cdot I +\cK_h^k )^{-1} \yb_h^k.
\end{aligned}
\end{align}
The third equality is by 
\begin{align*}
(\Phi_h^k)^\top[\lambda \cdot I +\Phi_h^k (\Phi_h^k)^\top ] = [\lambda \cdot I_\cH + (\Phi_h^k)^\top \Phi_h^k](\Phi_h^k)^\top,
\end{align*}
such that 
\begin{align*}
[\lambda \cdot I_\cH + (\Phi_h^k)^\top \Phi_h^k]^{-1}(\Phi_h^k)^\top = (\Phi_h^k)^\top[\lambda \cdot I +\Phi_h^k (\Phi_h^k)^\top ]^{-1}, 
\end{align*}
where $I$ is an identity matrix in $\RR^{(k-1) \times (k-1)}$. The last equality in \eqref{eq:RKHS_approx_ker} is by the definitions of $\psi_h^k (z)$ and $\cK_h^k$ in \eqref{eq:kernel_sol_def_proof}. This completes the proof.
\end{proof}

\begin{lemma} [Boundedness of Solution] \label{lem:linear_approx_param_bound}  When $\lambda \geq 1$, for any $(k, h) \in [K] \times [H]$,  $\hat{f}_h^k$ defined in \eqref{eq:RKHS_sol} satisfies
\begin{align*}
\|\hat{f}_h^k\|_\cH \leq H\sqrt{2K/\lambda \cdot \log\det(I + \cK_h^k/\lambda)} \leq 2H \sqrt{K\cdot \Gamma (K, \lambda; \ker)},
\end{align*}
where $\cK_h^k$ is defined in \eqref{eq:kernel_sol_def_proof} and $\Gamma (K, \lambda; \ker)$ is defined in Definition \ref{def:eff_dim}.
\end{lemma}
\begin{proof} For any vector $f\in \cH$, we have
\begin{align*}
|\langle f, \hat{f}_h^k\rangle_\cH | &= | f^\top (\Lambda^k_h)^{-1}(\Phi_h^k)^\top \yb_h^k | \\
&=\left| f^\top (\Lambda^k_h)^{-1}\sum_{\tau=1}^{k-1}\phi(s_h^\tau, a_h^\tau) V_{h+1}^k (s_{h+1}^\tau)  \right| \leq H\sum_{\tau=1}^{k-1} \left | f^\top (\Lambda^k_h)^{-1}\phi(s_h^\tau, a_h^\tau) \right|,
\end{align*}
where the last inequality is due to $|V_{h+1}^k (s_{h+1}^\tau)| \leq H$. Then, with Lemma \ref{lem:direct_sum_bound}, the rest of the proof is the same as the proof of Lemma H.1 in \citet{yang2020provably}, which finishes the proof.
\end{proof}

\begin{lemma} \label{lem:approx_concentrate} With probability at least $1-\delta'$, we have $\forall (h,k) \in [H] \times [K]$,
\begin{align*}
&\left \|\sum_{\tau=1}^{k-1} \phi(s_h^\tau, a_h^\tau) [V_{h+1}^k(s_{h+1}^\tau) - \PP_h V_{h+1}^k(s_h^\tau, a_h^\tau)] \right\|_{(\Lambda_h^k)^{-1}}^2 \\
&\qquad  \leq  4H^2 \Gamma(K, \lambda; \ker) + 10H^2+4H^2\log\cN_{\infty}(\varsigma^*;R_K,  B_K) + 4H^2\log(K /\delta'),
\end{align*}
where we set $\varsigma^* = H/K$ and $\lambda = 1 + 1/K$.

\end{lemma}
\begin{proof} We first define a value function class as follows
\begin{align*}
\overline{\cV}(\boldsymbol 0, R, B)= \{ V: V(\cdot) = \max_{a\in \cA} Q(\cdot, a) \text{ with } Q\in \overline{\cQ}(\boldsymbol 0, R, B) \},
\end{align*}
where $\overline{\cQ}$ is defined in \eqref{eq:Q_func_class}.
We denote the covering number of $\overline{\cV}(\bm{0}, R, B)$ w.r.t. the distance $\dist$ as $\cN^{\overline{\cV}}_{\dist}(\epsilon; R, B)$, where the distance $\dist$ is defined by $\dist(V_1, V_2) = \sup_{s\in \cS} |V_1(s)- V_2(s)|$.
Specifically, for any $k\times h \in [K]\times[H]$, we assume that there exist constants $R_K$ and $B_K$ that depend on the number of episodes $K$ such that any $V_h^k\in \overline{\cV}(\boldsymbol{0}, R_K, B_K)$ with $R_K = 2H\sqrt{\Gamma(K, \lambda; \ker)}$ and $B_K = (1+1/H) \beta$ since $Q_h^k(z) = \Pi_{[0, H]}[(r_h^k + u_h^k + f_h^k)(z)] = \Pi_{[0, H]}[\Pi_{[0, H]}[\langle \hat{f}_h^k, \phi(z)\rangle_\cH] + (1+1/H) \beta\cdot \min\{\|\phi(z)\|_{\Lhki}, H/\beta\}]$ (See the next lemma for the reformulation of the bonus term). By Lemma \ref{lem:self_normalize_uniform} with $\delta'/K$, we have
\begin{align*}
&\left \|\sum_{\tau=1}^{k-1} \phi(s_h^\tau, a_h^\tau) [V_{h+1}^k(s_{h+1}^\tau) - \PP_h V_{h+1}^k(s_h^\tau, a_h^\tau)] \right\|_{(\Lambda_h^k)^{-1}}^2 \\
&\qquad \leq \sup_{V\in \overline{\cV}(\boldsymbol 0, R_K, B_K)} \left \|\sum_{\tau=1}^{k-1} \phi(s_h^\tau, a_h^\tau) [V(s_{h+1}^\tau) - \PP_h V(s_h^\tau, a_h^\tau)] \right\|_{(\Lambda_h^k)^{-1}}^2  \\
&\qquad \leq  2H^2 \log\det(I+\cK_k/\lambda) + 2H^2k(\lambda-1)+4H^2\log(K\cN^{\overline{\cV}}_{\dist}(\epsilon; R_K, B_K)/\delta')+ 8k^2\epsilon^2/\lambda\\
&\qquad \leq  4H^2 \Gamma(K, \lambda; \ker) + 10H^2+4H^2\log\cN_{\infty}(\varsigma^*;R_K,  B_K) + 4H^2\log(K /\delta'),
\end{align*}
where the last inequality is by setting $\lambda = 1+1/K$ and $\epsilon = \varsigma^* =  H/K$. Moreover, the last inequality is also due to 
\begin{align*}
\dist(V_1, V_2) &= \sup_{s\in \cS} \left| V_1(s) - V_2(s) \right |= \sup_{s\in \cS}\left| \max_{a\in \cA} Q_1(s,a) - \max_{a\in \cA} Q_2(s,a) \right| \\
&\leq \sup_{(s,a)\in \cS\times \cA} | Q_1(s,a) - Q_2(s,a) | = \|Q_1-Q_2\|_\infty,
\end{align*}
which indicates that $\cN^{\overline{\cV}}_{\dist}(\varsigma^*; R_K, B_K)$ upper bounded by the covering number of the class $\overline{\cQ}$ w.r.t. $\|\cdot\|_\infty$, such that 
\begin{align*}
\cN^{\overline{\cV}}_{\dist}(\varsigma^*; R_K, B_K) \leq \cN_{\infty}(\varsigma^*; R_K, B_K).
\end{align*}
Here $\cN_{\infty}(\epsilon; R, B)$ denotes the upper bound of the covering number of $\overline{\cQ}(
h, R, B)$ w.r.t. $\ell_\infty$-norm, which is characterized in Section \ref{sec:convering}. Further by the union bound, we know that the above inequality holds for all $k\in [K]$ with probability at least $1-\delta'$. This completes the proof.
\end{proof}

\begin{lemma} \label{lem:bonus_concentrate} We define the event $\cE$ as that the following inequality holds $\forall z = (s,a) \in \cS \times \cA, \forall (h,k) \in [H] \times [K]$,
\begin{align*}
| \PP_h V_{h+1}^k(z) - f_h^k(z)|  \leq u_h^k(z),
\end{align*}
where $f_h^k(z) = \Pi_{[0, H]}[\hat{f}_h^k(z)]$ and $u_h^k(z) = \min\{w_h^k(z), H\}$ with $w_h^k(z) = \beta \lambda^{-1/2} [ \ker(z,z)- \psi_h^k(z)^\top (\lambda I + \cK_h^k )^{-1}  \psi_h^k(z)]^{1/2}$. Thus, setting $\beta  = B_K/ (1+1/H)$, if $B_K$ satisfies
\begin{align*}
16H^2\big[ R^2_Q  + 2 \Gamma(K, \lambda; \ker) + 5+\log\cN_{\infty}(\varsigma^*;R_K,  B_K) + 2\log(K /\delta')  \big] \leq  B^2_K, \forall h\in [H],
\end{align*}
then we have that with probability at least $1-\delta'$, the event $\cE$ happens, i.e.,
\begin{align*}
\Pr(\cE) \geq 1-\delta'.
\end{align*}
\end{lemma}

\begin{proof}
We assume that $\PP_h V_{h+1}^k(s,a) = \langle \tilde{f}_h^k, \phi(s,a) \rangle_\cH$ for some $\tilde{f}_h^k\in \cH$. Then, we bound the difference between $f_h^k(z)$ and $\PP_h V_{h+1}^k(s,a)$ in the following way
\begin{align*}
&|\PP_h V_{h+1}^k(s,a) - f_h^k(s,a)|  \\
&\qquad \leq |\langle \tilde{f}_h^k, \phi(s,a)\rangle_\cH  -   \psi_h^k(s,a)^\top (\lambda \cdot I +\cK_h^k )^{-1} \yb_h^k| \\
&\qquad = |\lambda  \phi(s,a)^\top(\Lambda_h^k)^{-1}\tilde{f}_h^k +  \psi_h^k(s,a)^\top (\lambda\cdot I + \cK_h^k)^{-1}\Phi_h^k\tilde{f}_h^k -   \psi_h^k(s,a)^\top (\lambda \cdot I +\cK_h^k )^{-1} \yb_h^k|\\
&\qquad = |\lambda  \phi(s,a)^\top(\Lambda_h^k)^{-1}\tilde{f}_h^k +  \psi_h^k(s,a)^\top (\lambda\cdot I + \cK_h^k)^{-1}(\Phi_h^k\overline{f}_h^k -  \yb_h^k)|,
\end{align*}
where the first inequality is due to $0\leq \PP_h V_{h+1}^k(s,a) \leq H$, non-expansiveness of the operator $\Pi_{[0, H]} [\cdot] := \min\{\cdot, H\}^+$, and the definition of $\hat{f}_h^k(z)$ in Lemma \ref{lem:compute_estimate}, and the first equality is due to 
\begin{align}
\begin{aligned}\label{eq:reform_phi}
\phi(s,a) &= (\Lambda_h^k)^{-1}\Lambda_h^k \phi(s,a) = (\Lambda_h^k)^{-1}(\lambda \cdot I +(\Phi_h^k)^\top\Phi_h^k ) \phi(s,a)\\
&=\lambda (\Lambda_h^k)^{-1}\phi(s,a) + (\Lambda_h^k)^{-1} (\Phi_h^k)^\top \Phi_h^k\phi(s,a) \\
&=\lambda (\Lambda_h^k)^{-1}\phi(s,a) + (\Phi_h^k)^\top (\lambda\cdot I + \cK_h^k)^{-1}\Phi_h^k  \phi(s,a) \\
&=\lambda (\Lambda_h^k)^{-1}\phi(s,a) + (\Phi_h^k)^\top (\lambda\cdot I + \cK_h^k)^{-1}\psi_h^k (s,a).
\end{aligned}
\end{align}
Thus, we have
\begin{align}
\begin{aligned}\label{eq:approx_err}
|\PP_h V_{h+1}^k(s,a, r^k) - f_h^k(s,a) | &\leq \underbrace{\lambda  \|\phi(s,a)^\top(\Lambda_h^k)^{-1}\|_\cH \cdot \|\tilde{f}_h^k\|_\cH}_{\text{Term(I)}} \\
&\quad +  \underbrace{|\psi_h^k(s,a)^\top (\lambda\cdot I + \cK_h^k)^{-1}(\Phi_h^k\tilde{f}_h^k -  \yb_h^k)|}_{\text{Term(II)}}.
\end{aligned}
\end{align}
For Term(I), we have
\begin{align}
\begin{aligned} \label{eq:bound_termI}
\text{Term(I)} &\leq \sqrt{\lambda} R_Q H \sqrt{\phi(s,a)^\top (\Lambda_h^k)^{-1} \cdot\lambda I \cdot (\Lambda_h^k)^{-1} \phi(s,a)} \\
&\leq \sqrt{\lambda} R_Q H \sqrt{\phi(s,a)^\top (\Lambda_h^k)^{-1} \cdot\Lambda_h^k \cdot (\Lambda_h^k)^{-1} \phi(s,a)} \\
&\leq \sqrt{\lambda} R_Q H \sqrt{\phi(s,a)^\top (\Lambda_h^k)^{-1}  \phi(s,a)} = \sqrt{\lambda} R_Q H \|\phi(s,a)\|_{(\Lambda_h^k)^{-1}},
\end{aligned}
\end{align}
where the first inequality is due to Assumption \ref{assump:kernel} and the second inequality is by $ \theta^\top  (\Phi_h^k)^\top \Phi_h^k \theta = \|\Phi_h^k \theta\|_\cH \geq 0$ for any $\theta \in \cH$.

For Term(II), we have 
\begin{align}
\begin{aligned}\label{eq:bound_termII}
\text{Term(II)} &= \left| \phi(s,a)^\top (\Lambda_h^k)^{-1} \left \{ \sum_{\tau=1}^{k-1} \phi(s_h^\tau, a_h^\tau) [V_{h+1}^k(s_{h+1}^\tau) - \PP_h V_{h+1}^k(s_h^\tau, a_h^\tau)] \right\} \right| \\
&= \left| \phi(s,a)^\top (\Lambda_h^k)^{-1/2}(\Lambda_h^k)^{-1/2} \left \{ \sum_{\tau=1}^{k-1} \phi(s_h^\tau, a_h^\tau) [V_{h+1}^k(s_{h+1}^\tau) - \PP_h V_{h+1}^k(s_h^\tau, a_h^\tau)] \right\} \right| \\
&\leq \| \phi(s,a)\|_{(\Lambda_h^k)^{-1}} \left \|\sum_{\tau=1}^{k-1} \phi(s_h^\tau, a_h^\tau) [V_{h+1}^k(s_{h+1}^\tau) - \PP_h V_{h+1}^k(s_h^\tau, a_h^\tau)] \right\|_{(\Lambda_h^k)^{-1}}
\end{aligned}
\end{align}
By Lemma \ref{lem:approx_concentrate}, we have that with probability at least $1-\delta'$, the following inequality holds for all $k\in [K]$
\begin{align*}
&\left \|\sum_{\tau=1}^{k-1} \phi(s_h^\tau, a_h^\tau) [V_{h+1}^k(s_{h+1}^\tau) - \PP_h V_{h+1}^k(s_h^\tau, a_h^\tau)] \right\|_{(\Lambda_h^k)^{-1}} \\
&\qquad  \leq  [4H^2 \Gamma(K, \lambda; \ker) + 10H^2+4H^2\log\cN_{\infty}(\varsigma^*;R_K,  B_K) + 4H^2\log(K /\delta')]^{1/2}. 
\end{align*}
Thus, Term(II) can be further bounded as 
\begin{align*}
\text{Term(II)} \leq H\big[ 4 \Gamma(K, \lambda; \ker) + 10 +4 \log\cN_{\infty}(\varsigma^*;R_K,  B_K) + 4 \log(K /\delta') \big]^{1/2}\| \phi(s,a)\|_{(\Lambda_h^k)^{-1}} .
\end{align*}
Plugging the upper bounds of Term(I) and Term(II) into \eqref{eq:approx_err}, we obtain
\begin{align*}
&|\PP_h V_{h+1}^k(s,a, r^k) - f_h^k(s,a) |  \\
&\ \  \leq H\big[ \sqrt{\lambda} R_Q + [4 \Gamma(K, \lambda; \ker) + 10+4\log\cN_{\infty}(\varsigma^*;R_K,  B_K) + 4\log(K /\delta')]^{1/2}  \big]\| \phi(s,a)\|_{(\Lambda_h^k)^{-1}}\\
&\ \  \leq H\big[ 2\lambda R^2_Q  + 8 \Gamma(K, \lambda; \ker) + 20 +4 \log\cN_{\infty}(\varsigma^*;R_K,  B_K) + 8 \log(K /\delta')  \big]^{1/2}\| \phi(s,a)\|_{(\Lambda_h^k)^{-1}}\\
&\ \  \leq \beta \|\phi(s,a)\|_{(\Lambda_h^k)^{-1}} = \beta \lambda^{-1/2} [ \ker(z,z)-  \psi_h^k(s,a)^\top (\lambda I + \cK_h^k )^{-1}  \psi_h^k(s,a)]^{1/2},
\end{align*}
where $\varsigma^* = H/K$, and $\lambda = 1 + 1/K$ as in Lemma \ref{lem:approx_concentrate}. In the last equality, we also use the identity that 
\begin{align} 
\begin{aligned}\label{eq:bonus_kernel_form}
\|\phi(s,a)\|^2_{(\Lambda_h^k)^{-1}} &= \lambda^{-1} \phi(s,a)^\top \phi(s,a)  - \lambda^{-1} \psi_h^k (s,a)^\top  (\lambda\cdot I + \cK_h^k)^{-1}\psi_h^k (s,a) \\
&= \lambda^{-1} \ker(z,z)- \lambda^{-1} \psi_h^k(s,a)^\top (\lambda I + \cK_h^k )^{-1}  \psi_h^k(s,a).
\end{aligned}
\end{align}
This is proved by  
\begin{align*}
\|\phi(s,a)\|^2_\cH &= \phi(s,a)^\top [\lambda (\Lambda_h^k)^{-1}\phi(s,a) + (\Phi_h^k)^\top (\lambda\cdot I + \cK_h^k)^{-1}\Phi_h^k \phi (s,a)] \\
& =  \lambda  \phi(s,a)^\top (\Lambda_h^k)^{-1}\phi(s,a) + \psi_h^k (s,a)^\top  (\lambda\cdot I + \cK_h^k)^{-1}\psi_h^k (s,a),
\end{align*}
where the first equality is by \eqref{eq:reform_phi}.

According to Lemma \ref{lem:linear_approx_param_bound}, we know that $\hat{f}_h^k$ satisfies  $\|\hat{f}_h^k\|_\cH \leq H\sqrt{2K/\lambda \cdot \log\det(I + \cK_h^k/\lambda)} \leq 2H \sqrt{K\cdot \Gamma (K, \lambda; \ker)}$. Then, one can set $R_K = 2H \sqrt{K\cdot \Gamma (K, \lambda; \ker)}$.  Moreover, as we set $(1+1/H)\beta = B_K$, then $\beta = B_K/(1+1/H)$. Thus, we let
\begin{align*}
 &\big[ 2\lambda R^2_Q H^2 + 8H^2 \Gamma(K, \lambda; \ker) + 20H^2+4H^2\log\cN_{\infty}(\varsigma^*;R_K,  B_K) + 8H^2\log(K /\delta')  \big]^{1/2} \\
& \qquad \leq \beta = B_K/(1+1/H),
\end{align*}
which can be further guaranteed by 
\begin{align*}
16H^2\big[ R^2_Q  + 2 \Gamma(K, \lambda; \ker) + 5+\log\cN_{\infty}(\varsigma^*;R_K,  B_K) + 2\log(K /\delta')  \big] \leq  B^2_K
\end{align*}
as $(1+1/H) \leq 2$ and $\lambda = 1+1/K \leq 2$. 

According to the above result, letting $w_h^k = \beta \|\phi(s,a)\|_{(\Lambda_h^k)^{-1}} = \beta \lambda^{-1/2} [ \ker(z,z)-  \psi_h^k(s,a)^\top (\lambda I + \cK_h^k )^{-1}  \psi_h^k(s,a)]^{1/2}$, we have $ -w_h^k\leq \PP_h V_{h+1}^k(s,a) - f_h^k(s,a)   \leq w_h^k$.  Note that we also have $|\PP_h V_{h+1}^k(s,a) - f_h^k(s,a)| \leq H$ due to $0\leq f_h^k(s,a)  \leq H$ and $0 \leq  \PP_h V_{h+1}^k(s,a) \leq H$. Thus, there is $| \PP_h V_{h+1}^k(s,a) - f_h^k(s,a) |  \leq \min\{w_h^k, H\}$.
This completes the proof.
\end{proof}

\begin{lemma} \label{lem:bonus_explore} Conditioned on the event $\cE$ defined in Lemma \ref{lem:bonus_concentrate}, with probability at least $1-\delta'$, we have
\begin{align*}
\sum_{k=1}^K V_1^*(s_1, r^k) \leq \sum_{k=1}^K V_1^k(s_1) \leq \cO\left(\sqrt{H^3 K \log (1/\delta')} +\beta  \sqrt{H^2 K \cdot \Gamma(K, \lambda; \ker)}\right).
\end{align*}
\end{lemma}

\begin{proof} We first show the first inequality in this lemma, i.e., $\sum_{k=1}^K V_1^*(s_1, r^k) \leq \sum_{k=1}^K V_1^k(s_1)$. To show this inequality holds, it suffices to show $V_h^*(s, r^k) \leq V_h^k(s)$ for all $s \in \cS, h \in [H]$. We prove it by induction. 

When  $h=H+1$, we know $V_{H+1}^*(s, r^k) = 0$ and $V_{H+1}^k(s)= 0 $ such that $V_{H+1}^*(s, r^k) = V_{H+1}^k(s_1)$. Now we assume that $V_{h+1}^*(s, r^k) \leq V_{h+1}^k(s)$. Then, conditioned on the event $\cE$ defined in Lemma \ref{lem:bonus_concentrate}, for all $s\in \cS$, $(h,k)\in [H]\times[K]$, we further have
\begin{align}
\begin{aligned}\label{eq:Q_diff}
&Q_h^*(s, a, r^k) - Q_h^k(s, a) \\
&\qquad= r_h^k(s,a) + \PP_h V_{h+1}^*(s,a, r^k)  -  \min \{ r_h^k(s,a) + f_h^k(s,a) + u_h^k(s,a), H \}^+  \\
&\qquad\leq  \max \{\PP_h V_{h+1}^*(s,a, r^k)  - f_h^k(s,a) - u_h^k(s,a), 0 \}  \\
&\qquad\leq  \max \{\PP_h V_{h+1}^k(s,a)  - f_h^k(s,a) - u_h^k(s,a), 0 \}  \\
&\qquad\leq 0,
\end{aligned}
\end{align}
where the first inequality is due to $0 \leq r_h^k(s,a) + \PP_h V_{h+1}^*(s,a, r^k) \leq H$ and $\min \{x, y\}^+ \geq \min \{x, y\}$, the second inequality is by the assumption that $V_{h+1}^*(s, r^k)\leq V_{h+1}^k(s)$, the last inequality is by Lemma \ref{lem:bonus_concentrate} such that $\PP_h V_{h+1}^k(s,a) - f_h^k(s,a)  \leq u_h^k(s,a)$ holds for any $(s,a)\in \cS \times \cA$ and $(k,h)\in [K]\times[H]$. The above inequality \eqref{eq:Q_diff} further leads to 
\begin{align*}
V_h^*(s, r^k) = \max_{a\in \cA} Q_h^*(s, a, r^k) \leq \max_{a\in \cA} Q_h^k(s, a) = V_h^k(s). 
\end{align*}
Therefore, we obtain that conditioned on event $\cE$, we have
\begin{align*}
\sum_{k=1}^K V_1^*(s, r^k) \leq \sum_{k=1}^K V_1^k(s).
\end{align*}
Next, we prove the second inequality in this lemma, namely the upper bound of $\sum_{k=1}^K V_1^k(s_1)$. Specifically, conditioned on $\cE$ defined in Lemma \ref{lem:bonus_concentrate}, we have 
\begin{align*}
V_h^k(s_h^k) &= Q_h^k(s_h^k, a_h^k) \leq  f_h^k(s_h^k, a_h^k) + r_h^k(s_h^k, a_h^k) + u_h^k(s_h^k, a_h^k)\\
&\leq \PP_h V_{h+1}^k(s_h^k, a_h^k)  + u_h^k(s_h^k, a_h^k) + r_h^k(s_h^k, a_h^k) + u_h^k(s_h^k, a_h^k)\\
&\leq \PP_h V_{h+1}^k(s_h^k, a_h^k)  + (2+1/H)w_h^k\\
& = \zeta_h^k + V_{h+1}^k(s_{h+1}^k) + (2+1/H)\beta \|\phi(s_h^k,a_h^k)\|_{(\Lambda_h^k)^{-1}},
\end{align*}
where the second inequality is due to Lemma \ref{lem:bonus_concentrate} and in the last equality, we define
\begin{align*}
\zeta_h^k := \PP_h V_{h+1}^k(s_h^k, a_h^k) - V_{h+1}^k(s_{h+1}^k).
\end{align*}
Recursively applying the above inequality gives
\begin{align*}
V_1^k(s_1) \leq  \sum_{h=1}^H \zeta_h^k  + (2+1/H)\beta \sum_{h=1}^H \|\phi(s_h^k,a_h^k)\|_{(\Lambda_h^k)^{-1}},
\end{align*}
where we use the fact that $V_{H+1}^k(\cdot) = 0$. Taking summation on both sides of the above inequality, we have
\begin{align*}
\sum_{k=1}^KV_1^k(s_1)  = \sum_{k=1}^K\sum_{h=1}^H \zeta_h^k  + (2+1/H)\beta \sum_{k=1}^K\sum_{h=1}^H \|\phi(s_h^k,a_h^k)\|_{(\Lambda_h^k)^{-1}}.
\end{align*}
By Azuma-Hoeffding inequality, with probability at least $1-\delta'$, the following inequalities hold 
\begin{align*}
&\sum_{k=1}^K \sum_{h=1}^H \zeta_h^k \leq \cO\left(\sqrt{H^3 K \log \frac{1}{\delta'}} \right).
\end{align*} 
On the other hand, by Lemma \ref{lem:direct_sum_bound}, we have
\begin{align*}
\sum_{k=1}^K\sum_{h=1}^H \|\phi(s_h^k,a_h^k)\|_{(\Lambda_h^k)^{-1}} &= \sum_{k=1}^K\sum_{h=1}^H \sqrt{\phi(s_h^k,a_h^k)^\top (\Lambda_h^k)^{-1}\phi(s_h^k,a_h^k)}\\
&\leq \sum_{h=1}^H \sqrt{K\sum_{k=1}^K\phi(s_h^k,a_h^k)^\top (\Lambda_h^k)^{-1}\phi(s_h^k,a_h^k)} \\
&\leq \sum_{h=1}^H \sqrt{2K\log\det (I +\lambda \cK_h^K)} = 2H \sqrt{K\cdot \Gamma(K, \lambda; \ker)}.
\end{align*}
where the first inequality is by Jensen's inequality. Thus, conditioned on event $\cE$, we obtain that with probability at least $1-\delta'$, there is 
\begin{align*}
\sum_{k=1}^KV_1^k(s_1) \leq \cO\left(\sqrt{H^3 K \log (1/\delta')} + \beta \sqrt{H^2 K \cdot \Gamma(K, \lambda; \ker)}\right),
\end{align*}
which completes the proof.
\end{proof}

\begin{lemma} \label{lem:bonus_concentrate_plan} We define the event $\tilde{\cE}$ as that the following inequality holds $\forall z = (s,a) \in \cS \times \cA, \forall h \in [H]$,
\begin{align*}
|\PP_h V_{h+1}(z) - f_h(z)|  \leq u_h(z),
\end{align*}
where $u_h(z) = \min\{w_h(z), H \}^+$ with  $w_h(z) = \beta \lambda^{-1/2} [ \ker(z,z)- \psi_h(z)^\top (\lambda I + \cK_h )^{-1}  \psi_h(z)]^{1/2}$. Thus, setting $\beta = \tilde{B}_K$, if $\tilde{B}_K$ satisfies
\begin{align*}
4H^2\big[ R^2_Q  + 2 \Gamma(K, \lambda; \ker) + 5+\log\cN_{\infty}(\varsigma^*; \tilde{R}_K,  \tilde{B}_K) + 2\log(K /\delta')  \big] \leq  \tilde{B}_K^2, \forall h\in [H],
\end{align*}
then we have that with probability at least $1-\delta'$, the event $\cE$ happens, i.e.,
\begin{align*}
\Pr(\tilde{\cE}) \geq 1-\delta'.
\end{align*}
\end{lemma}

\begin{proof} The proof of this lemma is nearly the same as the proof of Lemma \ref{lem:bonus_concentrate}. We provide the sketch of this proof below. 

We assume that the true transition is formulated as $\PP_h V_{h+1}(z) = \langle \tilde{f}_h, \phi(z) \rangle_\cH =: \tilde{f}_h(z)$.  We have the following definitions
\begin{align*}
&\Phi_h = [\phi(s_h^1,a_h^1), \phi(s_h^2,a_h^2), \cdots, \phi(s_h^K,a_h^K)]^\top, \\
&\Lambda_h = \sum_{\tau=1}^{K} \phi(s_h^{\tau}, a_h^{\tau}) \phi(s_h^{\tau}, a_h^{\tau})^\top + \lambda \cdot I_\cH = \lambda \cdot I_\cH + (\Phi_h)^\top \Phi_h, \\
&\yb_h = [V_{h+1}(s_{h+1}^1), V_{h+1}(s_{h+1}^2) , \cdots, V_{h+1}(s_{h+1}^K)  ]^\top, \quad \cK_h = \Phi_h \Phi_h^\top, \quad \psi_h (s,a) = \Phi_h  \phi(s,a). 
\end{align*}
Then, we bound the following term
\begin{align*}
&|\PP_h V_{h+1}(s,a) - f_h(s,a)|  \\
&\qquad \leq |\langle \tilde{f}_h, \phi(s,a)\rangle_\cH  -   \psi_h(s,a)^\top (\lambda \cdot I +\cK_h )^{-1} \yb_h| \\
&\qquad = |\lambda  \phi(s,a)^\top\Lambda_h^{-1}\tilde{f}_h +  \psi_h(s,a)^\top (\lambda\cdot I + \cK_h)^{-1}\Phi_h\tilde{f}_h -   \psi_h(s,a)^\top (\lambda \cdot I +\cK_h )^{-1} \yb_h|\\
&\qquad = |\lambda  \phi(s,a)^\top\Lambda_h^{-1}\tilde{f}_h +  \psi_h^k(s,a)^\top (\lambda\cdot I + \cK_h)^{-1}(\Phi_h\tilde{f}_h -  \yb_h)|,
\end{align*}
where the first inequality is due to $0\leq \PP_h V_{h+1}(s,a)\leq H$, the non-expansiveness of the operator $\Pi_{[0, H]}$, and the definition of $\hat{f}_h(s,a)$ in \eqref{eq:kernel_plan_regression}, and the first equality is by the same reformulation as \eqref{eq:reform_phi} such that 
\begin{align*}
\phi(s,a) =\lambda \Lambda_h^{-1}\phi(s,a) + (\Phi_h)^\top (\lambda\cdot I + \cK_h)^{-1}\psi_h (s,a).
\end{align*}
Thus, we have
\begin{align}
\begin{aligned}\label{eq:approx_err_plan}
|\PP_h V_{h+1}(s,a) - f_h(s,a) | &\leq \underbrace{\lambda  \|\phi(s,a)^\top\Lambda_h^{-1}\|_\cH \cdot \|\tilde{f}_h\|_\cH}_{\text{Term(I)}} \\
&\quad +  \underbrace{|\psi_h(s,a)^\top (\lambda\cdot I + \cK_h)^{-1}(\Phi_h\tilde{f}_h -  \yb_h)|}_{\text{Term(II)}}.
\end{aligned}
\end{align}
Analogous to  \eqref{eq:bound_termI}, for Term(I) here, we have
\begin{align*}
\text{Term(I)} \leq \sqrt{\lambda} R_Q H \|\phi(s,a)\|_{\Lambda_h^{-1}}.
\end{align*}
Similar to \eqref{eq:bound_termII}, for Term(II), we have 
\begin{align*}
\text{Term(II)} \leq \| \phi(s,a)\|_{\Lambda_h^{-1}} \left \|\sum_{\tau=1}^{K} \phi(s_h^\tau, a_h^\tau) [V_{h+1}(s_{h+1}^\tau) - \PP_h V_{h+1}(s_h^\tau, a_h^\tau)] \right\|_{\Lambda_h^{-1}}.
\end{align*}
Then, we need to bound the last factor in the above inequality. Here we apply the similar argument as Lemma \ref{lem:approx_concentrate}. We have the function class for $V_h$ is
\begin{align*}
\overline{\cV}(r_h, \tilde{R}_K, \tilde{B}_K)= \{ V: V(\cdot) = \max_{a\in \cA} Q(\cdot, a) \text{ with } Q\in \overline{\cQ}(r_h,\tilde{R}_K, \tilde{B}_K) \}.
\end{align*}
By Lemma \ref{lem:self_normalize_uniform} with $\delta'$, we have
\begin{align*}
&\left \|\sum_{\tau=1}^K \phi(s_h^\tau, a_h^\tau) [V_{h+1}(s_{h+1}^\tau) - \PP_h V_{h+1}(s_h^\tau, a_h^\tau)] \right\|_{(\Lambda_h)^{-1}}^2 \\
&\qquad \leq \sup_{V\in \overline{\cV}(r_h, \tilde{R}_K, \tilde{B}_K)} \left \|\sum_{\tau=1}^K \phi(s_h^\tau, a_h^\tau) [V(s_{h+1}^\tau) - \PP_h V(s_h^\tau, a_h^\tau)] \right\|_{(\Lambda_h)^{-1}}^2  \\
&\qquad \leq  2H^2 \log\det(I+\cK/\lambda) + 2H^2K(\lambda-1)+4H^2\log(\cN^{\overline{\cV}}_{\dist}(\epsilon; \tilde{R}_K, \tilde{B}_K)/\delta')+ 8K^2\epsilon^2/\lambda\\
&\qquad \leq  4H^2 \Gamma(K, \lambda; \ker) + 10H^2+4H^2\log\cN_{\infty}(\varsigma^*;\tilde{R}_K, \tilde{B}_K) + 4H^2\log(1 /\delta'),
\end{align*}
where the last inequality is by setting $\lambda = 1+1/K$ and $\epsilon = \varsigma^* =  H/K$, and also due to 
\begin{align*}
\cN^{\overline{\cV}}_{\dist}(\varsigma^*; \tilde{R}_K, \tilde{B}_K) \leq \cN_{\infty}(\varsigma^*; \tilde{R}_K, \tilde{B}_K).
\end{align*}
We have that with probability at least $1-\delta'$, the following inequality holds for all $k\in [K]$
\begin{align*}
&\left \|\sum_{\tau=1}^K \phi(s_h^\tau, a_h^\tau) [V_{h+1}(s_{h+1}^\tau) - \PP_h V_{h+1}^k(s_h^\tau, a_h^\tau)] \right\|_{\Lambda_h^{-1}} \\
&\qquad  \leq  [4H^2 \Gamma(K, \lambda; \ker) + 10H^2+4H^2\log\cN_{\infty}(\varsigma^*;\tilde{R}_K,  \tilde{B}_K) + 4H^2\log(K /\delta')]^{1/2}. 
\end{align*}
Thus, Term(II) can be further bounded as 
\begin{align*}
\text{Term(II)} \leq H\big[ 4 \Gamma(K, \lambda; \ker) + 10+4\log\cN_{\infty}(\varsigma^*;\tilde{R}_K, \tilde{B}_K) + 4\log(K /\delta') \big]^{1/2}\| \phi(s,a)\|_{(\Lambda_h^k)^{-1}} .
\end{align*}
Plugging the upper bounds of Term(I) and Term(II) into \eqref{eq:approx_err_plan}, we obtain
\begin{align*}
&|\PP_h V_{h+1}(s,a) - f_h(s,a) |  \\
&\qquad \leq u_h(s,a)\leq \beta \|\phi(s,a)\|_{\Lambda_h^{-1}} = \beta \lambda^{-1/2} [ \ker(z,z)- \psi_h(s,a)^\top (\lambda I + \cK_h )^{-1}  \psi_h(s,a)]^{1/2} ,
\end{align*}
where we let $z = (s,a)$,  $\varsigma^* = H/K$, and $\lambda = 1 + 1/K$. In the last equality, similar to \eqref{eq:bonus_kernel_form}, we have 
\begin{align} 
\begin{aligned}\label{eq:bonus_kernel_form_plan}
\|\phi(s,a)\|^2_{\Lambda_h^{-1}} 
&= \lambda^{-1} \phi(s,a)^\top \phi(s,a)  - \lambda^{-1} \phi(s,a)^\top (\Phi_h)^\top  [\lambda I + \Phi_h (\Phi_h)^\top ]^{-1}\Phi_h \phi(s,a) \\
&= \lambda^{-1} \ker(z,z)- \lambda^{-1} \psi_h(s,a)^\top (\lambda I + \cK_h )^{-1}  \psi_h(s,a).
\end{aligned}
\end{align}
Similar to Lemma \ref{lem:linear_approx_param_bound}, we know that the function $\hat{f}_h$ satisfies  $\|\hat{f}_h\|_\cH \leq H\sqrt{2K/\lambda \cdot \log\det(I + \cK_h^k/\lambda)} \leq 2H \sqrt{K\cdot \Gamma (K, \lambda; \ker)}$. Then, one can set $\tilde{R}_K = 2H \sqrt{K\cdot \Gamma (K, \lambda; \ker)}$.  Moreover, we set $\beta = \tilde{B}_K$. Thus, we let
\begin{align*}
 H\big[ 2\lambda R^2_Q + 8 \Gamma(K, \lambda;\ker) + 20+4\log\cN_{\infty}(\varsigma^*;\tilde{R}_K,  \tilde{B}_K) + 8\log(K /\delta')  \big]^{1/2} \leq \beta = \tilde{B}_K,
\end{align*}
which can be further guaranteed by 
\begin{align*}
4H^2\big[ R^2_Q  + 2 \Gamma(K, \lambda; \ker) + 5+\log\cN_{\infty}(\varsigma^*; \tilde{R}_K,  \tilde{B}_K) + 2\log(K /\delta')  \big] \leq  \tilde{B}^2_K
\end{align*}
as $(1+1/H) \leq 2$ and $\lambda = 1+1/K \leq 2$. This completes the proof.
\end{proof}

\begin{lemma} \label{lem:bonus_plan} Conditioned on the event $\tilde{\cE}$ as defined in Lemma \ref{lem:bonus_concentrate_plan}, we have
\begin{align*}
V_h^*(s, r) \leq V_h(s) \leq r_h(s,\pi_h(s)) + \PP_h V_{h+1}(s,\pi_h(s)) + 2 u_h(s,\pi_h(s)), \forall s \in \cS, \forall h \in [H],
\end{align*}
where $\pi_h(s) = \argmax_{a\in \cA} Q_h(s,a)$.
\end{lemma}
\begin{proof} We first prove the first inequality in this lemma. We prove it by induction. For $h = H+1$, by the planning algorithm, we have $V_{H+1}^*(s, r)  = V_{H+1}(s) = 0$ for any $s\in \cS$. Then, we assume that $V_{h+1}^*(s, r)  \leq V_{h+1}(s)$. Thus, conditioned on the event $\tilde{\cE}$ as defined in Lemma \ref{lem:bonus_concentrate_plan}, we have
\begin{align*}
&Q_h^*(s, a, r) - Q_h(s, a) \\
&\qquad= r_h(s,a) + \PP_h V_{h+1}^*(s,a,r)  -  \min \{ r_h(s,a) + f_h(s,a) + u_h(s,a), H \}^+  \\
&\qquad\leq  \max \{ \PP_h V_{h+1}^*(s,a,r)  - f_h(s,a) - u_h(s,a), 0 \}  \\
&\qquad\leq  \max \{ \PP_h V_{h+1}(s,a)  - f_h(s,a) - u_h(s,a), 0 \}  \\
&\qquad\leq 0
\end{align*}
where the first inequality is due to $0 \leq r_h(s,a) + \PP_h V_{h+1}^*(s,a,r) \leq H$ and $\min \{x, H\}^+ \geq \min \{x, H\}$, the second inequality is by the assumption that $V_{h+1}^*(s,a, r)\leq V_{h+1}(s,a)$, the last inequality is by Lemma \ref{lem:bonus_concentrate_plan} such that $|\PP_h V_{h+1}(s,a) - f_h(s,a) | \leq u_h(s,a)$ holds for any $(s,a)\in \cS \times \cA$ and $(k,h)\in [K]\times[H]$. The above inequality further leads to 
\begin{align*}
V_h^*(s, r) = \max_{a\in \cA} Q_h^*(s, a, r) \leq \max_{a\in \cA} Q_h(s, a) = V_h(s). 
\end{align*}
Therefore, we have
\begin{align*}
V_h^*(s, r)  \leq  V_h(s), \forall h\in [H], \forall s\in \cS. 
\end{align*}
In addition, we prove the second inequality in this lemma. We have
\begin{align*}
Q_h(s, a) &= \min \{ r_h(s,a) + f_h(s,a) + u_h(s,a), H \}^+\\
& \leq \min \{ r_h(s,a) + \PP_h V_{h+1}(s,a) + 2u_h(s,a), H \}^+\\
& \leq r_h(s,a) + \PP_h V_{h+1}(s,a) + 2u_h(s,a),
\end{align*}
where the first inequality is also by Lemma \ref{lem:bonus_concentrate_plan} such that $|\PP_h V_{h+1}(s,a) - f_h(s,a) | \leq u_h(s,a)$, and the last inequality is because of the non-negativity of  $r_h(s,a) + \PP_h V_{h+1}(s,a) + 2u_h(s,a)$. Therefore, we have
\begin{align*}
V_h(s) &= \max_{a\in \cA}Q_h(s, a) = Q_h(s, \pi_h(s)) \\
&\leq r_h(s,\pi_h(s)) + \PP_h V_{h+1}(s,\pi_h(s)) + 2u_h(s,\pi_h(s)).
\end{align*}
This completes the proof.
\end{proof}

\begin{lemma} \label{lem:explore_plan_connect} With the exploration and planning phases, we have the following inequality 
\begin{align*}
K \cdot V_1^*(s_1, u/H) \leq  \sum_{k=1}^K V_1^*(s_1, r^k).
\end{align*}

\end{lemma}

\begin{proof} As shown in \eqref{eq:bonus_kernel_form_plan}, we know that
\begin{align*}
w_h(s,a)=\beta\|\phi(s,a)\|_{\Lambda_h^{-1}} =  \beta \sqrt{ \phi(s,a)^\top \left[\lambda I_\cH + \sum_{\tau=1}^K \phi(s_h^\tau, a_h^\tau)\phi(s_h^\tau, a_h^\tau)^\top\right]^{-1} \phi(s,a)}.
\end{align*}
On the other hand, by \eqref{eq:bonus_kernel_form}, we similarly have
\begin{align*}
w_h^k(s,a) = \beta \|\phi(s,a)\|_{(\Lambda^k_h)^{-1}} = \beta \sqrt{ \phi(s,a)^\top \left[\lambda I_\cH + \sum_{\tau=1}^{k-1} \phi(s_h^\tau, a_h^\tau)\phi(s_h^\tau, a_h^\tau)^\top \right]^{-1} \phi(s,a)}.
\end{align*}
Since $k-1 \leq K$ and $f^\top \phi(s_h^\tau, a_h^\tau)\phi(s_h^\tau, a_h^\tau)^\top f = [f^\top \phi(s_h^\tau, a_h^\tau)]^2\geq 0$ for any $\tau$, then we know that \begin{align*}
\Lambda_h = \lambda I_\cH + \sum_{\tau=1}^K \phi(s_h^\tau, a_h^\tau)\phi(s_h^\tau, a_h^\tau)^\top  \succcurlyeq \lambda I_\cH + \sum_{\tau=1}^{k-1} \phi(s_h^\tau, a_h^\tau)\phi(s_h^\tau, a_h^\tau)^\top = \Lambda_h^k.
\end{align*}
We use $A\succcurlyeq B$ (or $A\succ B$) to denote $f^\top A f \geq f^\top Bf$ (or $f^\top A f > f^\top Bf$), $\forall f\in \cH,$ for two self-adjoint operators $A$ and $B$. Moreover, if a linear operator $A$ satisfies $A \succ 0$, we say $A$ is a positive operator.

The above relation further implies that $(\Lambda_h^k)^{-1} \succcurlyeq \Lambda_h^{-1}$ such that we have $\phi(s,a)^\top \Lambda_h^{-1} \phi(s,a) \leq \allowbreak \phi(s,a)^\top (\Lambda_h^k)^{-1} \phi(s,a)$, where $(\Lambda_h^k)^{-1}$ and $\Lambda_h^{-1}$ are the inverse of $\Lambda_h^k$ and $\Lambda_h$ respectively. Here we use the fact that $\Lambda_h  \succcurlyeq \Lambda_h^k$ implies $(\Lambda_h^k)^{-1} \succcurlyeq \Lambda_h^{-1}$ , which can be proved by extending the standard matrix case to the self-adjoint operator. For completeness, we give a short proof below. 

Let $\lambda > 0$ be a fixed constant. Since $\Lambda_h \succcurlyeq \Lambda_h^k \succcurlyeq \lambda I_\cH \succ 0$, then there exist the inverse $\Lambda_h^{-1}$, $(\Lambda_h^k)^{-1}$ and square root $\Lambda_h^{1/2}$, $(\Lambda_h^k)^{1/2}$, which are also positive self-adjoint and invertible operators.
We also have $\Lambda_h^{-1/2}:=(\Lambda_h^{1/2})^{-1} = (\Lambda_h^{-1})^{1/2}$ and $(\Lambda_h^k)^{-1/2}:=[(\Lambda_h^k)^{1/2}]^{-1} = [(\Lambda_h^k)^{-1}]^{1/2}$. 
Thus, for any $f\in \cH$, we have $f^\top f = f^\top \Lambda_h^{-1/2} \Lambda_h^{1/2} \Lambda_h^{1/2} \Lambda_h^{-1/2} f = f^\top \Lambda_h^{-1/2} \Lambda_h \Lambda_h^{-1/2} f \geq f^\top \Lambda_h^{-1/2} \Lambda_h^k \Lambda_h^{-1/2} f$ where the inequality is due to $\Lambda_h \succcurlyeq \Lambda_h^k$ and $\Lambda_h^{-1/2} = (\Lambda_h^{-1/2})^\top$. Then, we further have $f^\top f \geq f^\top \Lambda_h^{-1/2} \Lambda_h^k \Lambda_h^{-1/2} f = f^\top \Lambda_h^{-1/2} (\Lambda_h^k)^{1/2} (\Lambda_h^k)^{1/2} \Lambda_h^{-1/2} f = f^\top A^\top A f$ if we let $A = (\Lambda_h^k)^{1/2} \Lambda_h^{-1/2}$, where we use the fact that $(\Lambda_h^k)^{1/2}$ and $\Lambda_h^{-1/2}$ are self-adjoint operators. Then, we know that $\|f\|_\cH \geq  \|A f\|_\cH$ holds for all $f\in \cH$, indicating that $\|A\|_\op:=\sup_{f\neq \boldsymbol 0} \|Af\|_\cH/\|f\|_\cH \leq 1$, where $\|\cdot\|_\op$ denotes the operator norm. Since $\|A\|_\op = \|A^\top\|_\op$, we have $\|A^\top\|_\op \leq 1$ or equivalently $\|f\|_\cH \geq  \|A^\top f\|_\cH, \forall f\in \cH$, which gives $f^\top f \geq f^\top (\Lambda_h^k)^{1/2} \Lambda_h^{-1/2} \Lambda_h^{-1/2} (\Lambda_h^k)^{1/2}  f = f^\top (\Lambda_h^k)^{1/2} \Lambda_h^{-1}(\Lambda_h^k)^{1/2} f$. For any $g\in \cH$, letting $f = (\Lambda_h^k)^{-1/2} g$, by $f^\top f \geq f^\top (\Lambda_h^k)^{1/2} \Lambda_h^{-1}(\Lambda_h^k)^{1/2} f$, we have $g^\top (\Lambda_h^k)^{-1} g \geq g^\top \Lambda_h^{-1} g$, which gives $(\Lambda_h^k)^{-1} \succcurlyeq \Lambda_h^{-1}$. The above derivation is based on the basic properties of the linear operator, the (self-)adjoint operator, the inverse, and the square root of an operator. See \citet{kreyszig1978introductory,schechter2001principles,maccluer2008elementary} for the details.

Thus, by the above result, we have
\begin{align*}
w_h(s,a)\leq w_h^k(s,a).
\end{align*}
Since $r_h^k = 1/H\cdot u_h^k(s,a) = 1/H\cdot \min\{w_h^k(s,a), H\}$ and $u_h(s,a) =\min\{w_h(s,a), H\}$, then we have
\begin{align*}
u_h(s,a)/H \leq r_h^k(s,a),
\end{align*}
such that
\begin{align*}
V_1^*(s_1, u/H) \leq  V_1^*(s_1, r^k), 
\end{align*}
and thus
\begin{align*}
K \cdot V_1^*(s_1, u/H) \leq  \sum_{k=1}^K V_1^*(s_1, r^k).
\end{align*}
This completes the proof.
\end{proof}

\subsection{Proof of Theorem \ref{thm:main_kernel_single}} \label{sec:proof_main_kernel_single}
\begin{proof} Conditioned on the event $\cE$ defined in Lemma \ref{lem:bonus_concentrate} and the event $\tilde{\cE}$ defined in Lemma \ref{lem:bonus_concentrate_plan}, we have
\begin{align}\label{eq:proof_start_kernel}
V_1^*(s_1, r) - V_1^\pi(s_1, r) \leq V_1(s_1) - V_1^\pi(s_1, r),
\end{align}
where the inequality is by Lemma \ref{lem:bonus_plan}. Further by this lemma, we have
\begin{align*}
V_h(s) - V_h^\pi(s, r) &\leq  r_h(s,\pi_h(s)) + \PP_h V_{h+1}(s,\pi_h(s)) + 2u_h(s,\pi_h(s))  - Q_h^\pi(s, \pi_h(s), r)\\
&= r_h(s,\pi_h(s)) + \PP_h V_{h+1}(s,\pi_h(s)) + 2u_h(s,\pi_h(s)) \\
&\quad - r_h(s,\pi_h(s)) - \PP_h V_{h+1}^\pi(s,\pi_h(s), r) \\
&= \PP_h V_{h+1} (s,\pi_h(s)) - \PP_h V_{h+1}^\pi(s,\pi_h(s), r) + 2u_h(s,\pi_h(s)).
\end{align*}
Recursively applying the above inequality and making use of $V_{H+1}^\pi(s, r) = V_{H+1} (s)  = 0$ gives
\begin{align*}
V_1(s_1) - V_1^\pi(s_1,r) &\leq  \EE_{\forall h\in [H]: ~s_{h+1}\sim\PP_h(\cdot|s_h, \pi_h(s_h))}\left[\sum_{h=1}^H 2u_h(s_h,\pi_h(s_h))\Bigg| s_1 \right]\\
&=2H \cdot V_1^\pi(s_1, u/H). 
\end{align*}
Combining this inequality with \eqref{eq:proof_start_kernel} gives
\begin{align*}
V_1^*(s_1, r) - V_1^\pi(s_1, r) &\leq 2H \cdot V_1^\pi(s_1, u/H) \leq \frac{2H}{K} \sum_{k=1}^K V_1^*(s_1, r^k) \\
&\leq \frac{2H}{K} \cO\left(\sqrt{H^3 K \log (1/\delta')} + \beta\sqrt{H^2 K \cdot \Gamma(K, \lambda; \ker)}\right)\\
&= \cO\left([\sqrt{H^5 \log (1/\delta')} + \beta\sqrt{H^4 \cdot \Gamma(K, \lambda; \ker)}] /\sqrt{K} \right),
\end{align*}
where the second inequality is due to Lemma \ref{lem:explore_plan_connect} and the third inequality is by Lemma \ref{lem:bonus_explore}. 

By the union bound, we have $P(\cE \wedge \tilde{\cE}) \geq 1-2\delta'$ . Therefore, by setting $\delta' = \delta / 2$, we obtain that with probability at least $1-\delta$
\begin{align*}
V_1^*(s_1, r) - V_1^\pi(s_1, r) \leq  \cO\left([\sqrt{H^5 \log (2/\delta)} + \beta\sqrt{H^4 \cdot \Gamma(K, \lambda; \ker)}] /\sqrt{K} \right).
\end{align*}
Note that $\cE \wedge \tilde{\cE} $ happens when the following two conditions are satisfied, i.e., 
\begin{align*}
&4H^2\big[ R^2_Q  + 2 \Gamma(K, \lambda; \ker) + 5+\log\cN_{\infty}(\varsigma^*; \tilde{R}_K,  \tilde{B}_K) + 2\log(2K /\delta)  \big] \leq  \tilde{B}_K^2,\\
&16H^2\big[ R^2_Q  + 2 \Gamma(K, \lambda; \ker) + 5+\log\cN_{\infty}(\varsigma^*;R_K,  B_K) + 2\log(2K /\delta)  \big] \leq  B^2_K, \forall h\in [H],
\end{align*}
where $\beta = \tilde{B}_K$, $(1+1/H)\beta = B_K$, $\lambda = 1+1/K$, $\tilde{R}_K = R_K=2H\sqrt{\Gamma(K, \lambda; \ker)}$, and $\varsigma^* = H/K$. The above inequalities hold if we further let $\beta$ satisfy
\begin{align*}
16H^2\big[ R^2_Q  + 2 \Gamma(K, \lambda; \ker) + 5+\log\cN_{\infty}(\varsigma^*;R_K,  2\beta) + 2\log(2K /\delta)  \big] \leq  \beta^2, \forall h\in [H],
\end{align*}
since $2\beta\geq (1+1/H)\beta\geq \beta$ such that $\cN_{\infty}(\varsigma^*;R_K,  2\beta)\geq \cN_{\infty}(\varsigma^*;R_K,  B_K) \geq \cN_{\infty}(\varsigma^*;\tilde{R}_K,  \tilde{B}_K)$.  Since the above conditions imply that $\beta \geq H$, further setting $\delta = 1/(2K^2H^2)$, we obtain that 
\begin{align*}
V_1^*(s_1, r) - V_1^\pi(s_1, r) \leq \cO\left(\beta\sqrt{H^4 [ \Gamma(K, \lambda; \ker) + \log(KH)]} /\sqrt{K} \right), 
\end{align*}
with further letting 
\begin{align*}
16H^2\big[ R^2_Q  + 2 \Gamma(K, \lambda; \ker) + 5+\log\cN_{\infty}(\varsigma^*;R_K,  2\beta) + 6\log(2K H)  \big] \leq  \beta^2, \forall h\in [H].
\end{align*}
This completes the proof.
\end{proof}

\section{Proofs for Single-Agent MDP with Neural Function Approximation}

\subsection{Lemmas}

\begin{lemma} [Lemma C.7 of \citet{yang2020provably}]\label{lem:regularity_neural} With $KH^2 = \cO(m\log^{-6}m)$, then there exists a constant $\digamma \geq 1$ such that the following inequalities hold with probability at least $1-1/m^2$ for any $z\in \cS\times \cA$ and any $W \in \{W : \|W-\W0\|_2\leq H\sqrt{K/\lambda} \}$,
\begin{align*}
&|f(z;W)-\varphi(z;\W0)^\top(W-\W0)| \leq \digamma  K^{2/3} H^{4/3} m^{-1/6} \sqrt{\log m},\\
&\|\varphi(z;W) - \varphi(z;\W0)\|_2 \leq \digamma (KH^2/m)^{1/6} \sqrt{\log m}, \qquad \|\varphi(z;W) \|_2\leq \digamma.
\end{align*}
\end{lemma}

\begin{lemma} \label{lem:bonus_concentrate_neural} We define the event $\cE$ as that the following inequality holds $\forall (s,a) \in \cS \times \cA, \forall (h,k) \in [H] \times [K]$,
\begin{align*}
&|\PP_h V_{h+1}^k(s,a) - f_h^k(s,a)|  \leq u_h^k(s,a) + \beta \iota,\\
&\left|\|\varphi(z; \Whk)\|_{\Lhki} - \|\varphi(z; \W0)\|_{\tLhki}\right| \leq \iota,
\end{align*}
where $\iota = 5K^{7/12} H^{1/6} m^{-1/12} \log^{1/4} m$ and we define
\begin{align*}
&\Lambda_h^k = \sum_{\tau=1}^{k-1} \varphi(s_h^{\tau}, a_h^{\tau}; \Whk) \varphi(s_h^{\tau}, a_h^{\tau}; \Whk)^\top + \lambda I, \ \   \tilde{\Lambda}_h^k = \sum_{\tau=1}^{k-1} \varphi(s_h^{\tau}, a_h^{\tau}; \W0) \varphi(s_h^{\tau}, a_h^{\tau}; \W0)^\top + \lambda I.
\end{align*}
Setting $(1+1/H)\beta = B_K$, $R_K = H\sqrt{K}$, $\varsigma^* = H/K$, and $\lambda=\F^2(1+1/K)$, $\varsigma^* = H/K$, if we set
\begin{align*}
\beta^2 \geq H^2[8R_Q^2  (1+\sqrt{\lambda/d})^2 + 32 \Gamma(K, \lambda; \ker_m) + 80+32\log\cN_{\infty}(\varsigma^*;R_K,  B_K) + 32\log(K /\delta')],
\end{align*}
and also 
\begin{align*}
m = \Omega(K^{19} H^{14}\log^3 m),
\end{align*}
then we have that with probability at least $1-2/m^2 - \delta'$, the event $\cE$ happens, i.e.,
\begin{align*}
\Pr(\cE) \geq 1-2/m^2 - \delta'.
\end{align*}
\end{lemma}

\begin{proof}Recall that we assume $\PP_h V_{h+1}$ for any $V$ can be expressed as
\begin{align*}
\PP_h V_{h+1}(z) = \int_{\RR^d}  \act'(\boldsymbol{\omega}^\top z) \cdot z^\top \boldsymbol{\alpha}(\boldsymbol\omega) \mathrm{d} p_0(\bomega),
\end{align*}
which thus implies that we have
\begin{align*}
\PP_h V_{h+1}^k(z) = \int_{\RR^d}  \act'(\boldsymbol{\omega}^\top z) \cdot z^\top \boldsymbol{\alpha}_h^k(\boldsymbol\omega) \mathrm{d} p_0(\bomega),
\end{align*}
for some $\balpha_h^k(\bomega)$. Our algorithm suggests to estimate $\PP_h V_{h+1}^k(s,a)$ via learning the parameters $W_h^k$ by solving
\begin{align} \label{eq:solve_Whk}
W_h^k = \argmin_{W} \sum_{\tau=1}^{k-1}[V_{h+1}^k(s_{h+1}^\tau) - f(s_h^\tau, a_h^\tau; W)]^2 + \lambda \|W-\W0\|_2^2,
\end{align}
such that we have the estimate of $\PP_h V_{h+1}^k(s,a)$ as $f_h^k(z) = \Pi_{[0,H]}[f(z;W_h^k)]$ with  
\begin{align*}
f(z; W_h^k) = \frac{1}{\sqrt{2m}} \sum_{i=1}^{2m} v_i \cdot \act([W_h^k]_i^\top z).
\end{align*}
Furthermore, we have
\begin{align*}
\|W_h^k-\W0\|_2^2 &\leq \frac{1}{\lambda } \left( \sum_{\tau=1}^{k-1}[V_{h+1}^k(s_{h+1}^\tau) - f(s_h^\tau, a_h^\tau; W_h^k)]^2 + \lambda \|W_h^k-\W0\|_2^2\right)\\
&\leq \frac{1}{\lambda } \left( \sum_{\tau=1}^{k-1}[V_{h+1}^k(s_{h+1}^\tau) - f(s_h^\tau, a_h^\tau; \W0)]^2 + \lambda \|\W0-\W0\|_2^2\right) \\
&= \frac{1}{\lambda } \sum_{\tau=1}^{k-1}[V_{h+1}^k(s_{h+1}^\tau)]^2\leq H^2K/\lambda, 
\end{align*}
where the second inequality is due to $W_h^k$ is the minimizer of the objective function.

We also define a linearization of the function $f(z;W)$ at the point $\W0$, which is
\begin{align}\label{eq:f_lin_form}
f_{\lin}(z; W) = f(z; W^{(0)}) + \langle \varphi(z;\W0), W-W^{(0)}\rangle = \langle \varphi(z;\W0), W-W^{(0)}\rangle,
\end{align}
where 
\begin{align*}
\varphi(z;W) = \nabla_W f(z;W) =[\nabla_{W_1} f(z;W), \cdots, \nabla_{W_{2m}} f(z;W)].
\end{align*}
Based on this linearization formulation, we similarly define a parameter matrix $W_{\lin, h}^k$ that is generated by solving an optimization problem with the linearied function $f_{\lin}$, such that
\begin{align} \label{eq:solve_W_lin_hk}
W_{\lin, h}^k = \argmin_{W} \sum_{\tau=1}^{k-1}[V_{h+1}^k(s_{h+1}^\tau) - f_\lin(s_h^\tau, a_h^\tau; W)]^2 + \lambda \|W-\W0\|_2^2.
\end{align}
Due to the linear structure of $f_{\lin}(z; W)$, one can easily solve the above optimization problem and obtain the closed form of the solution $W_{\lin, h}^k$, which is 
\begin{align} \label{eq:W_lin_form}
W_{\lin, h}^k = \W0 + (\tilde{\Lambda}_h^t)^{-1} (\tilde{\Phi}_h^k)^\top \yb_h^k,
\end{align}
where we define $\Lambda_h^t$, $\Phi_h^k$, and $\yb_h^k$ as
\begin{align*}
&\tilde{\Phi}_h^k = [\varphi(s_h^1,a_h^1; \W0), \cdots, \varphi(s_h^{k-1},a_h^{k-1}; \W0)]^\top, \\
&\tilde{\Lambda}_h^k = \sum_{\tau=1}^{k-1} \varphi(s_h^{\tau}, a_h^{\tau}; \W0) \varphi(s_h^{\tau}, a_h^{\tau}; \W0)^\top + \lambda \cdot I = \lambda \cdot I + (\tilde{\Phi}_h^k)^\top \tilde{\Phi}_h^k   , \\
&\yb_h^k = [V_{h+1}^k(s_{h+1}^1), V_{h+1}^k(s_{h+1}^2) , \cdots, V_{h+1}^k(s_{h+1}^{k-1})  ]^\top.
\end{align*}
Here we also have the upper bound of $\|W_{\lin,h}^k-\W0\|_2$ as
\begin{align*}
\|W_{\lin,h}^k-\W0\|_2^2 &\leq \frac{1}{\lambda } \left( \sum_{\tau=1}^{k-1}[V_{h+1}^k(s_{h+1}^\tau) - f_\lin(s_h^\tau, a_h^\tau; W_{\lin,h}^k)]^2 + \lambda \|W_{\lin,h}^k-\W0\|_2^2\right)\\
&\leq \frac{1}{\lambda } \left( \sum_{\tau=1}^{k-1}[V_{h+1}^k(s_{h+1}^\tau) - f_\lin(s_h^\tau, a_h^\tau; \W0)]^2 + \lambda \|\W0-\W0\|_2^2\right) \\
&= \frac{1}{\lambda } \sum_{\tau=1}^{k-1}[V_{h+1}^k(s_{h+1}^\tau)]^2\leq H^2K/\lambda, 
\end{align*}
where the second inequality is due to $W_{\lin,h}^k$ is the minimizer of the objective function. Based on the matrix $W_{\lin, h}^k$, we define the function 
\begin{align*}
f_{\lin, h}^k(z) := \Pi_{[0, H]} [f_{\lin}(z; W_{\lin,h}^k)],
\end{align*}
where $\Pi_{[0, H]} [\cdot]$ is short for $\min\{\cdot, H\}^+$.

Moreover, we further define an approximation of $\PP_h V^k_{h+1}$ as 
\begin{align*}
\tilde{f}(z) =\Pi_{[0,H]}\left[\frac{1}{\sqrt{m}} \sum_{i=1}^m \act'(\W0_i{}^\top z) z^\top \balpha_i\right],
\end{align*}
where $\|\balpha_i\|\leq R_Q H /\sqrt{dm}$. According to \citet{gao2019convergence}, we have that with probability at least $1-1/m^2$ over the randomness of initialization, for any $(h,k)\in [H]\times [K]$, there exists a constant $C_\act$ such that $\forall z=(s,a)\in \cS\times \cA$, we have
\begin{align*} 
\left|\PP_h V^k_{h+1}(z) - \frac{1}{\sqrt{m}} \sum_{i=1}^m \act'(\W0_i{}^\top z) z^\top \balpha_i \right| \leq 10 C_\act R_Q H  \sqrt{\log(mKH)/m}.
\end{align*}
which further implies that 
\begin{align}\label{eq:real_kernel_approx_neural}
|\PP_h V^k_{h+1}(z) - \tilde{f}(z)| \leq 10 C_\act R_Q H  \sqrt{\log(mKH)/m}, ~~\forall z=(s,a)\in \cS\times \cA.
\end{align}
This indicates that $\tilde{f}(z)$ is a good estimate of $\PP_h V^k_{h+1}(z)$ particularly when $m$ is large, i.e., the estimation error $10 C_\act R_Q H  \sqrt{\log(mKH)/m}$ is small.

Now, based on the above definitions and descriptions, we are ready to present our proof of this lemma. Overall, the basic idea of proving the upper bound of $|P_hV_{h+1}^k(z) - f_h^k(z)|$ is to bound the following difference terms, i.e., 
\begin{align}\label{eq:two_diff_neural}
|f_h^k(z) - f_{\lin, h}^k(z) | ~~\text{ and }~~  |f_{\lin, h}^k(z) - \tilde{f}(z)|.
\end{align} 
As we already have known the upper bound of the term $|\PP_h V^h_{h+1}(z) - \tilde{f}(z)|$ in \eqref{eq:real_kernel_approx_neural}, one can immediately obtain the upper bound of $|\PP_hV_{h+1}^k(z) - f_h^k(z)|$ by decomposing it into the two aforementioned terms and bounding them separately.

We first bound the first term in \eqref{eq:two_diff_neural}, i.e., $|f_h^k(z) - f_\lin(z; W_{\lin, h}^k) |$, in the following way
\begin{align}
\begin{aligned} \label{eq:bound_overall_1}
&|f_h^k(z) - f_{\lin, h}^k(z) | \\
&\qquad\leq |f(z; W_h^k) - \langle \varphi(z;\W0), W_{\lin, h}^k-\W0\rangle |\\
 &\qquad\leq |f(z; W_h^k) - \langle \varphi(z;\W0), W_h^k-\W0\rangle |  +  | \langle \varphi(z;\W0), W_h^k-W_{\lin, h}^k\rangle | \\
 &\qquad\leq   \digamma  K^{2/3} H^{4/3} m^{-1/6} \sqrt{\log m} + \digamma \underbrace{\|W_h^k-W_{\lin, h}^k\|_2}_{\text{Term(I)}},
\end{aligned}
\end{align}
where the first inequality is due to the non-expansiveness of projection operation $\Pi_{[0, H]}$, the third inequality is by Lemma \ref{lem:regularity_neural} that holds with probability at least $1-m^{-2}$. Then, we need to bound Term(I) in the above inequality. Specifically, by the first order optimality condition for the objectives in \eqref{eq:solve_Whk} and \eqref{eq:solve_W_lin_hk}, we have
\begin{align*}
\lambda (\Whk - \W0 ) &= \sum_{\tau=1}^{k-1}[V_{h+1}^k(s_{h+1}^\tau) - f(z_h^\tau; \Whk)]\varphi(z_h^\tau; \Whk) = \Phkt(\yb_h^k - \fb_h^k),\\
\lambda (\Wlinhk - \W0 ) &= \sum_{\tau=1}^{k-1}[V_{h+1}^k(s_{h+1}^\tau) - \langle\varphi(z_h^\tau;\W0), \Wlinhk-\W0 \rangle]\varphi(z_h^\tau; \W0) \\
&= \tPhkt\yb_h^k - \tPhkt \tPhk (\Wlinhk - \W0),
\end{align*}
where we define
\begin{align*}
&\Phi_h^k = [\varphi(s_h^1,a_h^1; \Whk), \cdots, \varphi(s_h^{k-1},a_h^{k-1}; \Whk)]^\top, \\
&\Lambda_h^k = \sum_{\tau=1}^{k-1} \varphi(s_h^{\tau}, a_h^{\tau}; \Whk) \varphi(s_h^{\tau}, a_h^{\tau}; \Whk)^\top + \lambda \cdot I = \lambda \cdot I + (\Phi_h^k)^\top \Phi_h^k   , \\
&\fb_h^k = [f(z_h^1;\Whk), f(z_h^2;\Whk) , \cdots, f(z_h^{k-1};\Whk)  ]^\top.
\end{align*}
Thus, we have
\begin{align*}
\text{Term(I)} &= \lambda^{-1}\|\Phkt(\yb_h^k - \fb_h^k) -  \tPhkt\yb_h^k + \tPhkt \tPhk (\Wlinhk - \W0)\|_2 \\
&= \lambda^{-1}\|\Phkt(\yb_h^k - \fb_h^k) -  \tPhkt\yb_h^k + \tPhkt \tPhk (\Wlinhk - \W0)\|_2 \\
&\leq  \lambda^{-1}\|(\Phkt-  \tPhkt)\yb_h^k\| + \lambda^{-1}\| \Phkt [\fb_h^k -\tPhk (\Wlinhk - \W0)]\|_2 \\
&\quad + \lambda^{-1}\|(\Phkt - \tPhkt)\tPhk (\Wlinhk - \W0)\|_2. 
\end{align*}
According to Lemma \ref{lem:regularity_neural}, we can bound the last three terms in the above inequality separately as follows
\begin{align*}
&\lambda^{-1}\|(\Phkt-  \tPhkt)\yb_h^k\|_2 \\
&\qquad\leq \lambda^{-1}K \max_{\tau\in [k-1]} |[ \varphi(z_h^\tau; \Whk) - \varphi(z_h^\tau; \W0)]\cdot [\yb_h^k]_\tau | \\
&\qquad\leq \digamma \lambda^{-1}K^{7/6}H^{4/3}  m^{-1/6} \sqrt{\log m},
\end{align*}
and similarly,
\begin{align*}
&\lambda^{-1}\| \Phkt [\fb_h^k -\tPhk (\Wlinhk - \W0)]\|_2 \leq \lambda^{-1} \digamma^2  K^{5/3} H^{4/3} m^{-1/6} \sqrt{\log m} ,\\
&\lambda^{-1}\|(\Phkt - \tPhkt)\tPhk (\Wlinhk - \W0)\|_2 \leq \lambda^{-3/2}  \digamma^2 K^{5/3} H^{4/3}m^{-1/6} \sqrt{\log m}.
\end{align*} 
Thus, we have
\begin{align*}
\text{Term(I)} &\leq \lambda^{-1}(\digamma  K^{7/6} + 2 \digamma^2 K^{5/3}) H^{4/3}m^{-1/6} \sqrt{\log m} \\
&\leq 3K^{5/3} H^{4/3}m^{-1/6} \sqrt{\log m}.
\end{align*}
where we set $\lambda = \digamma^2 (1+1/K) $, and use the fact that $\lambda \geq 1$ as $\digamma \geq 1$ as well as $\F^2/\lambda \in [1/2, 1]$ and $\F/\lambda \in [1/2, 1]$. Combining the above upper bound of Term(I) with \eqref{eq:bound_overall_1}, we obtain 
\begin{align}\label{eq:bound_overall_1_final}
|f_h^k(z) - f_{\lin, h}^k(z) | \leq 4 \digamma K^{5/3} H^{4/3}m^{-1/6} \sqrt{\log m}.
\end{align}
Next, we bound the second term in \eqref{eq:two_diff_neural}, namely $|f_{\lin,h}^k(z) - \tilde{f}(z)|$. Note that we have
\begin{align*}
&\frac{1}{\sqrt{m}} \sum_{i=1}^m \act'(\W0_i{}^\top z) z^\top \balpha_i \\
&\qquad= \frac{1}{\sqrt{2m}} \sum_{i=1}^m \frac{(v^{(0)}_i)^2 }{\sqrt{2}}\act'(\W0_i{}^\top z) z^\top \balpha_i + \frac{1}{\sqrt{2m}} \sum_{i=1}^m \frac{(v^{(0)}_i)^2}{\sqrt{2}}\act'(\W0_i{}^\top z) z^\top \balpha_i\\
&\qquad= \frac{1}{\sqrt{2m}} \sum_{i=1}^m \frac{(v^{(0)}_i)^2 }{\sqrt{2}}\act'(\W0_i{}^\top z) z^\top \balpha_i + \frac{1}{\sqrt{2m}} \sum_{i=m+1}^{2m} \frac{(v^{(0)}_{i-m})^2}{\sqrt{2}}\act'(\W0_i{}^\top z) z^\top \balpha_{i-m}\\
&\qquad= \frac{1}{\sqrt{2m}} \sum_{i=1}^m \frac{(v^{(0)}_i)^2 }{\sqrt{2}}\act'(\W0_i{}^\top z) z^\top \balpha_i + \frac{1}{\sqrt{2m}} \sum_{i=m+1}^{2m} \frac{(v^{(0)}_{i})^2}{\sqrt{2}}\act'(\W0_i{}^\top z) z^\top \balpha_i \\
&\qquad= \frac{1}{\sqrt{2m}} \sum_{i=1}^{2m } v^{(0)}_i\act'(\W0_i{}^\top z) z^\top (\tilde{W}_i - \W0_i) = \langle \varphi(z; \W0), \tilde{W} - \W0\rangle,
\end{align*}
where we define
\begin{align*}
\tilde{W}_i = \left\{\begin{matrix}
\W0_i + \frac{v_i^{(0)}}{\sqrt{2}} \balpha_i, \text{ if } 1\leq i \leq m,\\ 
\W0_i + \frac{v_i^{(0)}}{\sqrt{2}} \balpha_{i-m}, \text{ if } m+1\leq i \leq 2m.
\end{matrix}\right.
\end{align*}
Then, we can reformulate $\tilde{f}(z)$ as follows
\begin{align*}
\tilde{f}(z) =\Pi_{[0,H]}[\langle \varphi(z; \W0), \tilde{W} - \W0\rangle].
\end{align*}
Since $\|\balpha_i\|_2\leq R_Q H /\sqrt{d}$, then there is $\|\tilde{W} - \W0\|_2 \leq R_Q H / \sqrt{d}$. Equivalently, we further have
\begin{align} 
\begin{aligned}\label{eq:f_tilde_reform}
\langle \varphi(z; \W0), \tilde{W} - \W0\rangle&= \langle \varphi(z; \W0), \tLhki \tLhk(\tilde{W} - \W0)\rangle\\
&= \langle \varphi(z; \W0), \lambda \tLhki (\tilde{W} - \W0)\rangle \\
&\quad + \langle \varphi(z; \W0), \tLhki \tPhkt\tPhk(\tilde{W} - \W0)\rangle,
\end{aligned}
\end{align}
since $\Lhk = \lambda I + \tPhkt\tPhk$. Thus, by the above equivalent form of $\tilde{f}(z)$ in \eqref{eq:f_tilde_reform}, and further with the formulation of $f_{\lin, h}^k(z)$ according to \eqref{eq:f_lin_form} and \eqref{eq:W_lin_form}, we have
\begin{align*}
&|f_{\lin, h}^k(z) - \tilde{f}(z)| \\
&\qquad\leq | \langle \varphi(z; \W0), W_{\lin, h}^k - \tilde{W} \rangle| \\
&\qquad\leq \underbrace{ |\langle \varphi(z; \W0), \lambda \tLhki (\tilde{W} - \W0)\rangle|}_{\text{Term(II)}} \\
&\qquad \quad + \underbrace{| \langle \varphi(z; \W0),  (\tilde{\Lambda}_h^t)^{-1} (\tilde{\Phi}_h^k)^\top [\yb_h^k - \tPhk(\tilde{W} - \W0)] \rangle|}_{\text{Term(III)}}.
\end{align*}
The first term Term(II) can be bounded as
\begin{align*}
\text{Term(II)} &= |\langle \varphi(z; \W0), \lambda \tLhki (\tilde{W} - \W0)\rangle| \\
&\leq \lambda \|\varphi(z; \W0)\|_{\tLhki} \|\tilde{W} - \W0\|_{\tLhki} \\
&\leq \sqrt{\lambda} \|\varphi(z; \W0)\|_{\tLhki} \|\tilde{W} - \W0\|_2 \\
&\leq  \sqrt{\lambda}R_Q H / \sqrt{d} \cdot \|\varphi(z; \W0)\|_{\tLhki},
\end{align*}
where the first inequality is by $\|\tilde{W} - \W0\|_{\tLhki} = \sqrt{(\tilde{W} - \W0)^\top \tLhki (\tilde{W} - \W0) } \leq 1/\sqrt{\lambda} \|\tilde{W} - \W0\|_2$ since $\tLhki \preccurlyeq 1/\lambda \cdot I$ and the last inequality is due to $\|\tilde{W} - \W0\|_2 \leq R_Q H / \sqrt{d}$.

Next, we prove the bound of Term(III) in the following way
\begin{align*}
\text{Term(III)} &= | \langle \varphi(z; \W0),  (\tilde{\Lambda}_h^t)^{-1} (\tilde{\Phi}_h^k)^\top [\yb_h^k - \tPhk(\tilde{W} - \W0)] \rangle| \\
&\leq | \langle \varphi(z; \W0),  (\tilde{\Lambda}_h^t)^{-1} (\tilde{\Phi}_h^k)^\top [\tilde{\yb}_h^k - \tPhk(\tilde{W} - \W0)] \rangle| \\
&\quad+ | \langle \varphi(z; \W0),  (\tilde{\Lambda}_h^t)^{-1} (\tilde{\Phi}_h^k)^\top [\yb_h^k - \tilde{\yb}_h^k] \rangle|\\
&\leq \|\varphi(z; \W0)\|_{\tLhki} \cdot\|(\tilde{\Phi}_h^k)^\top [\tilde{\yb}_h^k - \tPhk(\tilde{W} - \W0)] \|_{\tLhki} \\
&\quad + \|\varphi(z; \W0)\|_{\tLhki}\cdot \| (\Phi_h^k)^\top [\yb_h^k - \tilde{\yb}_h^k] \|_{\tLhki}\\
&\leq 10 C_\act R_Q H  \sqrt{K\log(mKH)/m} \|\varphi(z; \W0)\|_{\tLhki}  \\
&\quad+ \|\varphi(z; \W0)\|_{\tLhki}\cdot \underbrace{\| (\tilde{\Phi}_h^k)^\top [\yb_h^k - \tilde{\yb}_h^k] \|_{\tLhki}}_{\text{Term(IV)}},
\end{align*}
where we define $\tilde{\yb}_h^k = [\PP_h V^k_{h+1}(s_{h+1}^1), \PP_h V^k_{h+1}(s_{h+1}^2) , \cdots, \PP_h V^k_{h+1}(s_{h+1}^{k-1})  ]^\top$. Here, the last inequality is by 
\begin{align*}
&\|(\Phi_h^k)^\top [\tilde{\yb}_h^k - \tPhk(\tilde{W} - \W0)] \|_{\tLhki} \\
& \qquad = \sqrt{[\tilde{\yb}_h^k - \tPhk(\tilde{W} - \W0)]^\top\tPhk [\lambda I + \tPhkt \tPhk]^{-1} \tPhkt[\tilde{\yb}_h^k - \tPhk(\tilde{W} - \W0)]} \\
& \qquad = \sqrt{[\tilde{\yb}_h^k - \tPhk(\tilde{W} - \W0)]^\top\tilde{\Phi}_h^k \tPhkt [\lambda I +  \tPhk \tPhkt]^{-1}   [\tilde{\yb}_h^k - \tPhk(\tilde{W} - \W0)]} \\
& \qquad \leq \sqrt{[\tilde{\yb}_h^k - \tPhk(\tilde{W} - \W0)]^\top[\lambda I +\tilde{\Phi}_h^k \tPhkt] [\lambda I +  \tPhk \tPhkt]^{-1}   [\tilde{\yb}_h^k - \tPhk(\tilde{W} - \W0)]} \\
& \qquad = \|\tilde{\yb}_h^k - \tPhk(\tilde{W} - \W0)\|_2 \leq 10 C_\act R_Q H  \sqrt{K\log(mKH)/m} ,
\end{align*}
where the second equality is by Woodbury matrix identity,  the first inequality is due to $[\lambda I +  \tPhk \tPhkt]^{-1} \succ 0$, and the second inequality is by  \eqref{eq:real_kernel_approx_neural} such that
\begin{align*}
\|\tilde{\yb}_h^k - \tPhk(\tilde{W} - \W0)\|_2 &\leq \sqrt{k-1} \|\tilde{\yb}_h^k - \tPhk(\tilde{W} - \W0)\|_\infty \\
&= \sqrt{k-1} \sup_{\tau\in [k-1]} |\PP_h V^k_{h+1}(s_h^\tau, a_h^\tau) - \tilde{f}(s_h^\tau, a_h^\tau)|\\
&\leq 10 C_\act R_Q H  \sqrt{K\log(mKH)/m}.
\end{align*}
In order to further bound Term(IV), we define a new Q-function based on $\Wlinhk$, which is
\begin{align*}
Q_{\lin,h}^k(z) &:= \Pi_{[0, H]}[ r_{\lin,h}^k(z) + f_{\lin, h}^k(z) + u_{\lin,h}^k(z)], 
\end{align*}
where $r_{\lin,h}(s,a) = u_{\lin,h}^k(z)/H$, and $u_{\lin,h}^k(z) = \min\{\beta \|\varphi(z;\W0)\|_{\tLhki}, H\}$. This Q-function can be equivalently reformulated with a normalized representation $\vartheta = \varphi/\digamma$ as follows
\begin{align}
\begin{aligned}\label{eq:Q_lin_reform}
Q_{\lin,h}^k(z) &= \min\{ \Pi_{[0, H]} [\langle\vartheta(z; \W0),\F\cdot (\Wlinhk-\W0) \rangle] \\
&\quad + (1+1/H)\cdot \min\{\beta \|\vartheta(z;\W0)\|_{(\Xi_h^k)^{-1}}\},H \}^+, 
\end{aligned}
\end{align}
where we have
\begin{align*}
&\Xi_h^k := \lambda/\F^2 \cdot I + (\Theta_h^k)^\top \Theta_h^k, \qquad \Theta_h^k := \Phi_h^k/\F.
\end{align*}
Note that $\F\|W_{\lin, h}^k-\W0\|_2\leq \F H\sqrt{K/\lambda}\leq H\sqrt{K}$ since $\lambda = \F^2(1+1/K)$. Thus, we can see that this new Q-function lies in the space $\overline{\cQ}(\bm{0}, R_K, B_K)$ as in \eqref{eq:Q_func_class}, with $R_K = H\sqrt{K}$ and $B_K = (1+1/H) \beta$ with the kernel function defined as $\tilde{\ker}_m (z,z') := \langle\vartheta(z), \vartheta(z')\rangle$. 

Now we try to bound the difference between the Q-function $Q_h^k(z)$ in the exploration algorithm and the one $Q_{\lin,h}^k(z)$, which is 
\begin{align*}
&|Q_h^k(z) - Q_{\lin,h}^k(z)| \\
&\qquad \leq |f_h^k(z) -f_{\lin, h}^k(z)| + (1+1/H)\beta \left|\|\varphi(z; \Whk)\|_{\Lhki} - \|\varphi(z; \W0)\|_{\tLhki}\right|,
\end{align*}
where the inequality is by the contraction of the operator $\min\{\cdot, H\}^+$. The upper bound of the term $|f_h^k(z) -f_{\lin, h}^k(z)|$ has already been studied in \eqref{eq:bound_overall_1_final}. Then, we focus on bounding the last term. Thus, we have
\begin{align*}
&\left|\|\varphi(z; \Whk)\|_{\Lhki} - \|\varphi(z; \W0)\|_{\tLhki}\right| \\
&  \leq \sqrt{\left|\varphi(z; \Whk)^\top\Lhki\varphi(z; \Whk)  - \varphi(z; \W0)^\top\tLhki \varphi(z; \W0)\right|}\\
&  \leq \sqrt{\left|[\varphi(z; \Whk)- \varphi(z; \W0)]^\top\Lhki\varphi(z; \Whk)\right|} + \sqrt{\left|\varphi(z; \W0)^\top(\Lhki-\tLhki)\varphi(z; \Whk)\right|} \\
&\quad  \quad + \sqrt{\left|\varphi(z; \W0)^\top\tLhki[\varphi(z; \Whk)  - \varphi(z; \W0)]\right|}.
\end{align*}
Conditioned on the event that all the inequalities in Lemma \ref{lem:regularity_neural} hold, we can bound the last three terms above as follows
\begin{align*}
&\left|[\varphi(z; \Whk)- \varphi(z; \W0)]^\top\Lhki\varphi(z; \Whk)\right| \\
&\qquad \leq \|\varphi(z; \Whk)- \varphi(z; \W0)\|_2\|\Lhki\|_2\|\varphi(z; \Whk)\|_2 \leq  \lambda^{-1} \F^2 (KH^2/m)^{1/6}\sqrt{\log m},\\
&\left|\varphi(z; \W0)^\top\tLhki[\varphi(z; \Whk)  - \varphi(z; \W0)]\right|  \leq  \lambda^{-1} \F^2 (KH^2/m)^{1/6}\sqrt{\log m}, \\
&\left|\varphi(z; \W0)^\top(\Lhki-\tLhki)\varphi(z; \Whk)\right|\\
&\qquad \leq \|\varphi(z; \W0)\|_2 \|\Lhki(\Lhk-\tLhk)\tLhki\|_2 \|\varphi(z; \Whk)\|_2 \\
&\qquad \leq \lambda^{-2} \F^2 \|\Phkt\Phk -\tPhkt\tPhk \|_{\fro} \leq \lambda^{-2} \F^2 (\|(\Phk -\tPhk)^\top\Phk \|_{\fro} + \|\tPhkt(\Phk -\tPhk) \|_{\fro})\\
&\qquad \leq \lambda^{-2} \F^4 K^{7/6} H^{1/3} m^{-1/6} \sqrt{\log m},
\end{align*}
which thus lead to
\begin{align} \label{eq:bonus_diff}
\left|\|\varphi(z; \Whk)\|_{\Lhki} - \|\varphi(z; \W0)\|_{\tLhki}\right| \leq 3K^{7/12} H^{1/6} m^{-1/12} \log^{1/4} m,
\end{align}
and thus
\begin{align*}
|Q_h^k(z) - Q_{\lin,h}^k(z)| \leq 4 \digamma K^{5/3} H^{4/3}m^{-1/6} \sqrt{\log m} + 3(1+ 1/H) \beta K^{7/12} H^{1/6} m^{-1/12} \log^{1/4} m,
\end{align*}
where we use the fact that $\lambda = \F^2 (1+ 1/K)\in [\F^2, 2\F^2]$. This further implies that we have the same bound for $|V_h^k(s) - V_{\lin,h}^k(s)|$, .i.e., 
\begin{align} 
|V_h^k(s) - V_{\lin,h}^k(s)| &\leq \max_{a\in \cA}|Q_h^k(s,a) - Q_{\lin,h}^k(s,a)| \label{eq:bound_V_diff}\\
&\leq 4 \digamma K^{5/3} H^{4/3}m^{-1/6} \sqrt{\log m} + 3(1+ 1/H) \beta K^{7/12} H^{1/6} m^{-1/12} \log^{1/4} m,\nonumber
\end{align}
where we define $V_{\lin,h}^k(s) = \max_{a\in \cA} Q_{\lin,h}^k(s,a)$.

Now, we are ready to bound Term(IV). With probability at least $1-\delta'$, we have
\begin{align*}
&\text{Term(IV)} \\
&\qquad = \left\|\sum_{\tau=1}^{k-1} [V_{h+1}^k(s_{h+1}^\tau) -  \PP_h V_{h+1}^k(z_h^\tau)] \varphi(z_h^\tau;\W0) \right\|_{\tLhki}\\
&\qquad\leq  \left\|\sum_{\tau=1}^{k-1} [V_{\lin, h+1}^k(s_{h+1}^\tau) -  \PP_h V_{\lin, h+1}^k(z_h^\tau)] \varphi(z_h^\tau;\W0) \right\|_{\tLhki} \\
&\qquad\quad +  \left\|\sum_{\tau=1}^{k-1} \{ [V_{h+1}^k(s_{h+1}^\tau) - V_{\lin, h+1}^k(s_{h+1}^\tau)] -  \PP_h [V_{h+1}^k- V_{\lin, h+1}^k(s_{h+1}^\tau)] \}\varphi(z_h^\tau;\W0) \right\|_{\tLhki} \\
&\qquad\leq  [4H^2 \Gamma(K, \lambda'; \tilde{\ker}_m) + 10H^2+4H^2\log\cN_{\infty}(\varsigma^*;R_K,  B_K) + 4H^2\log(K /\delta')]^{1/2} \\
&\qquad\quad +  8 \digamma K^{8/3} H^{4/3}m^{-1/6} \sqrt{\log m} + 12 \beta K^{19/12} H^{1/6} m^{-1/12} \log^{1/4} m.
\end{align*}
Here we set $\lambda' = \lambda/\F^2 = (1+1/K)$, $\varsigma^* = H/K$, $R_K = H\sqrt{K}$, $B_K = (1+1/H)\beta$, and $\tilde{\ker}_m(z,z') = \langle \vartheta(z), \vartheta(z') \rangle$.
Here the second inequality is by \eqref{eq:Q_lin_reform}, and also follows the similar proof of Lemma \ref{lem:approx_concentrate}. The last inequality is by \eqref{eq:bound_V_diff} and Lemma \ref{lem:regularity_neural}, which lead to
\begin{align*}
&\left\|\sum_{\tau=1}^{k-1} \{ [V_{h+1}^k(s_{h+1}^\tau) - V_{\lin, h+1}^k(s_{h+1}^\tau)] -  \PP_h [V_{h+1}^k- V_{\lin, h+1}^k(s_{h+1}^\tau)] \}\varphi(z_h^\tau;\W0) \right\|_{\tLhki} \\
&\qquad \leq \sum_{\tau=1}^{k-1} [8 \digamma K^{5/3} H^{4/3}m^{-1/6} \sqrt{\log m} + 12 \beta K^{7/12} H^{1/6} m^{-1/12} \log^{1/4} m] \|\varphi(z_h^\tau;\W0)\|_{\tLhki} \\
&\qquad \leq K \F/\sqrt{\lambda} [8 \digamma K^{5/3} H^{4/3}m^{-1/6} \sqrt{\log m} + 12 \beta K^{7/12} H^{1/6} m^{-1/12} \log^{1/4} m] \\
&\qquad \leq 8 \digamma K^{8/3} H^{4/3}m^{-1/6} \sqrt{\log m} + 12 \beta K^{19/12} H^{1/6} m^{-1/12} \log^{1/4} m,
\end{align*} 
where we use $\F^2/\lambda = 1/(1+1/K)\leq 1$ and $(1+1/H)\leq 2$ due to $H\geq 1$. Now we let $\beta$ satisfy
\begin{align*}
&\sqrt{\lambda}R_Q H / \sqrt{d} + 10 C_\act R_Q H  \sqrt{K\log(mKH)/m} + H[4 \Gamma(K, \lambda'; \tilde{\ker}_m) +4\log\cN_{\infty}(\varsigma^*;R_K,  B_K) \\
&\quad + 10 + 4\log(K /\delta')]^{1/2} + 8 \digamma K^{8/3} H^{4/3}m^{-1/6} \sqrt{\log m} + 12 \beta K^{19/12} H^{1/6} m^{-1/12} \log^{1/4} m \leq \beta.
\end{align*}
To obtain the above relation, it suffices to set 
\begin{align*}
m = \Omega(K^{19} H^{14}\log^3m)
\end{align*}
such that $m$ is sufficient large which results in
\begin{align*}
&10 C_\act R_Q H  \sqrt{K\log(mKH)/m}  + 8 \digamma K^{8/3} H^{4/3}m^{-1/6} \sqrt{\log m} \\
&\qquad + 12 \beta K^{19/12} H^{1/6} m^{-1/12} \log^{1/4} m \leq R_Q H +  \beta / 2. 
\end{align*}
Then, there is
\begin{align*}
&\sqrt{\lambda}R_Q H / \sqrt{d} +  R_Q H +  \beta / 2  \\
&\qquad+ 2H[ \Gamma(K, \lambda; \ker_m) + 5/2 +\log\cN_{\infty}(\varsigma^*;R_K,  B_K) + \log(K /\delta')]^{1/2} \leq \beta,
\end{align*}
where $\Gamma(K, \lambda; \ker_m)  =  \Gamma(K, \lambda'; \tilde{\ker}_m) $ with $\ker_m := \langle \varphi(z;\W0), \varphi(z';\W0)\rangle$. This inequality can be satisfied if we set $\beta$ as 
\begin{align*}
\beta^2 \geq H^2[8R_Q^2  (1+\sqrt{\lambda/d})^2 + 32 \Gamma(K, \lambda; \ker_m) + 80 +32 \log\cN_{\infty}(\varsigma^*;R_K,  B_K) + 32\log(K /\delta')].
\end{align*}
If the above conditions hold, we have
\begin{align*}
|f_{\lin, h}^k(z) - \tilde{f}(z)| \leq \beta  \|\varphi(z; \W0)\|_{\tLhki} \leq w_h^k + \beta (3K^{7/12} H^{1/6} m^{-1/12} \log^{1/4} m),
\end{align*}
where the inequality is due to \eqref{eq:bonus_diff}. Since $f_{\lin, h}^k(z)\in [0,H]$ and $\tilde{f}(z)\in [0,H]$, thus we have $|f_{\lin, h}^k(z) - \tilde{f}(z)| \leq H$, which further gives
\begin{align}
\begin{aligned}\label{eq:bound_overall_2_final}
|f_{\lin, h}^k(z) - \tilde{f}(z)| &\leq \min\{w_h^k, H\} + \beta (3K^{7/12} H^{1/6} m^{-1/12} \log^{1/4} m) \\
&= u_h^k + \beta (3K^{7/12} H^{1/6} m^{-1/12} \log^{1/4} m).
\end{aligned}
\end{align}
Now we combine \eqref{eq:bound_overall_1_final} and \eqref{eq:bound_overall_2_final} as well as \eqref{eq:real_kernel_approx_neural} and obtain
\begin{align*}
&|\PP_hV_{h+1}^k(z) - f_h^k(z)| \\
&\qquad\leq  |\PP_hV_{h+1}^k(z) - \tilde{f}(z)|  + |f_h^k(z)-f_{\lin, h}^k(z)|  + |f_{\lin, h}^k(z) - \tilde{f}(z)|\\
&\qquad\leq 10C_\act R_Q H \sqrt{\log(mKH)/m} + 4\F K^{5/3} H^{4/3} m^{-1/6}\sqrt{\log m} \\
&\qquad \quad +  u_h^k + \beta (3K^{7/12} H^{1/6} m^{-1/12} \log^{1/4} m)\\
&\qquad\leq  u_h^k + \beta (5K^{7/12} H^{1/6} m^{-1/12} \log^{1/4} m),
\end{align*}
with $m$ are sufficiently. We also have $\left|\|\varphi(z; \Whk)\|_{\Lhki} - \|\varphi(z; \W0)\|_{\tLhki}\right| \leq \iota$ according to \eqref{eq:bonus_diff}. The above inequalities hold with probability at least $1-2/m^2 - \delta'$ by the union bound. This completes the proof.
\end{proof}

\begin{lemma} \label{lem:bonus_explore_neural} Conditioned on the event $\cE$ defined in Lemma \ref{lem:bonus_concentrate_neural}, with probability at least $1-\delta'$, we have
\begin{align*}
&\sum_{k=1}^K V_1^*(s_1, r^k) \leq \sum_{k=1}^K V_1^k(s_1) + \beta HK \iota,\\
&\sum_{k=1}^K V_1^k(s_1)\leq \cO\left(\sqrt{H^3 K \log (1/\delta')} +\beta  \sqrt{H^2 K \cdot \Gamma(K, \lambda; \ker_m)}\right) + \beta H K \iota,
\end{align*}
where $\iota = 5K^{7/12} H^{1/6} m^{-1/12} \log^{1/4} m$.
\end{lemma}

\begin{proof} We first show the first inequality in this lemma. We prove $V_h^*(s, r^k) \leq V_h^k(s) +(H+1-h)\iota$ for all $s \in \cS, h \in [H]$ by induction.  When  $h=H+1$, we know $V_{H+1}^*(s, r^k) = 0$ and $V_{H+1}^k(s)= 0 $ such that $V_{H+1}^*(s, r^k) \leq V_{H+1}^k(s_1)$. Now we assume that $V_{h+1}^*(s, r^k) \leq V_{h+1}^k(s)+(H-h)\beta\iota$. Then, conditioned on the event $\cE$ defined in Lemma \ref{lem:bonus_concentrate}, for all $s\in \cS$, $(h,k)\in [H]\times[K]$, we further have
\begin{align}
\begin{aligned}\label{eq:Q_diff_neural}
&Q_h^*(s, a, r^k) - Q_h^k(s, a) \\
&\qquad= r_h^k(s,a) + \PP_h V_{h+1}^*(s,a, r^k)  -  \min \{ r_h^k(s,a) + f_h^k(s,a) + u_h^k(s,a), H \}^+  \\
&\qquad\leq  \max \{ \PP_h V_{h+1}^*(s,a, r^k)  - f_h^k(s,a) - u_h^k(s,a), 0 \}  \\
&\qquad\leq  \max \{ \PP_h V_{h+1}^k(s,a) + \beta(H-h)\iota - f_h^k(s,a) - u_h^k(s,a), 0 \}  \\
&\qquad\leq \beta(H+1-h)\iota,
\end{aligned}
\end{align}
where the first inequality is due to $0 \leq r_h^k(s,a) + \PP_h V_{h+1}^*(s,a, r^k) \leq H$ and $\min \{x, y\}^+ \geq \min \{x, y\}$, the second inequality is by the assumption that $V_{h+1}^*(s, r^k)\leq V_{h+1}^k(s)+(H-h)\beta\iota$, the last inequality is by Lemma \ref{lem:bonus_concentrate_neural} such that $|\PP_h V_{h+1}^k(s,a) - f_h^k(s,a) | \leq u_h^k(s,a) + \beta\iota$ holds for any $(s,a)\in \cS \times \cA$ and $(k,h)\in [K]\times[H]$. The above inequality \eqref{eq:Q_diff_neural} further leads to 
\begin{align*}
V_h^*(s, r^k) &= \max_{a\in \cA} Q_h^*(s, a, r^k) \\
&\leq \max_{a\in \cA} Q_h^k(s, a) \\
&= V_h^k(s) + \beta(H+1-h)\iota. 
\end{align*}
Therefore, we obtain that conditioned on event $\cE$, we have
\begin{align*}
\sum_{k=1}^K V_1^*(s, r^k) \leq \sum_{k=1}^K V_1^k(s) +  \beta HK\iota.
\end{align*}
Next, we prove the second inequality in this lemma. Conditioned on $\cE$ defined in Lemma \ref{lem:bonus_concentrate_neural}, we have 
\begin{align*}
V_h^k(s_h^k) &= Q_h^k(s_h^k, a_h^k) \\
&\leq \max\{0, f_h^k(s_h^k, a_h^k) + r_h^k(s_h^k, a_h^k) + u_h^k(s_h^k, a_h^k)\}\\
&\leq \PP_h V_{h+1}^k(s_h^k, a_h^k)  + u_h^k(s_h^k, a_h^k) + r_h^k(s_h^k, a_h^k) + u_h^k(s_h^k, a_h^k)\\
& \leq \zeta_h^k + V_{h+1}^k(s_{h+1}^k) + (2+1/H)\beta \|\varphi(s_h^k,a_h^k; \Whk)\|_{(\Lambda_h^k)^{-1}},
\end{align*}
where we define
\begin{align*}
\zeta_h^k := \PP_h V_{h+1}^k(s_h^k, a_h^k) - V_{h+1}^k(s_{h+1}^k).
\end{align*}
Recursively applying the above inequality gives
\begin{align*}
V_1^k(s_1) \leq  \sum_{h=1}^H \zeta_h^k  + (2+1/H)\beta \sum_{h=1}^H \|\varphi(s_h^k,a_h^k; \Whk)\|_{(\Lambda_h^k)^{-1}},
\end{align*}
where we use the fact that $V_{H+1}^k(\cdot) = 0$. Taking summation on both sides of the above inequality, we have
\begin{align*}
\sum_{k=1}^KV_1^k(s_1)  = \sum_{k=1}^K\sum_{h=1}^H \zeta_h^k  + (2+1/H)\beta \sum_{k=1}^K\sum_{h=1}^H \|\varphi(s_h^k,a_h^k; \Whk)\|_{(\Lambda_h^k)^{-1}}.
\end{align*}
By Azuma-Hoeffding inequality, with probability at least $1-\delta'$, the following inequalities hold 
\begin{align*}
&\sum_{k=1}^K \sum_{h=1}^H \zeta_h^k \leq \cO\left(\sqrt{H^3 K \log \frac{1}{\delta'}} \right).
\end{align*} 
On the other hand, by Lemma \ref{lem:direct_sum_bound}, we have
\begin{align*}
\sum_{k=1}^K\sum_{h=1}^H \|\phi(s_h^k,a_h^k)\|_{(\Lambda_h^k)^{-1}} &= \sum_{k=1}^K\sum_{h=1}^H \sqrt{\varphi(s_h^k,a_h^k; \Whk)^\top (\Lambda_h^k)^{-1}\phi(s_h^k,a_h^k; \Whk)}\\
&\leq \sum_{k=1}^K\sum_{h=1}^H \sqrt{\varphi(s_h^k,a_h^k; \W0)^\top (\tilde{\Lambda}_h^k)^{-1}\varphi(s_h^k,a_h^k; \W0)} + HK\iota\\
&\leq \sum_{h=1}^H \sqrt{K\sum_{k=1}^K\varphi(s_h^k,a_h^k; \W0)^\top (\tilde{\Lambda}_h^k)^{-1}\varphi(s_h^k,a_h^k; \W0))} + HK\iota \\
&= 2H \sqrt{K\cdot \Gamma(K, \lambda; \ker_m)} + HK\iota.
\end{align*}
where the first inequality is due to Lemma \ref{lem:bonus_concentrate_neural}, the second inequality is by Jensen's inequality. Thus, conditioned on event $\cE$, we obtain that with probability at least $1-\delta'$, there is 
\begin{align*}
\sum_{k=1}^KV_1^k(s_1) \leq \cO\left(\sqrt{H^3 K \log (1/\delta')} + \beta \sqrt{H^2 K \cdot \Gamma(K, \lambda; \ker)}\right) + \beta HK\iota,
\end{align*}
which completes the proof.
\end{proof}

\begin{lemma} \label{lem:bonus_concentrate_plan_neural} We define the event $\tilde{\cE}$ as that the following inequality holds $\forall (s,a) \in \cS \times \cA, \forall h \in [H]$,
\begin{align*}
&|\PP_h V_{h+1}(s,a) - f_h(s,a)|  \leq u_h(s,a) + \beta \iota,\\
&\left|\|\varphi(z; W_h)\|_{(\Lambda_h)^{-1}} - \|\varphi(z; \W0)\|_{(\tilde{\Lambda}_h)^{-1}}\right| \leq \iota,
\end{align*}
where $\iota = 5K^{7/12} H^{1/6} m^{-1/12} \log^{1/4} m$ and we define
\begin{align*}
&\Lambda_h = \sum_{\tau=1}^K \varphi(z_h^{\tau}; \Wh) \varphi(z_h^{\tau}; \Wh)^\top + \lambda \cdot I, \quad  \tilde{\Lambda}_h = \sum_{\tau=1}^K \varphi(z_h^{\tau}; \W0) \varphi(z_h^{\tau}; \W0)^\top + \lambda \cdot I.
\end{align*}
Setting $\beta = \tilde{B}_K$, $\tilde{R}_K = H\sqrt{K}$, $\varsigma^* = H/K$, and $\lambda=\F^2(1+1/K)$, $\varsigma^* = H/K$, if we set
\begin{align*}
\beta^2 \geq H^2[8R_Q^2  (1+\sqrt{\lambda/d})^2 + 32 \Gamma(K, \lambda; \ker_m) + 80 +32\log\cN_{\infty}(\varsigma^*;\tilde{R}_K,  \tilde{B}_K) + 32\log(K /\delta')],
\end{align*}
and also 
\begin{align*}
m =  \Omega(K^{19} H^{14}\log^3m),
\end{align*}
then we have that with probability at least $1-2/m^2 - \delta'$, the event $\tilde{\cE}$ happens, i.e.,
\begin{align*}
\Pr(\tilde{\cE}) \geq 1-2/m^2 - \delta'.
\end{align*}
\end{lemma}

\begin{proof} The proof of this lemma exactly follows our proof of Lemma \ref{lem:bonus_concentrate_neural}. There are several minor differences here. In the proof of this lemma, we set $\tilde{B}_K=\beta$ instead of $(1+1/H)\beta$ due to the structure of the planning phase. Moreover, we use $\cN_\infty(\epsilon;R_K, B_K)$ to denote covering number of the Q-function class $\overline{\cQ}(r_h, R_K, B_K)$. Since the covering numbers of $\overline{\cQ}(r_h, R_K, B_K)$ and $\overline{\cQ}(\bm{0}, R_K, B_K)$ are the same where the former one only has an extra bias $r_h$, we use the same notation $\cN_\infty(\epsilon;R_K, B_K)$ to denote their covering number. Then,  the rest of this proof can be completed by using the same argument as the proof of Lemma \ref{lem:bonus_concentrate_neural}.
\end{proof}

\begin{lemma} \label{lem:bonus_plan_neural} Conditioned on the event $\tilde{\cE}$ as defined in Lemma \ref{lem:bonus_concentrate_plan_neural}, we have
\begin{align*}
&V_h^*(s, r) \leq V_h(s) + (H+1-h) \beta \iota,  \forall s \in \cS, \forall h \in [H],\\
&V_h(s) \leq r_h(s,\pi_h(s)) + \PP_h V_{h+1}(s,\pi_h(s)) + 2 u_h(s,\pi_h(s)) + \beta \iota, \forall s \in \cS, \forall h \in [H],
\end{align*}
where $\pi_h(s) = \argmax_{a\in \cA} Q_h(s,a)$.
\end{lemma}
\begin{proof} We first prove the first inequality in this lemma by induction. For $h = H+1$, we have $V_{H+1}^*(s, r)  = V_{H+1}(s) = 0$ for any $s\in \cS$. Then, we assume that $V_{h+1}^*(s, r)  \leq V_{h+1}(s) + (H-h)\beta\iota$. Thus, conditioned on the event $\tilde{\cE}$ as defined in Lemma \ref{lem:bonus_concentrate_plan_neural}, we have
\begin{align*}
&Q_h^*(s, a, r) - Q_h(s, a) \\
&\qquad= r_h(s,a) + \PP_h V_{h+1}^*(s,a,r)  -  \min \{ r_h(s,a) + f_h(s,a) + u_h(s,a), H \}^+  \\
&\qquad\leq  \max \{ \PP_h V_{h+1}^*(s,a,r)  - f_h(s,a) - u_h(s,a), 0 \}  \\
&\qquad\leq  \max \{ \PP_h V_{h+1}(s,a) + (H-h)\beta \iota  - f_h(s,a) - u_h(s,a), 0 \}  \\
&\qquad\leq  (H+1-h)\beta \iota,
\end{align*}
where the first inequality is due to $0 \leq r_h(s,a) + \PP_h V_{h+1}^*(s,a,r) \leq H$ and $\min \{x, y\}^+ \geq \min \{x, y\}$, the second inequality is by the assumption that $V_{h+1}^*(s,a, r)\leq V_{h+1}(s,a) + (H-h)\beta \iota$, the last inequality is by Lemma \ref{lem:bonus_concentrate_plan_neural} such that $|\PP_h V_{h+1}(s,a) - f_h(s,a) | \leq u_h(s,a) + \beta \iota$ holds for any $(s,a)\in \cS \times \cA$ and $(k,h)\in [K]\times[H]$. The above inequality further leads to 
\begin{align*}
V_h^*(s, r) = \max_{a\in \cA} Q_h^*(s, a, r) \leq \max_{a\in \cA} Q_h(s, a) +  (H+1-h)\beta \iota= V_h(s) +  (H+1-h)\beta \iota. 
\end{align*}
Therefore, we have
\begin{align*}
V_h^*(s, r)  \leq  V_h(s)+  (H+1-h)\beta \iota, \forall h\in [H], \forall s\in \cS. 
\end{align*}

We further prove the second inequality in this lemma. We have
\begin{align*}
Q_h(s, a) &= \min \{ r_h(s,a) + f_h(s,a) + u_h(s,a), H \}^+\\
& \leq \min \{ r_h(s,a) + \PP_h V_{h+1}(s,a) + 2u_h(s,a) + \beta \iota, H \}^+\\
& \leq r_h(s,a) + \PP_h V_{h+1}(s,a) + 2u_h(s,a) + \beta \iota,
\end{align*}
where the first inequality is also by Lemma \ref{lem:bonus_concentrate_plan_neural} such that $|\PP_h V_{h+1}(s,a) - f_h(s,a) | \leq u_h(s,a)+ \beta \iota$, and the last inequality is because of the non-negativity of  $r_h(s,a) + \PP_h V_{h+1}(s,a) + 2u_h(s,a)+ \beta \iota$. Therefore, we have
\begin{align*}
V_h(s) &= \max_{a\in \cA}Q_h(s, a) = Q_h(s, \pi_h(s)) \leq r_h(s,\pi_h(s)) + \PP_h V_{h+1}(s,\pi_h(s)) + 2u_h(s,\pi_h(s)) + \beta \iota.
\end{align*}
This completes the proof.
\end{proof}

\begin{lemma} \label{lem:explore_plan_connect_neural} With the exploration and planning phases, conditioned on events $\cE$ and $\tilde{\cE}$, we have the following inequality 
\begin{align*}
K \cdot V_1^*(s_1, u/H) \leq  \sum_{k=1}^K V_1^*(s_1, r^k)  + 2K\beta \iota,
\end{align*}
where $\iota = 5K^{7/12} H^{1/6} m^{-1/12} \log^{1/4} m$.
\end{lemma}

\begin{proof} The bonus for the planning phase is $u_h(s,a) =  \min\{ \beta w_h(s,a), H\}$ where $w_h(s,a)=\|\varphi(s,a; W_h)\|_{\Lambda_h^{-1}}$. We also have $H\cdot r_h^k(s,a) =  u_h^k(s,a) =  \min\{ \beta w_h^k(s,a), H\}$ where $w_h^k(s,a) = \|\varphi(s,a; \Whk)\|_{(\Lambda^k_h)^{-1}}$. Conditioned on events $\cE$ and $\tilde{\cE}$, according to Lemmas \ref{lem:bonus_concentrate_neural} and \ref{lem:bonus_concentrate_plan_neural}, we have
\begin{align*}
&\left|\|\varphi(s,a; \Whk)\|_{\Lhki} - \|\varphi(s,a; \W0)\|_{\tLhki}\right| \leq \iota,\\
&\left|\|\varphi(s,a; \Wh)\|_{\Lhi} - \|\varphi(s,a; \W0)\|_{\tLhi}\right| \leq \iota,
\end{align*}
such that
\begin{align*}
&\beta w_h(s,a) \leq  \beta\|\varphi(s,a; \W0)\|_{\tLhi} + \beta\iota ,\\
&\beta\iota + \beta w_h^k(s,a) \geq   \beta\|\varphi(s,a; \W0)\|_{\tLhki}.
\end{align*}
Moreover, we know 
\begin{align*}
&\|\varphi(s,a; \W0)\|_{\tLhi} \\
&\qquad =   \sqrt{ \varphi(s,a; \W0)^\top \left[\lambda I  + \sum_{\tau=1}^K \varphi(s_h^\tau, a_h^\tau; \W0)\varphi(s_h^\tau, a_h^\tau; \W0)^\top\right]^{-1} \varphi(s,a; \W0)},
\end{align*}
and also 
\begin{align*}
&\|\varphi(s,a; \W0)\|_{\tLhki} \\
&\qquad =  \sqrt{ \varphi(s,a;\W0)^\top \left[\lambda I  + \sum_{\tau=1}^{k-1} \varphi(s_h^\tau, a_h^\tau; \W0)\varphi(s_h^\tau, a_h^\tau; \W0)^\top \right]^{-1} \varphi(s,a;\W0)}.
\end{align*}
Since $k-1 \leq K$ and $x^\top \phi(s_h^\tau, a_h^\tau)\phi(s_h^\tau, a_h^\tau)^\top x = [x^\top \phi(s_h^\tau, a_h^\tau)]^2\geq 0, \forall x$, then we know that 
\begin{align*}
\tLh &= \lambda I + \sum_{\tau=1}^K \varphi(s_h^\tau, a_h^\tau; \W0)\varphi(s_h^\tau, a_h^\tau; \W0)^\top \\
&\succcurlyeq \lambda I + \sum_{\tau=1}^{k-1} \varphi(s_h^\tau, a_h^\tau; \W0)\varphi(s_h^\tau, a_h^\tau; \W0)^\top = \tLhk.
\end{align*}
The above relation further implies that $\tilde{\Lambda}_h^{-1} \preccurlyeq (\tilde{\Lambda}_h^k)^{-1}$ such that 
\begin{align*}
\varphi(s,a; \W0)^\top \tLh^{-1} \varphi(s,a; \W0) \leq \varphi(s,a; \W0)^\top (\tLhk)^{-1} \varphi(s,a; \W0).
\end{align*}
Thus, we have
\begin{align*}
\beta w_h(s,a) \leq \beta w_h^k(s,a) + 2 \beta\iota,
\end{align*}
such that 
\begin{align*}
\min\{\beta w_h(s,a),H\} &\leq \min\{\beta w_h^k(s,a) + 2 \beta\iota,H\} \\
&\leq  \min\{\beta w_h^k(s,a),H\} +  2 \beta\iota,
\end{align*}
which further implies that
\begin{align*}
    u_h(s,a) \leq u_h^k(s,a) + 2\beta \iota = H\cdot r_h^k(s,a) + 2\beta \iota.
\end{align*}
Then, by the definition of the value function, we have
\begin{align*}
V_1^*(s_1, u/H) \leq  V_1^*(s_1, r^k) + 2 \beta \iota, 
\end{align*}
which thus gives
\begin{align*}
K \cdot V_1^*(s_1, u/H) \leq  \sum_{k=1}^K V_1^*(s_1, r^k) + 2K\beta \iota.
\end{align*}
This completes the proof.
\end{proof}

\subsection{Proof of Theorem \ref{thm:main_neural_single}} \label{sec:proof_main_neural_single}
\begin{proof} Conditioned on the event $\cE$ in Lemma \ref{lem:bonus_concentrate_neural} and the event $\tilde{\cE}$ in Lemma \ref{lem:bonus_concentrate_plan_neural}, we have
\begin{align}\label{eq:proof_start_neural}
V_1^*(s_1, r) - V_1^\pi(s_1, r) \leq V_1(s_1) - V_1^\pi(s_1, r) + H\beta\iota,
\end{align}
where the inequality is by Lemma \ref{lem:bonus_plan_neural}. Further by this lemma, we have
\begin{align*}
V_h(s) - V_h^\pi(s, r) &\leq   r_h(s,\pi_h(s)) + \PP_h V_{h+1}(s,\pi_h(s)) + 2u_h(s,\pi_h(s)) - Q_h^\pi(s, \pi_h(s), r) +\beta\iota\\
&= r_h(s,\pi_h(s)) + \PP_h V_{h+1}(s,\pi_h(s)) + 2u_h(s,\pi_h(s)) - r_h(s,\pi_h(s)) \\
&\quad - \PP_h V_{h+1}^\pi(s,\pi_h(s), r) +\beta\iota\\
&= \PP_h V_{h+1} (s,\pi_h(s)) - \PP_h V_{h+1}^\pi(s,\pi_h(s), r) + 2u_h(s,\pi_h(s)) +\beta\iota.
\end{align*}
Recursively applying the above inequality and making use of $V_{H+1}^\pi(s,r) = V_{H+1} (s)  = 0$ gives
\begin{align*}
V_1(s_1) - V_1^\pi(s_1,r) &\leq  \EE_{\forall h\in [H]: ~s_{h+1}\sim\PP_h(\cdot|s_h, \pi_h(s_h))}\left[\sum_{h=1}^H 2u_h(s_h,\pi_h(s_h))\Bigg| s_1 \right] + H\beta\iota\\
&=2H \cdot V_1^\pi(s_1, u/H) + H\beta\iota. 
\end{align*}
Combining with \eqref{eq:proof_start_neural} gives
\begin{align*}
V_1^*(s_1, r) - V_1^\pi(s_1, r) &\leq 2H \cdot V_1^\pi(s_1, u/H) + 2H\beta\iota \leq \frac{2H}{K} \sum_{k=1}^K V_1^*(s_1, r^k) + 4H\beta\iota\\
&\leq \frac{2H}{K} \cO\left(\sqrt{H^3 K \log (1/\delta')} + \beta\sqrt{H^2 K \cdot \Gamma(K, \lambda; \ker_m)}\right) + H\beta\iota (H + 4)\\
&\leq \cO\left([\sqrt{H^5 \log (1/\delta')} + \beta\sqrt{H^4 \cdot \Gamma(K, \lambda; \ker_m)}] /\sqrt{K} + H^2 \beta \iota \right),
\end{align*}
where the second inequality is due to Lemma \ref{lem:explore_plan_connect_neural} and the third inequality is by Lemma \ref{lem:bonus_explore_neural}. 

By the union bound, we have $P(\cE \wedge \tilde{\cE}) \geq 1-2\delta'-4/m^2$ . Therefore, by setting $\delta' = 1 / (4K^2H^2)$, we obtain that with probability at least $1-1 / (2K^2H^2)-4/m^2$
\begin{align*}
V_1^*(s_1, r) - V_1^\pi(s_1, r) &\leq  \cO\left([\sqrt{H^5 \log (1/\delta')} + \beta\sqrt{H^4 \cdot \Gamma(K, \lambda; \ker_m)}] /\sqrt{K} + H^2 \beta \iota\right)\\
&\leq  \cO\left(\beta\sqrt{H^4 [ \Gamma(K, \lambda; \ker_m)+\log (KH)]} /\sqrt{K} + H^2 \beta \iota\right),
\end{align*}
where the last inequality is due to $\beta \geq H$. Note that $\cE \wedge \tilde{\cE} $ happens when the following two conditions are satisfied, i.e., 
\begin{align*}
&\beta^2 \geq H^2[8R_Q^2  (1+\sqrt{\lambda/d})^2 + 32 \Gamma(K, \lambda; \ker_m) + 80+32\log\cN_{\infty}(\varsigma^*;R_K,  B_K) + 32\log(K /\delta')],\\
&\beta^2 \geq H^2[8R_Q^2  (1+\sqrt{\lambda/d})^2 + 32 \Gamma(K, \lambda; \ker_m) + 80+32\log\cN_{\infty}(\varsigma^*;\tilde{R}_K,  \tilde{B}_K) + 32\log(K /\delta')],
\end{align*}
where $\beta = \tilde{B}_K$,$(1+1/H)\beta = B_K$, $\lambda = \F(1+1/K)$, $\tilde{R}_K = R_K=H\sqrt{K}$, and $\varsigma^* = H/K$. The above inequalities hold if we further let $\beta$ satisfy
\begin{align*}
\beta^2 \geq H^2[8R_Q^2  (1+\sqrt{\lambda/d})^2 + 32 \Gamma(K, \lambda; \ker_m) + 80 +32 \log\cN_{\infty}(\varsigma^*;R_K,  2\beta) + 96\log(2KH)],
\end{align*}
since $2\beta\geq (1+1/H)\beta\geq \beta$ such that $\cN_{\infty}(\varsigma^*;R_K,  2\beta)\geq \cN_{\infty}(\varsigma^*;R_K,  B_K) \geq \cN_{\infty}(\varsigma^*;\tilde{R}_K,  \tilde{B}_K)$. This completes the proof.
\end{proof}


\section{Proofs for Markov Game with Kernel Function Approximation}

\subsection{Lemmas}

\begin{lemma} \label{lem:bonus_concentrate_game} We define the event $\cE$ as that the following inequality holds $\forall (s,a,b) \in \cS \times \cA \times \cB, \forall (h,k) \in [H] \times [K]$,
\begin{align*}
|\PP_h V_{h+1}^k(s,a,b) - f_h^k(s,a,b)|  \leq u_h^k(s,a,b),
\end{align*}
where $u_h^k(s,a,b) = \min\{w_h^k(s,a,b), H\}$, $w_h^k(s,a,b) = \beta \lambda^{-1/2} [ \ker(z,z)- \psi_h^k(s,a,b)^\top (\lambda I + \cK_h^k )^{-1}  \psi_h^k(s,a,b)]^{1/2}$ with $z = (s,a,b)$, and $f_h^k(z) =  \Pi_{[0,H]}[\psi_h^k(z)^\top (\lambda \cdot I +\cK_h^k )^{-1} \yb_h^k]$ with
\begin{align*}
 &\psi_h^k (z) = \Phi_h^k  \phi(z)= [\ker(z, z_h^1), \cdots, \ker(z, z_h^{k-1})]^\top, \\
 &\Phi_h^k = [\phi(z_h^1), \phi(z_h^2), \cdots, \phi(z_h^{k-1})]^\top, \\
&\yb_h^k = [V_{h+1}^k(s_{h+1}^1), V_{h+1}^k(s_{h+1}^2) , \cdots, V_{h+1}^k(s_{h+1}^{k-1})  ]^\top,\\
 & \cK_h^k = \Phi_h^k (\Phi_h^k)^\top = \begin{bmatrix}
\ker(z_h^1, z_h^1) & \ldots & \ker(z_h^1, z_h^{k-1}) \\ 
\vdots &\ddots   &\vdots \\ 
\ker( z_h^{k-1} , z_h^1) &\ldots  & \ker( z_h^{k-1}, z_h^{k-1})
\end{bmatrix}, 
\end{align*}
Thus, setting $\beta = B_K/(1+1/H)$, if $B_K$ satisfies
\begin{align*}
16H^2\big[ R^2_Q  + 2 \Gamma(K, \lambda; \ker) + 5+\log\cN_{\infty}(\varsigma^*;R_K,  B_K) + 2\log(K /\delta')  \big] \leq  B^2_K, \forall h\in [H],
\end{align*}
then we have that with probability at least $1-\delta'$, the event $\cE$ happens, i.e.,
\begin{align*}
\Pr(\cE) \geq 1-\delta'.
\end{align*}
\end{lemma}

\begin{proof} According to the exploration algorithm for the game, we can see that by letting $\ba=(a,b)$ be an action in the space $\cA\times \cB$, Algorithm \ref{alg:exploration_phase_game} reduces to Algorithm \ref{alg:exploration_phase_single} with the action space $\cA \times \cB$ and state space $\cS$. Now, we also have a transition in the form of $\PP_h(s|\ba)$ and a product policy $(\pi_h^k\otimes\nu_h^k)(s)$ such that $\ba \sim (\pi_h^k\otimes\nu_h^k)(s)$ at state $s\in \cS$ for all $(h,k)\in [H]\times[K]$. Similarly, we have $Q_h^k(s,a,b) = Q_h^k(s,\ba)$ and $V_h^k(s,a,b) = V_h^k(s,\ba)$ as well as $u_h^k(s,a,b) = u_h^k(s,\ba)$ and $u_h^k(s,a,b) = u_h^k(s,\ba)$ and $r_h^k(s,a,b) = r_h^k(s,\ba)$. Thus, we can simply apply the proof of Lemma \ref{lem:bonus_concentrate} and obtain the proof for this lemma. This completes the proof.
\end{proof}

\begin{lemma} \label{lem:bonus_explore_game} Conditioned on the event $\cE$ defined in Lemma \ref{lem:bonus_concentrate_game}, with probability at least $1-\delta'$, we have
\begin{align*}
\sum_{k=1}^K V_1^*(s_1, r^k) \leq \sum_{k=1}^K V_1^k(s_1) \leq \cO\left(\sqrt{H^3 K \log (1/\delta')} +\beta  \sqrt{H^2 K \cdot \Gamma(K, \lambda; \ker)}\right).
\end{align*}
\end{lemma}

\begin{proof} By the reduction of Algorithm \ref{alg:exploration_phase_game} to Algorithm \ref{alg:exploration_phase_single}, we can apply the same proof as the one for Lemma \ref{lem:bonus_explore}, which completes the proof.
\end{proof}

\begin{lemma} \label{lem:bonus_concentrate_plan_game} We define the event $\tilde{\cE}$ as that the following inequality holds $\forall (s,a,b) \in \cS \times \cA \times \cB, \forall h \in [H]$,
\begin{align}
&|\PP_h \overline{V}_{h+1}(s,a,b) - \overline{f}_h(s,a,b)|  \leq u_h(s,a,b), \label{eq:bonus_concentrate_game1}\\
&|\PP_h \underline{V}_{h+1}(s,a,b) - \underline{f}_h(s,a,b)|  \leq u_h(s,a,b), \label{eq:bonus_concentrate_game2}
\end{align}
where $u_h(s,a,b) = \overline{u}_h(s,a,b) = \underline{u}_h(s,a,b) = \min\{w_h(s,a,b), H\}$, $w_h(s,a,b) = \beta \lambda^{-1/2} [ \ker(z,z)- \psi_h(s,a,b)^\top (\lambda I + \cK_h )^{-1}  \psi_h(s,a,b)]^{1/2}$ with $z=(s,a,b)$,  $\cK_h = \Phi_h \Phi_h^\top$, and $\psi_h (s,a,b) = \Phi_h  \phi(s,a,b)$ with $\Phi_h = [\phi(z_h^1), \phi(z_h^2), \cdots, \phi(z_h^K)]^\top$. Moreover, we have
\begin{align*}
&\overline{f}_h(s,a,b) =\Pi_{[0, H]} [\psi_h(s,a,b)^\top (\lambda \cdot I +\cK_h )^{-1} \overline{\yb}_h],\\
&\underline{f}_h(s,a,b) = \Pi_{[0, H]} [\psi_h(s,a,b)^\top (\lambda \cdot I +\cK_h )^{-1} \underline{\yb}_h],
\end{align*}
where $\overline{\yb}_h := [\overline{V}_{h+1}(s_{h+1}^1), \cdots, \overline{V}_{h+1}(s_{h+1}^K)  ]^\top$ and $\underline{\yb}_h := [\underline{V}_{h+1}(s_{h+1}^1), \cdots, \underline{V}_{h+1}(s_{h+1}^K)  ]^\top$. 

Thus, setting $\beta = \tilde{B}_K$, if $\tilde{B}_K$ satisfies
\begin{align*}
4H^2\big[ R^2_Q  + 2 \Gamma(K, \lambda; \ker) + 5+\log\cN_{\infty}(\varsigma^*; \tilde{R}_K,  \tilde{B}_K) + 2\log(2K /\delta')  \big] \leq  \tilde{B}_K^2, \forall h\in [H],
\end{align*}
then we have that with probability at least $1-\delta'$, the event $\cE$ happens, i.e.,
\begin{align*}
\Pr(\tilde{\cE}) \geq 1-\delta'.
\end{align*}
\end{lemma}

\begin{proof} According to the construction of $u_h$ and $\overline{f}_h$, the proof for the the first inequality in this lemma is nearly the same as the proof of Lemma \ref{lem:bonus_concentrate_plan} but one difference for computing the covering number of the value function space. Specifically, we have the function class for $\overline{V}_h$ which is 
\begin{align*}
\overline{\cV}(r_h, \tilde{R}_K, \tilde{B}_K)= \{ V: V(\cdot) = \max_{a\sim\pi'} \min_{b\sim\nu'} \EE_{\pi',\nu'}Q(\cdot, a, b) \text{ with } Q\in \overline{\cQ}(r_h,\tilde{R}_K, \tilde{B}_K) \}.
\end{align*}
By Lemma \ref{lem:self_normalize_uniform} with $\delta'/2$, we have
\begin{align*}
&\left \|\sum_{\tau=1}^K \phi(s_h^\tau, a_h^\tau, b_h^\tau) [\overline{V}_{h+1}(s_{h+1}^\tau) - \PP_h \overline{V}_{h+1}(s_h^\tau, a_h^\tau, b_h^\tau)] \right\|_{(\Lambda_h)^{-1}}^2 \\
&\qquad \leq \sup_{V\in \overline{\cV}(r_h, \tilde{R}_K, \tilde{B}_K)} \left \|\sum_{\tau=1}^K \phi(s_h^\tau, a_h^\tau, b_h^\tau) [\overline{V}(s_{h+1}^\tau) - \PP_h \overline{V}(s_h^\tau, a_h^\tau, b_h^\tau)] \right\|_{(\Lambda_h)^{-1}}^2  \\
&\qquad \leq  2H^2 \log\det(I+\cK/\lambda) + 2H^2K(\lambda-1)+4H^2\log(\cN^{\overline{\cV}}_{\dist}(\epsilon; \tilde{R}_K, \tilde{B}_K)/\delta')+ 8K^2\epsilon^2/\lambda\\
&\qquad \leq  4H^2 \Gamma(K, \lambda; \ker) + 10H^2+4H^2\log\cN_{\infty}(\varsigma^*;\tilde{R}_K, \tilde{B}_K) + 4H^2\log(2 /\delta'),
\end{align*}
where the last inequality is by setting $\lambda = 1+1/K$ and $\epsilon = \varsigma^* =  H/K$. Here $\cN^{\overline{\cV}}_{\dist}$ is the covering number of the function space $\overline{\cV}$ w.r.t. the distance $\dist(V_1, V_2) =\sup_s |V_1(s)-V_2(s)|$, and $\cN_{\infty}$ is the covering number for the function space $\overline{\cQ}$ w.r.t. the infinity norm. In the last inequality, we also use
\begin{align*}
\cN^{\overline{\cV}}_{\dist}(\varsigma^*; \tilde{R}_K, \tilde{B}_K) \leq \cN_{\infty}(\varsigma^*; \tilde{R}_K, \tilde{B}_K),
\end{align*}
which is in particular due to
\begin{align}
\begin{aligned} \label{eq:covering_minmax}
\dist(V_1, V_2) &=\sup_{s\in\cS} |V_1(s)-V_2(s)| \\
&= \sup_{s\in\cS} |\max_{\pi'} \min_{\nu'} \EE_{a\sim \pi', b\sim \nu'} [Q_1(s,a,b)]-\max_{\pi''} \min_{\nu''} \EE_{a\sim \pi'', b\sim \nu''} [Q_2(s,a,b)]|\\
&\leq \sup_{s\in\cS}\sup_{a\in \cA}\sup_{b\in \cB} |Q_1(s,a,b)-Q_2(s,a,b)|\\
&= \|Q_1(\cdot,\cdot,\cdot)-Q_2(\cdot,\cdot,\cdot)\|_\infty,
\end{aligned}
\end{align}
where we use the fact that max-min operator is non-expansive. Thus, we have that with probability at least $1-\delta'/2$, the following inequality holds for all $k\in [K]$
\begin{align*}
&\left \|\sum_{\tau=1}^K \phi(s_h^\tau, a_h^\tau, b_h^\tau) [\overline{V}_{h+1}(s_{h+1}^\tau) - \PP_h \overline{V}_{h+1}(s_h^\tau, a_h^\tau, b_h^\tau)] \right\|_{\Lambda_h^{-1}} \\
&\qquad  \leq  [4H^2 \Gamma(K, \lambda; \ker) + 10H^2+4H^2\log\cN_{\infty}(\varsigma^*;\tilde{R}_K,  \tilde{B}_K) + 4H^2\log(2K /\delta')]^{1/2}. 
\end{align*}
Then, the rest of the proof for \eqref{eq:bonus_concentrate_game1} follows the proof of Lemma \ref{lem:bonus_concentrate_plan}. 

Next, we give the proof of \eqref{eq:bonus_concentrate_game2}. We define another function class for $\underline{V}_h$ as
\begin{align*}
\underline{\cV}(r_h, \tilde{R}_K, \tilde{B}_K)= \{ V: V(\cdot) = \max_{a\sim\pi'} \min_{b\sim\nu'} \EE_{\pi',\nu'}Q(\cdot, a, b) \text{ with } Q\in \underline{\cQ}(r_h,\tilde{R}_K, \tilde{B}_K) \}.
\end{align*}
Note that as we can show in  the covering number for the function spaces  $\underline{\cQ}$ and $\overline{\cQ}$ have the same covering number upper bound. Therefore, we use the same notation $\cN_{\infty}$ for their upper bound. Thus, by the similar argument as \eqref{eq:covering_minmax}, we have that with probability at least $1-\delta'/2$, the following inequality holds for all $k\in [K]$
\begin{align*}
&\left \|\sum_{\tau=1}^K \phi(s_h^\tau, a_h^\tau, b_h^\tau) [\underline{V}_{h+1}(s_{h+1}^\tau) - \PP_h \underline{V}_{h+1}(s_h^\tau, a_h^\tau, b_h^\tau)] \right\|_{\Lambda_h^{-1}} \\
&\qquad  \leq  [4H^2 \Gamma(K, \lambda; \ker) + 10H^2+4H^2\log\cN_{\infty}(\varsigma^*;\tilde{R}_K,  \tilde{B}_K) + 4H^2\log(2K /\delta')]^{1/2},
\end{align*}
where we use the fact that\begin{align*}
\cN^{\underline{\cV}}_{\dist}(\varsigma^*; \tilde{R}_K, \tilde{B}_K) \leq \cN_{\infty}(\varsigma^*; \tilde{R}_K, \tilde{B}_K).
\end{align*}
The rest of the proof are exactly the same as the proof of Lemma \ref{lem:bonus_concentrate_plan}. 

In this lemma, we let
\begin{align*}
 H\big[ 2\lambda R^2_Q  + 8 \Gamma(K, \lambda; \ker) + 20+4\log\cN_{\infty}(\varsigma^*;\tilde{R}_K, \tilde{B}_K) + 8\log(2K /\delta')  \big]^{1/2} \leq \beta = \tilde{B}_K,
\end{align*}
which can be further guaranteed by 
\begin{align*}
4H^2\big[ R^2_Q  + 2 \Gamma(K, \lambda; \ker) + 5+\log\cN_{\infty}(\varsigma^*; \tilde{R}_K,  \tilde{B}_K) + 2\log(2K /\delta')  \big] \leq  \tilde{B}^2_K
\end{align*}
as $(1+1/H) \leq 2$ and $\lambda = 1+1/K \leq 2$. This completes the proof.
\end{proof}

\begin{lemma} \label{lem:bonus_plan_game} Conditioned on the event $\tilde{\cE}$ as defined in Lemma \ref{lem:bonus_concentrate_plan_game}, we have
\begin{align}
&V_h^\dagger(s, r) \leq \overline{V}_h(s) \leq \EE_{a\sim \pi_h, b\sim \bre(\pi)_h} [(\PP_h\overline{V}_{h+1}  + r_h + 2u_h)(s,a,b)], \forall s \in \cS, \forall h \in [H], \label{eq:optimism_game_1} \\
&V_h^\dagger(s, r) \geq \underline{V}_h(s)\geq \EE_{a\sim \bre(\nu)_h, b\sim \nu_h} [(\PP_h\underline{V}_{h+1}  - r_h - 2u_h)(s,a,b)], \forall s \in \cS, \forall h \in [H]. \label{eq:optimism_game_2} 
\end{align}
\end{lemma}
\begin{proof} For the first inequality of \eqref{eq:optimism_game_1}, we can prove it by induction. We first prove the first inequality in this lemma. We prove it by induction. For $h = H+1$, by the planning algorithm, we have $V_{H+1}^\dagger(s, r)  = V_{H+1}(s) = 0$ for any $s\in \cS$. Then, we assume that $V_{h+1}^\dagger(s, r)  \leq \overline{V}_{h+1}(s)$. Thus, conditioned on the event $\tilde{\cE}$ as defined in Lemma \ref{lem:bonus_concentrate_plan_game}, we have
\begin{align*}
&Q_h^\dagger(s, a, b, r) - \overline{Q}_h(s, a, b) \\
&\qquad= r_h(s,a,b) + \PP_h V_{h+1}^\dagger(s,a,b, r)  -  \min \{ r_h(s,a,b) + \overline{f}_h(s,a,b) + u_h(s,a,b), H \}^+  \\
&\qquad\leq  \max \{ \PP_h V_{h+1}^\dagger(s,a,b,r)  - \overline{f}_h(s,a,b) - u_h(s,a,b), 0 \}  \\
&\qquad\leq  \max \{ \PP_h \overline{V}_{h+1}(s,a,b)  - \overline{f}_h(s,a,b) - u_h(s,a,b), 0 \}  \leq 0,
\end{align*}
where the first inequality is due to $0 \leq r_h(s,a,b) + \PP_h V_{h+1}^\dagger(s,a,b,r) \leq H$ and $\min \{x, y\}^+ \geq \min \{x, y\}$, the second inequality is by the assumption that $V_{h+1}^\dagger(s,a,b, r)\leq \overline{V}_{h+1}(s,a,b)$, the last inequality is by Lemma \ref{lem:bonus_concentrate_plan_game} such that $|\PP_h \overline{V}_{h+1}(s,a,b) - \overline{f}_h(s,a,b) | \leq u_h(s,a,b)$ holds for any $(s,a,b)\in \cS \times \cA\times \cB$ and $(k,h)\in [K]\times[H]$. Thus, the above inequality leads to 
\begin{align*}
V_h^\dagger(s, r) &= \max_{\pi_h'}\min_{\nu_h'}\EE_{a\sim \pi_h', b\sim \nu_h'} [Q_h^\dagger(s, a, b, r)] \\
&\leq \max_{\pi_h'}\min_{\nu_h'} \EE_{a\sim \pi_h', b\sim \nu_h'} [\overline{Q}_h(s, a, b)] = \overline{V}_h(s), 
\end{align*}
which eventually gives 
\begin{align*}
V_h^*(s, r)  \leq  V_h(s), \forall h\in [H], \forall s\in \cS. 
\end{align*}
To prove the second inequality of \eqref{eq:optimism_game_1}, we have
\begin{align*}
\overline{V}_h(s) &=  \min_{\nu'} \EE_{a\sim \pi_h, b\sim \nu'} \overline{Q}_h(s,a, b)\\
&\leq  \EE_{a\sim \pi_h, b\sim \bre(\pi)_h} \overline{Q}_h(s,a, b)\\
&=  \EE_{a\sim \pi_h, b\sim \bre(\pi)_h} \min\{(\overline{f}_h + r_h + u_h)(s,a,b), H\}^+ \\
&\leq  \EE_{a\sim \pi_h, b\sim \bre(\pi)_h} \min\{(\PP_h \overline{V}_{h+1}  + r_h + 2u_h)(s,a,b), H\}^+ \\
&\leq  \EE_{a\sim \pi_h, b\sim \bre(\pi)_h} [(\PP_h\overline{V}_{h+1}  + r_h + 2u_h)(s,a,b)],
\end{align*}
where the first and the second equality is by the iterations in Algorithm \ref{alg:plan_phase_game}, the second inequality is by Lemma \ref{lem:bonus_concentrate_plan_game}, and the last inequality is due to the non-negativity of $(\PP_h\overline{V}_{h+1}  + r_h + 2u_h)(s,a,b)$.
	
For the inequalities in \eqref{eq:optimism_game_2}, one can similarly adopt the argument above to give the proof. From the perspective of Player 2, this player is trying to find a policy to maximize the cumulative rewards w.r.t. a reward function $\{-r_h(s,a,b)\}_{h\in [H]}$. Thus, the proof of \eqref{eq:optimism_game_2} follows the proof of \eqref{eq:optimism_game_1}. This completes the proof.
\end{proof}

\begin{lemma} \label{lem:explore_plan_connect_game} With the exploration and planning phases, we have the following inequalities
\begin{align*}
&K \cdot V_1^{\pi, \bre(\pi)}(s_1, u/H) \leq  \sum_{k=1}^K V_1^*(s_1, r^k), \quad K \cdot V_1^{\bre(\nu) , \nu }(s_1, u/H) \leq  \sum_{k=1}^K V_1^*(s_1, r^k).
\end{align*}

\end{lemma}

\begin{proof} First, we have $K \cdot V_1^{\pi, \bre(\pi)}(s_1, u/H) \leq K \cdot V_1^*(s_1, u/H)$,  as well as $K \cdot V_1^{\bre(\nu) , \nu }(s_1, u/H) \leq K \cdot V_1^*(s_1, u/H)$ due to the definition of $V_1^*(\cdot, u/H)$. Thus, to prove this lemma, we only need to show
\begin{align*}
K \cdot V_1^*(s_1, u/H) \leq  \sum_{k=1}^K V_1^*(s_1, r^k).
\end{align*}
Since the constructions of $u_h$ and $r_h^k$ are the same as the ones for the single-agent case, similar to the proof of Lemma \ref{lem:explore_plan_connect}, we have
\begin{align*}
u_h(s,a)/H \leq r_h^k(s,a),
\end{align*}
such that
\begin{align*}
V_1^*(s_1, u/H) \leq  V_1^*(s_1, r^k), 
\end{align*}
and thus
\begin{align*}
K \cdot V_1^*(s_1, u/H) \leq  \sum_{k=1}^K V_1^*(s_1, r^k).
\end{align*}
Therefore, we eventually obtain 
\begin{align*}
&K \cdot V_1^{\pi, \bre(\pi)}(s_1, u/H) \leq K \cdot V_1^*(s_1, u/H) \leq  \sum_{k=1}^K V_1^*(s_1, r^k),\\
&K \cdot V_1^{\bre(\nu) , \nu }(s_1, u/H) \leq K \cdot V_1^*(s_1, u/H) \leq  \sum_{k=1}^K V_1^*(s_1, r^k).
\end{align*}
This completes the proof.
\end{proof}

\subsection{Proof of Theorem \ref{thm:main_kernel_game}}\label{sec:proof_main_kernel_game}
\begin{proof} Conditioned on the event $\cE$ defined in Lemma \ref{lem:bonus_concentrate_game} and the event $\tilde{\cE}$ defined in Lemma \ref{lem:bonus_concentrate_plan_game}, we have
\begin{align}\label{eq:proof_start_kernel_game}
V_1^\dagger(s_1, r) - V_1^{\pi, \bre(\pi)}(s_1, r) \leq \overline{V}_1(s_1) - V_1^{\pi, \bre(\pi)}(s_1, r),
\end{align}
where the inequality is by Lemma \ref{lem:bonus_plan_game}. Further by this lemma, we have
\begin{align*}
&\overline{V}_h(s_h) - V_h^{\pi, \bre(\pi)}(s_h, r) \\
&~~\leq  \EE_{a_h\sim \pi_h, b_h\sim \bre(\pi)_h} [(\PP_h\overline{V}_{h+1}  + r_h + 2u_h)(s_h,a_h,b_h)]- V_h^{\pi, \bre(\pi)}(s_h, r) \\
&~~= \EE_{a_h\sim \pi_h, b_h\sim \bre(\pi)_h} [(r_h+ \PP_h \overline{V}_{h+1}+ 2u_h)(s_h,a_h, b_h) - r_h(s_h,a_h, b_h) - \PP_h V_{h+1}^{\pi, \bre(\pi)}(s_h,a_h,b_h, r) ]\\
&~~= \EE_{a_h\sim \pi_h, b_h\sim \bre(\pi)_h} [ \PP_h \overline{V}_{h+1}(s_h,a_h, b_h) - \PP_h V_{h+1}^{\pi, \bre(\pi)}(s_h,a_h,b_h, r) + 2u_h(s_h,a_h,b_h) ]\\
&~~= \EE_{a_h\sim \pi_h, b_h\sim \bre(\pi)_h, s_{h+1}\sim\PP_h} [ \overline{V}_{h+1}(s_{h+1}) -  V_{h+1}^{\pi, \bre(\pi)}(s_{h+1}, r) + 2u_h(s_h,a_h,b_h) ].
\end{align*}
Recursively applying the above inequality and making use of $\overline{V}_{H+1}(s) = V_{H+1}^{\pi, \bre(\pi)} (s, r)= 0$ yield
\begin{align*}
&\overline{V}_1(s_1) - V_1^{\pi, \bre(\pi)}(s_1,r) \\
&\qquad \leq  \EE_{\forall h\in [H]: ~a_h\sim \pi_h, b_h\sim \bre(\pi)_h, s_{h+1}\sim\PP_h}\left[\sum_{h=1}^H 2u_h(s_h, a_h, b_h)\Bigg| s_1 \right]\\
&\qquad=2H \cdot V_1^{\pi, \bre(\pi)}(s_1, u/H). 
\end{align*}
Combining this inequality with \eqref{eq:proof_start_kernel_game} gives
\begin{align*}
V_1^\dagger(s_1, r) - V_1^{\pi, \bre(\pi)}(s_1, r) &\leq 2H \cdot V_1^{\pi, \bre(\pi)}(s_1, u/H) \leq \frac{2H}{K} \sum_{k=1}^K V_1^*(s_1, r^k) \\
&\leq \frac{2H}{K} \cO\left(\sqrt{H^3 K \log (1/\delta')} + \beta\sqrt{H^2 K \cdot \Gamma(K, \lambda; \ker)}\right)\\
&\leq \cO\left([\sqrt{H^5 \log (1/\delta')} + \beta\sqrt{H^4 \cdot \Gamma(K, \lambda; \ker)}] /\sqrt{K} \right),
\end{align*}
where the second inequality is due to Lemma \ref{lem:explore_plan_connect_game} and the third inequality is by Lemma \ref{lem:bonus_explore_game}. 

Next, we prove the upper bound of the term $V_1^{\bre(\nu), \nu}(s_1, r)  - V_1^\dagger(s_1, r)$. Conditioned on the event $\cE$ defined in Lemma \ref{lem:bonus_concentrate_game} and the event $\tilde{\cE}$ defined in Lemma \ref{lem:bonus_concentrate_plan_game}, we have
\begin{align}\label{eq:proof_start_kernel_game_2}
V_1^{\bre(\nu), \nu}(s_1, r)  - V_1^\dagger(s_1, r) \leq V_1^{\bre(\nu), \nu}(s_1, r)  - \underline{V}_1(s_1, r),
\end{align}
where the inequality is by Lemma \ref{lem:bonus_plan_game}. Further by Lemma \ref{lem:bonus_plan_game}, we have
\begin{align*}
&V_h^{\bre(\nu), \nu}(s_h, r)  - \underline{V}_h(s_h) \\
&\qquad\leq  V_h^{\bre(\nu), \nu}(s_h, r) - \EE_{a\sim \bre(\nu)_h, b\sim \nu_h} [(\PP_h\underline{V}_{h+1}  - r_h - 2u_h)(s_h,a_h,b_h)] \\
&\qquad= \EE_{a_h\sim \bre(\nu)_h, b_h\sim \nu_h} [\PP_h V_{h+1}^{\bre(\nu), \nu}(s_h,a_h,b_h, r) - \PP_h \underline{V}_{h+1}(s_h,a_h, b_h)  + 2u_h(s_h,a_h,b_h) ]\\
&\qquad= \EE_{a_h\sim \bre(\nu)_h, b_h\sim \nu_h, s_{h+1}\sim\PP_h} [ V_{h+1}^{\bre(\nu), \nu}(s_{h+1}, r) - \PP_h \underline{V}_{h+1}(s_{h+1})  + 2u_h(s_h,a_h,b_h) ].
\end{align*}
Recursively applying the above inequality yields
\begin{align*}
V_1^{\bre(\nu), \nu}(s_1, r)  - \underline{V}_1(s_h, r)&\leq 2H \cdot V_1^{\bre(\nu), \nu}(s_1, u/H). 
\end{align*}
Combining this inequality with \eqref{eq:proof_start_kernel_game_2} gives
\begin{align*}
V_1^{\bre(\nu), \nu}(s_1, r)  - V_1^\dagger(s_1, r)  &\leq 2H \cdot V_1^{\bre(\nu), \nu}(s_1, u/H) \leq \frac{2H}{K} \sum_{k=1}^K V_1^*(s_1, r^k) \\
&\leq \frac{2H}{K} \cO\left(\sqrt{H^3 K \log (1/\delta')} + \beta\sqrt{H^2 K \cdot \Gamma(K, \lambda; \ker)}\right)\\
&\leq \cO\left([\sqrt{H^5 \log (1/\delta')} + \beta\sqrt{H^4 \cdot \Gamma(K, \lambda; \ker)}] /\sqrt{K} \right),
\end{align*}
where the second inequality is due to Lemma \ref{lem:explore_plan_connect_game} and the third inequality is by Lemma \ref{lem:bonus_explore_game}.

Since $\Pr(\cE \wedge \tilde{\cE}) \geq 1-2\delta'$ by the union bound, by setting $\delta' = 1 / (4H^2K^2)$, we obtain that with probability at least $1-1/(2H^2K^2)$
\begin{align*}
&V_1^\dagger(s_1, r) - V_1^{\pi, \bre(\pi)}(s_1, r)  \leq  \cO\left([\sqrt{2H^5 \log (2HK)} + \beta\sqrt{H^4 \cdot \Gamma(K, \lambda; \ker)}] /\sqrt{K} \right),\\
&V_1^{\bre(\nu), \nu}(s_1, r)  - V_1^\dagger(s_1, r)  \leq  \cO\left([\sqrt{2H^5 \log (2HK)} + \beta\sqrt{H^4 \cdot \Gamma(K, \lambda; \ker)}] /\sqrt{K} \right),
\end{align*}
such that
\begin{align*}
V_1^{\bre(\nu), \nu}(s_1, r) - V_1^{\pi, \bre(\pi)}(s_1, r)  &\leq  \cO\left([\sqrt{2H^5 \log (2HK)} + \beta\sqrt{H^4 \cdot \Gamma(K, \lambda; \ker)}] /\sqrt{K} \right)\\
&\leq  \cO\left(\beta\sqrt{H^4 [ \Gamma(K, \lambda; \ker)+\log(HK)]}/\sqrt{K} \right),
\end{align*}
where the last inequality is due to $\beta \geq H$. The event $\cE \wedge \tilde{\cE} $ happens if we further let $\beta$ satisfy
\begin{align*}
16H^2\big[ R^2_Q  + 2 \Gamma(K, \lambda;\ker) + 5+\log\cN_{\infty}(\varsigma^*;R_K,  2\beta) + 6\log(2HK)  \big] \leq  \beta^2, \forall h\in [H],
\end{align*}
where $\lambda = 1+1/K$, $\tilde{R}_K = R_K=2H\sqrt{\Gamma(K, \lambda; \ker)}$, and $\varsigma^* = H/K$. This completes the proof.
\end{proof}


\section{Proofs for Markov Game with Neural Function Approximation}

\subsection{Lemmas}

\begin{lemma} [Lemma C.7 of \citet{yang2020provably}]\label{lem:regularity_neural_game} With $TH^2 = \cO(m\log^{-6}m)$, then there exists a constant $\digamma$ such that the following inequalities hold with probability at least $1-1/m^2$ for any $z\in \cS\times \cA\times \cB$ and any $W \in \{W : \|W-\W0\|_2\leq H\sqrt{K/\lambda} \}$,
\begin{align*}
&|f(z;W)-\varphi(z;\W0)^\top(W-\W0)| \leq \digamma  K^{2/3} H^{4/3} m^{-1/6} \sqrt{\log m},\\
&\|\varphi(z;W) - \varphi(z;\W0)\|_2 \leq \digamma (KH^2/m)^{1/6} \sqrt{\log m}, \qquad \|\varphi(z;W) \|_2\leq \digamma,
\end{align*}
with $\digamma\geq 1$.
\end{lemma}

\begin{lemma} \label{lem:bonus_concentrate_neural_game} We define the event $\cE$ as that the following inequality holds $\forall z = (s,a, b) \in \cS \times \cA \times \cB, \forall (h,k) \in [H] \times [K]$,
\begin{align*}
&|\PP_h V_{h+1}^k(s,a, b) - f_h^k(s,a, b)|  \leq u_h^k(s,a, b) + \beta \iota,\\
&\left|\|\varphi(z; \Whk)\|_{\Lhki} - \|\varphi(z; \W0)\|_{\tLhki}\right| \leq \iota,
\end{align*}
where $\iota = 5K^{7/12} H^{1/6} m^{-1/12} \log^{1/4} m$ and we define
\begin{align*}
&\Lambda_h^k = \sum_{\tau=1}^{k-1} \varphi(z_h^{\tau}; \Whk) \varphi(z_h^{\tau}; \Whk)^\top + \lambda \cdot I, \quad  \tilde{\Lambda}_h^k = \sum_{\tau=1}^{k-1} \varphi(z_h^{\tau}; \W0) \varphi(z_h^{\tau}; \W0)^\top + \lambda \cdot I.
\end{align*}
Setting $(1+1/H)\beta = B_K$, $R_K = H\sqrt{K}$, $\varsigma^* = H/K$, and $\lambda=\F^2(1+1/K)$, $\varsigma^* = H/K$, if we let
\begin{align*}
\beta^2 &\geq 8R_Q^2 H^2 (1+\sqrt{\lambda/d})^2 + 32H^2 \Gamma(K, \lambda; \ker_m) + 80H^2\\
&\quad +32H^2\log\cN_{\infty}(\varsigma^*;R_K,  B_K) + 32H^2\log(K /\delta'),
\end{align*}
and also 
\begin{align*}
m = \Omega(K^{19} H^{14}\log^3 m),
\end{align*}
then we have that with probability at least $1-2/m^2 - \delta'$, the event $\cE$ happens, i.e.,
\begin{align*}
\Pr(\cE) \geq 1-2/m^2 - \delta'.
\end{align*} 
\end{lemma}

\begin{proof} By letting $\ba=(a,b)$ be an action in the space $\cA\times \cB$, Algorithm \ref{alg:exploration_phase_game} reduces to Algorithm \ref{alg:exploration_phase_single} with the action space $\cA \times \cB$ and state space $\cS$. We have $Q_h^k(s,a,b) = Q_h^k(s,\ba)$, $V_h^k(s,a,b) = V_h^k(s,\ba)$, $u_h^k(s,a,b) = u_h^k(s,\ba)$, $u_h^k(s,a,b) = u_h^k(s,\ba)$ and $r_h^k(s,a,b) = r_h^k(s,\ba)$. Simply applying the proof of Lemma \ref{lem:bonus_concentrate_neural}, we have the proof of this lemma. 
\end{proof}

\begin{lemma} \label{lem:bonus_explore_neural_game} Conditioned on the event $\cE$ defined in Lemma \ref{lem:bonus_concentrate_neural_game}, with probability at least $1-\delta'$, we have
\begin{align*}
&\sum_{k=1}^K V_1^*(s_1, r^k) \leq \sum_{k=1}^K V_1^k(s_1) + \beta HK \iota,\\
&\sum_{k=1}^K V_1^k(s_1)\leq \cO\left(\sqrt{H^3 K \log (1/\delta')} +\beta  \sqrt{H^2 K \cdot \Gamma(K, \lambda; \ker_m)}\right) + \beta H K \iota,
\end{align*}
where $\iota = 5K^{7/12} H^{1/6} m^{-1/12} \log^{1/4} m$.
\end{lemma}

\begin{proof} By the reduction of Algorithm \ref{alg:exploration_phase_game} to Algorithm \ref{alg:exploration_phase_single}, we can apply the same proof for Lemma \ref{lem:bonus_explore_neural}, which completes the proof.
\end{proof}

\begin{lemma} \label{lem:bonus_concentrate_plan_neural_game} We define the event $\tilde{\cE}$ as that the following inequality holds $\forall (s,a,b) \in \cS \times \cA \times \cB, \forall h \in [H]$,
\begin{align*}
&|\PP_h \overline{V}_{h+1}(s,a,b) - \overline{f}_h(s,a,b)|  \leq \overline{u}_h(s,a) + \beta \iota,\\
&|\PP_h \underline{V}_{h+1}(s,a,b) - \underline{f}_h(s,a,b)|  \leq \underline{u}_h(s,a) + \beta \iota,\\
&\left|\|\varphi(z; \overline{W}_h)\|_{(\overline{\Lambda}_h)^{-1}} - \|\varphi(z; \W0)\|_{(\tilde{\Lambda}_h)^{-1}}\right| \leq \iota,\\
&\left|\|\varphi(z; \underline{W}_h)\|_{(\underline{\Lambda}_h)^{-1}} - \|\varphi(z; \W0)\|_{(\tilde{\Lambda}_h)^{-1}}\right| \leq \iota.
\end{align*}
where $\iota = 5K^{7/12} H^{1/6} m^{-1/12} \log^{1/4} m$, and we define $\overline{f}_h(z) = \Pi_{[0, H]}[f(z; \overline{W}_h)]$ and $\underline{f}_h(z) = \Pi_{[0, H]}[f(z; \underline{W}_h)]$ as well as
\begin{align*}
&\overline{\Lambda}_h = \sum_{\tau=1}^K \varphi(z_h^{\tau}; \overline{W}_h) \varphi(z_h^{\tau}; \overline{W}_h)^\top + \lambda \cdot I, \quad \underline{\Lambda}_h = \sum_{\tau=1}^K \varphi(z_h^{\tau}; \underline{W}_h) \varphi(z_h^{\tau}; \underline{W}_h)^\top + \lambda \cdot I ,\\
&\tilde{\Lambda}_h = \sum_{\tau=1}^K \varphi(z_h^{\tau}; \W0) \varphi(z_h^{\tau}; \W0)^\top + \lambda \cdot I.
\end{align*}
Setting $\beta = \tilde{B}_K$, $\tilde{R}_K = H\sqrt{K}$, $\varsigma^* = H/K$, and $\lambda=\F^2(1+1/K)$, $\varsigma^* = H/K$, if we set
\begin{align*}
\beta^2 &\geq 8R_Q^2 H^2 (1+\sqrt{\lambda/d})^2 + 32H^2 + \Gamma(K, \lambda; \ker_m)  \\
&\quad + 80H^2+32H^2\log\cN_{\infty}(\varsigma^*;\tilde{R}_K,  \tilde{B}_K) + 32H^2\log(2K /\delta'),
\end{align*}
and also 
\begin{align*}
m = \Omega(K^{19} H^{14}\log^3 m),
\end{align*}
then we have that with probability at least $1-2/m^2 - \delta'$, the event $\tilde{\cE}$ happens, i.e.,
\begin{align*}
\Pr(\tilde{\cE}) \geq 1-2/m^2 - \delta'.
\end{align*}
\end{lemma}

\begin{proof} The proof of this lemma follows our proof of Lemmas \ref{lem:bonus_concentrate_neural} and \ref{lem:bonus_concentrate_plan_neural} and apply some similar ideas from the proof of Lemma \ref{lem:bonus_concentrate_plan_game}. Particularly, to deal with the upper bounds of the estimation errors of $\PP_h\overline{V}_{h+1}$ and $\PP_h\underline{V}_{h+1}$, we define the two value function space $\overline{\cV}$ and $\underline{\cV}$ and show their covering numbers similar to the proof of Lemma \ref{lem:bonus_concentrate_plan_game}. 
Then, we further use the proof of Lemma \ref{lem:bonus_concentrate_plan_neural}, which is derived from the proof of Lemma  \ref{lem:bonus_concentrate_neural}, to show the eventual results in this lemma. In the proof of this lemma, we set $\tilde{B}_K=\beta$ instead of $(1+1/H)\beta$ due to the structure of the planning phase. This completes the proof.
\end{proof}

\begin{lemma} \label{lem:bonus_plan_neural_game} Conditioned on the event $\tilde{\cE}$ as defined in Lemma \ref{lem:bonus_concentrate_plan_neural_game}, we have
\begin{align}
\begin{aligned} \label{eq:optimism_neural_game_1} 
&V_h^\dagger(s, r) \leq \overline{V}_h(s) + (H+1-h) \beta \iota,  \forall s \in \cS, \forall h \in [H],\\
&\overline{V}_h(s) \leq \EE_{a\sim \pi_h, b\sim \bre(\pi)_h} [(\PP_h\overline{V}_{h+1}  + r_h + 2\overline{u}_h)(s,a,b)]+ \beta \iota, \forall s \in \cS, \forall h \in [H],
\end{aligned}
\end{align}
\vspace{-0.8cm}
\begin{align}
\begin{aligned}\label{eq:optimism_neural_game_2} 
&V_h^\dagger(s, r) \geq \underline{V}_h(s)-(H+1-h)\beta \iota, \forall s \in \cS, \forall h \in [H],\\ 
&\underline{V}_h(s)\geq \EE_{a\sim \bre(\nu)_h, b\sim \nu_h} [(\PP_h\underline{V}_{h+1}  - r_h - 2\underline{u}_h)(s,a,b)] - \beta\iota, \forall s \in \cS, \forall h \in [H].
\end{aligned}
\end{align}

\end{lemma}
\begin{proof} We prove the first inequality in  \eqref{eq:optimism_neural_game_1} by induction. For $h = H+1$, we have $V_{H+1}^\dagger(s, r)  = \overline{V}_{H+1}(s) = 0$ for any $s\in \cS$. Then, we assume that $V_{h+1}^\dagger(s, r)  \leq \overline{V}_{h+1}(s) + (H-h)\beta\iota$. Thus, conditioned on the event $\tilde{\cE}$ as defined in Lemma \ref{lem:bonus_concentrate_plan_neural_game}, we have
\begin{align*}
&Q_h^\dagger(s, a,b, r) - \overline{Q}_h(s, a,b) \\
&\qquad= r_h(s,a,b) + \PP_h V_{h+1}^\dagger(s,a,b,r)  -  \min \{ [r_h(s,a,b) + \overline{f}_h(s,a,b) + u_h(s,a,b)], H \}^+  \\
&\qquad\leq  \max \{ [\PP_h V_{h+1}^\dagger(s,a,b,r)  - \overline{f}_h(s,a,b) - \overline{u}_h(s,a,b)], 0 \}  \\
&\qquad\leq  \max \{ [\PP_h V_{h+1}(s,a,b) + (H-h)\beta \iota  - f_h(s,a,b) - \overline{u}_h(s,a,b)], 0 \}  \\
&\qquad\leq  (H+1-h)\beta \iota,
\end{align*}
where the first inequality is due to $0 \leq r_h(s,a,b) + \PP_h V_{h+1}^\dagger(s,a,b,r) \leq H$ and $\min \{x, y\}^+ \geq \min \{x, y\}$, the second inequality is by the assumption that $V_{h+1}^\dagger(s,a,b, r)\leq \overline{V}_{h+1}(s,a,b) + (H-h)\beta \iota$, the last inequality is by Lemma \ref{lem:bonus_concentrate_plan_neural_game} such that $|\PP_h \overline{V}_{h+1}(s,a,b) - \overline{f}_h(s,a,b) | \leq \overline{u}_h(s,a,b) + \beta \iota$ holds for any $(s,a,b)\in \cS \times \cA\times \cB$ and $(k,h)\in [K]\times[H]$. The above inequality leads to 
\begin{align*}
V_h^\dagger(s, r) &= \max_{\pi_h'}\min_{\nu_h'}\EE_{a\sim \pi_h', b\sim \nu_h'} [Q_h^\dagger(s, a, b, r)] \\
&\leq \max_{\pi_h'}\min_{\nu_h'} \EE_{a\sim \pi_h', b\sim \nu_h'}[\overline{Q}_h(s, a,b)] +  (H+1-h)\beta \iota \\
&= \overline{V}_h(s) +  (H+1-h)\beta \iota. 
\end{align*}
Therefore, we have
\begin{align*}
V_h^\dagger(s, r)  \leq  \overline{V}_h(s)+  (H+1-h)\beta \iota, \forall h\in [H], \forall s\in \cS. 
\end{align*}
We further prove the second inequality in \eqref{eq:optimism_neural_game_1}. We have
\begin{align*}
\overline{Q}_h(s, a,b) &= \min \{ [r_h(s,a,b) + \overline{f}_h(s,a,b) + \overline{u}_h(s,a,b)], H \}^+\\
& \leq \min \{ [r_h(s,a,b) + \PP_h \overline{V}_{h+1}(s,a,b) + 2\overline{u}_h(s,a,b) + \beta \iota], H \}^+\\
& \leq r_h(s,a,b) + \PP_h \overline{V}_{h+1}(s,a,b) + 2\overline{u}_h(s,a,b) + \beta \iota,
\end{align*}
where the first inequality is also by Lemma \ref{lem:bonus_concentrate_plan_neural_game} such that $|\PP_h \overline{V}_{h+1}(s,a,b) - \overline{f}_h(s,a,b) | \leq \overline{u}_h(s,a,b)+ \beta \iota$, and the last inequality is because of the non-negativity of  $r_h(s,a,b) + \PP_h V_{h+1}(s,a,b) + 2\overline{u}_h(s,a,b)+ \beta \iota$. Therefore, we have
\begin{align*}
\overline{V}_h(s) &=  \min_{\nu'} \EE_{a\sim \pi_h, b\sim \nu'} \overline{Q}_h(s,a, b)\\
&\leq  \EE_{a\sim \pi_h, b\sim \bre(\pi)_h} \overline{Q}_h(s,a, b)\\
&\leq  \EE_{a\sim \pi_h, b\sim \bre(\pi)_h} [ r_h(s,a,b) + \PP_h \overline{V}_{h+1}(s,a,b) + 2\overline{u}_h(s,a,b)] + \beta \iota.
\end{align*}
For the inequalities in \eqref{eq:optimism_neural_game_2}, we can prove them in the same way to proving \eqref{eq:optimism_neural_game_1}. From the perspective of Player 2, this player is trying to find a policy to maximize the cumulative rewards w.r.t. a reward function $\{-r_h(s,a,b)\}_{h\in [H]}$. Thus, one can further use the proof technique for \eqref{eq:optimism_neural_game_1} to prove \eqref{eq:optimism_neural_game_2}. This completes the proof.
\end{proof}

\begin{lemma} \label{lem:explore_plan_connect_neural_game}  With the exploration and planning phases, conditioned on the event $\cE$ defined in Lemma \ref{lem:bonus_concentrate_neural_game} and the event $\tilde{\cE}$ defined in Lemma \ref{lem:bonus_concentrate_plan_neural_game}, we have the following inequalities
\begin{align*}
&K \cdot V_1^{\pi, \bre(\pi)}(s_1, \overline{u}/H) \leq  \sum_{k=1}^K V_1^*(s_1, r^k)  + 2K\beta \iota, \\
&K \cdot V_1^{\bre(\nu) , \nu }(s_1, \underline{u}/H) \leq  \sum_{k=1}^K V_1^*(s_1, r^k)+ 2K\beta \iota.
\end{align*}

\end{lemma}

\begin{proof} First, we have $K \cdot V_1^{\pi, \bre(\pi)}(s_1, \overline{u}/H) \leq K \cdot V_1^*(s_1, \overline{u}/H)$ as well as $K \cdot V_1^{\bre(\nu) , \nu }(s_1, \underline{u}/H) \leq K \cdot V_1^*(s_1, \underline{u}/H)$ according to the definition of $V^*_1$. Thus, to prove this lemma, we only need to show
\begin{align*}
&K \cdot V_1^*(s_1, \overline{u}/H) \leq  \sum_{k=1}^K V_1^*(s_1, r^k) + 2K\beta \iota,\\
&K \cdot V_1^*(s_1, \underline{u}/H) \leq  \sum_{k=1}^K V_1^*(s_1, r^k) + 2K\beta \iota.
\end{align*}
Because the constructions of the planning bonus $\overline{u}_h$ and the exploration reward $r_h^k$ are the same as the ones for the single-agent case, similar to the proof of Lemma \ref{lem:explore_plan_connect_neural}, and according to Lemmas \ref{lem:bonus_concentrate_neural_game} and  \ref{lem:bonus_concentrate_plan_neural_game}, we have the following results
\begin{align*}
\overline{u}_h(s,a,b) \leq H\cdot  r_h^k(s,a,b) + 2\beta\iota, \quad \underline{u}_h(s,a,b) \leq H\cdot  r_h^k(s,a,b) + 2\beta\iota
\end{align*}
such that
\begin{align*}
V_1^*(s_1, \overline{u}/H) \leq  V_1^*(s_1, r^k) + 2 \beta \iota,  \quad  V_1^*(s_1, \underline{u}/H) \leq  V_1^*(s_1, r^k) + 2 \beta \iota, 
\end{align*}
Therefore, we eventually obtain 
\begin{align*}
&K \cdot V_1^{\pi, \bre(\pi)}(s_1, \overline{u}/H) \leq K \cdot V_1^*(s_1, \overline{u}/H) \leq  \sum_{k=1}^K V_1^*(s_1, r^k)+ 2 K \beta \iota,\\
&K \cdot V_1^{\bre(\nu) , \nu }(s_1, \underline{u}/H) \leq K \cdot V_1^*(s_1, \underline{u}/H) \leq  \sum_{k=1}^K V_1^*(s_1, r^k)+ 2K \beta \iota.
\end{align*}
This completes the proof.
\end{proof}

\subsection{Proof of Theorem \ref{thm:main_neural_game}} \label{sec:proof_main_neural_game}
\begin{proof} Conditioned on the events $\cE$ and $\tilde{\cE}$ defined in Lemmas \ref{lem:bonus_concentrate_neural_game} and \ref{lem:bonus_concentrate_plan_neural_game}, we have
\begin{align}\label{eq:proof_start_neural_game}
V_1^\dagger(s_1, r) - V_1^{\pi, \bre(\pi)}(s_1, r) \leq \overline{V}_1(s_1) - V_1^{\pi, \bre(\pi)}(s_1, r) + H\beta\iota,
\end{align}
where the inequality is by Lemma \ref{lem:bonus_plan_neural_game}. Further by this lemma, we have
\begin{align*}
&\overline{V}_h(s_h) - V_h^{\pi, \bre(\pi)}(s_h, r) \\
&\quad \leq  \EE_{a_h\sim \pi_h, b_h\sim \bre(\pi)_h} [(\PP_h\overline{V}_{h+1}  + r_h + 2\overline{u}_h)(s_h,a_h,b_h)]- V_h^{\pi, \bre(\pi)}(s_h, r) +\beta\iota \\
&\quad= \EE_{a_h\sim \pi_h, b_h\sim \bre(\pi)_h} [(r_h+ \PP_h \overline{V}_{h+1}+ 2\overline{u}_h)(s_h,a_h, b_h) - r_h(s_h,a_h, b_h) \\
&\quad\quad- \PP_h V_{h+1}^{\pi, \bre(\pi)}(s_h,a_h,b_h, r) ]+\beta\iota\\
&\quad= \EE_{a_h\sim \pi_h, b_h\sim \bre(\pi)_h} [ \PP_h \overline{V}_{h+1}(s_h,a_h, b_h) - \PP_h V_{h+1}^{\pi, \bre(\pi)}(s_h,a_h,b_h, r) + 2\overline{u}_h(s_h,a_h,b_h) ]+\beta\iota\\
&\quad= \EE_{a_h\sim \pi_h, b_h\sim \bre(\pi)_h, s_{h+1}\sim\PP_h} [ \overline{V}_{h+1}(s_{h+1}) -  V_{h+1}^{\pi, \bre(\pi)}(s_{h+1}, r) + 2\overline{u}_h(s_h,a_h,b_h) ]+\beta\iota.
\end{align*}
Recursively applying the above inequality and making use of $\overline{V}_{H+1}(s,r) = V_{H+1}^{\pi, \bre(\pi)}(s)  = 0$ gives
\vspace{-0.1cm}
\begin{align*}
&\overline{V}_1(s_1) - V_1^{\pi, \bre(\pi)}(s_1,r) \\
&\qquad  \leq  \EE_{\forall h\in [H]: ~a_h\sim \pi_h, b_h\sim \bre(\pi)_h, s_{h+1}\sim\PP_h}\left[\sum_{h=1}^H 2\overline{u}_h(s_h,a_h, b_h)\Bigg| s_1 \right]\\
&\qquad=2H \cdot V_1^{\pi, \bre(\pi)}(s_1, \overline{u}/H) + H\beta\iota. 
\end{align*}
Combining with \eqref{eq:proof_start_neural_game} gives
\begin{align*}
&V_1^\dagger(s_1, r) - V_1^{\pi, \bre(\pi)}(s_1,r)\\
&\qquad \leq 2H \cdot V_1^{\pi, \bre(\pi)}(s_1, \overline{u}/H) + 2H\beta\iota \leq \frac{2H}{K} \sum_{k=1}^K V_1^*(s_1, r^k) + 4H\beta\iota\\
&\qquad\leq \frac{2H}{K} \cO\left(\sqrt{H^3 K \log (1/\delta')} + \beta\sqrt{H^2 K \cdot \Gamma(K, \lambda; \ker_m)}\right) + (H  + 4)H\beta \iota \\
&\qquad \leq \cO\left([\sqrt{H^5 \log (1/\delta')} + \beta\sqrt{H^4 \cdot \Gamma(K, \lambda; \ker_m)}] /\sqrt{K} + H^2 \beta \iota \right),
\end{align*}
where the second inequality is due to Lemma \ref{lem:explore_plan_connect_neural_game} and the third inequality is by Lemma \ref{lem:bonus_explore_neural_game}.

Next, we give the upper bound of $V_1^{\bre(\nu), \nu}(s_1, r)- V_1^\dagger(s_1, r)$. Conditioned on the event $\cE$ defined in Lemma \ref{lem:bonus_concentrate_neural_game} and the event $\tilde{\cE}$ defined in Lemma \ref{lem:bonus_concentrate_plan_neural_game}, we have
\begin{align}\label{eq:proof_start_neural_game_2}
V_1^{\bre(\nu), \nu}(s_1, r)- V_1^\dagger(s_1, r) \leq V_1^{\bre(\nu), \nu}(s_1, r)  - \underline{V}_1(s_1) + H\beta\iota,
\end{align}
where the inequality is by Lemma \ref{lem:bonus_plan_neural_game}. Further by this lemma, we have
\begin{align*}
&V_h^{\bre(\nu), \nu}(s_h, r)  - \underline{V}_h(s_h)  \\
&\ \ \ \leq  V_h^{\bre(\nu), \nu}(s_h, r) - \EE_{a\sim \bre(\nu)_h, b\sim \nu_h} [(\PP_h\underline{V}_{h+1}  - r_h - 2\underline{u}_h)(s_h,a_h,b_h)] +\beta\iota \\
&\ \ \ = \EE_{a_h\sim \bre(\nu)_h, b_h\sim \nu_h} [\PP_h V_{h+1}^{\bre(\nu), \nu}(s_h,a_h,b_h, r) - \PP_h \underline{V}_{h+1}(s_h,a_h, b_h)  + 2\underline{u}_h(s_h,a_h,b_h) ]+\beta\iota\\
&\ \ \ = \EE_{a_h\sim \bre(\nu)_h, b_h\sim \nu_h, s_{h+1}\sim\PP_h} [ V_{h+1}^{\bre(\nu), \nu}(s_{h+1}, r) - \PP_h \underline{V}_{h+1}(s_{h+1})  + 2\underline{u}_h(s_h,a_h,b_h) ]+\beta\iota.
\end{align*}
Recursively applying the above inequality gives
\begin{align*}
V_1^{\bre(\nu), \nu}(s_1, r)  - \underline{V}_1(s_1) \leq 2H \cdot V_1^{\bre(\nu), \nu}(s_1, \underline{u}/H) + H\beta\iota. 
\end{align*}
Combining with \eqref{eq:proof_start_neural_game_2} gives
\begin{align*}
V_1^{\bre(\nu), \nu}(s_1, r)- V_1^\dagger(s_1, r)&\leq 2H \cdot V_1^\pi(s_1, \underline{u}/H) + 2H\beta\iota \leq \frac{2H}{K} \sum_{k=1}^K V_1^*(s_1, r^k) + 4H\beta\iota\\
&\leq \cO\left([\sqrt{H^5 \log (1/\delta')} + \beta\sqrt{H^4 \cdot \Gamma(K, \lambda; \ker_m)}] /\sqrt{K} + H^2 \beta \iota \right),
\end{align*}
where the second inequality is due to Lemma \ref{lem:explore_plan_connect_neural_game} and the last inequality is by Lemma \ref{lem:bonus_explore_neural_game}. Thus, we eventually have
\begin{align*}
&V_1^{\bre(\nu), \nu}(s_1, r)- V_1^{\pi, \bre(\pi)}(s_1, r) \\
&\qquad \leq   \cO\left(\big[\sqrt{H^5 \log (1/\delta')} + \beta\sqrt{H^4 \cdot \Gamma(K, \lambda; \ker_m)}\big] /\sqrt{K} + H^2 \beta \iota \right).
\end{align*}
Moreover, we also have $P(\cE \wedge \tilde{\cE}) \geq 1-2\delta'-4/m^2$ by the union bound. Therefore, since $\beta \geq H$ as shown in Lemmas \ref{lem:bonus_concentrate_neural_game} and \ref{lem:bonus_concentrate_plan_neural_game}, setting $\delta' = 1/(4K^2H^2)$, we obtain that with probability at least $1-1/(2K^2H^2)-4/m^2$, 
\begin{align*}
V_1^*(s_1, r) - V_1^\pi(s_1, r) \leq \cO\left(\beta\sqrt{H^4 [ \Gamma(K, \lambda; \ker_m)+\log(KH)]} /\sqrt{K} + H^2 \beta \iota\right).
\end{align*}
The event $\cE \wedge \tilde{\cE} $ happens if we further let $\beta$ satisfy
\begin{align*}
\beta^2 &\geq 8R_Q^2 H^2 (1+\sqrt{\lambda/d})^2  + 32H^2 \Gamma(K, \lambda; \ker_m) + 80H^2 \\
&\quad +32H^2\log\cN_{\infty}(\varsigma^*;R_K,  2\beta) + 96H^2\log(2KH).
\end{align*}
where guarantees the conditions in Lemmas \ref{lem:bonus_concentrate_neural_game} and \ref{lem:bonus_concentrate_plan_neural_game} hold. This completes the proof.
\end{proof}

\section{Other Supporting Lemmas}
\begin{lemma}[Lemma E.2 of \citet{yang2020provably}] \label{lem:self_normalize_uniform} Let $\{s_\tau\}_{\tau=1}^\infty$ and $\{\phi_\tau\}_{\tau=1}^\infty$ be $\cS$-valued and $\cH$-valued stochastic processes adapted to filtration $\{\cF_\tau\}_{\tau=0}^\infty$, respectively, where we assume that $\|\phi_\tau\| \leq 1$ for all $\tau \geq 1$. Moreover, for any $t \geq 1$, we let $\cK_t \in  \RR^{t\times t}$ be the Gram matrix of $\{\phi_\tau \}_{\tau\in[t]}$ and define an operator $\Lambda_t: \cH \mapsto \cH$ as $\Lambda_t = \lambda I  + \sum_{\tau=1}^t\phi_\tau \phi_\tau^\top$ with $\lambda > 1$. Let $\cV \subseteq \{V : \cS \mapsto [0, H]\}$ be a class of bounded functions on $\cS$. Then for any $\delta \in (0, 1)$, with probability at least $1 - \delta$, we have simultaneously for all $t \geq 1$ that
\begin{align*}
&\sup_{V\in \cV} \left\| \sum_{\tau=1}^t \phi_\tau \{V(s_\tau) - \EE[V(s_\tau)|\cF_{\tau-1}] \}\right\|_{\Lambda_t^{-1}}^2 \\
&\qquad \leq 2H^2 \log\det(I+\cK_t/\lambda) + 2H^2t(\lambda-1)+4H^2\log(\cN_\epsilon/\delta)+ 8t^2\epsilon^2/\lambda,
\end{align*}
where  $\cN_\epsilon$ is the $\epsilon$-covering number of $\cV$ with respect to the distance $\dist(\cdot, \cdot):=\sup_{\cS}|V_1(s)-V_2(s)|$.
\end{lemma}

\begin{lemma}[Lemma E.3  of \citet{yang2020provably}] \label{lem:direct_sum_bound} Let $\{\phi_t\}_{t\geq1}$ be a sequence in the RKHS $\cH$. Let
$\Lambda_0 : \cH \mapsto \cH$ be defined as $\lambda I$ where $\lambda \geq 1$ and $I$ is the identity mapping on $\cH$. For any $t \geq 1$, we define a self-adjoint and positive-definite operator $\Lambda_t$ by letting $\Lambda_t = \Lambda_0 + \sum_{j=1}^t \phi_j\phi_j^\top$. Then, for any $t \geq	 1$, we have
\begin{align*}
\sum_{j=1}^t \min\{ 1, \phi_j \Lambda_{j-1}^{-1} \phi_j^\top \} \leq 2\log\det (I+\cK_t/\lambda),
\end{align*}
where $\cK_t \in \RR^{t\times t}$ is the Gram matrix obtained from $\{\phi_j\}_{j\in[t]}$, i.e., for any $j, j' \in [t]$, the $(j, j')$-th entry of $\cK_t$ is $\langle \phi_j, \phi_j\rangle_\cH	$. Moreover, if we further have $\sup_{t\geq 0} \{\|\phi_t\|_\cH\} \leq 1$, then it holds that
\begin{align*}
\log\det(I+\cK_t/\lambda)\leq \sum_{j=1}^t \phi_j^\top \Lambda_{j-1}^{-1} \phi_j \leq 2\log\det(I+\cK_t/\lambda).
\end{align*}
\end{lemma}

\end{appendices}

\end{document}